\documentclass{article}

\usepackage{url}
\usepackage{times}
\usepackage{graphicx} 
\usepackage{caption}
\usepackage{subcaption}
\usepackage{wrapfig} 

\usepackage{natbib}

\usepackage{algorithm}
\usepackage{algorithmic}

\usepackage{bm}
\usepackage{bbm}
\usepackage[mathscr]{eucal}
\usepackage{amsthm}
\usepackage{amsfonts}
\usepackage{mathtools}

\usepackage{hyperref}

\usepackage{balance}

\theoremstyle{definition}
\newtheorem{definition}{Definition}[section]

\theoremstyle{theorem}
\newtheorem{theorem}{Theorem}[section]

\theoremstyle{lemma}
\newtheorem{lemma}{Lemma}[section]

\usepackage[accepted]{icml2017}

\newcommand{\citeeg}[1]{\citeauthor{#1}, \citeyear{#1}}

\icmltitlerunning{A Laplacian Framework for Option Discovery in Reinforcement Learning}

\DeclareMathOperator*{\argmax}{arg\,max}

\begin{document} 

\twocolumn[
\icmltitle{A Laplacian Framework for Option Discovery in Reinforcement Learning}


\begin{icmlauthorlist}
\icmlauthor{Marlos C. Machado}{equal}
\icmlauthor{Marc G. Bellemare}{to}
\icmlauthor{Michael Bowling}{equal}
\end{icmlauthorlist}

\icmlaffiliation{equal}{University of Alberta}
\icmlaffiliation{to}{Google DeepMind}

\icmlcorrespondingauthor{Marlos C. Machado}{machado@ualberta.ca}

\icmlkeywords{reinforcement learning, option discovery, proto-value functions}

\vskip 0.3in
]



\printAffiliationsAndNotice{}  

\begin{abstract}
Representation learning and option discovery are two of the biggest challenges in reinforcement learning (RL). Proto-value functions (PVFs) are a well-known approach for representation learning in MDPs. In this paper we address the option discovery problem by showing how PVFs implicitly define options. We do it by introducing \emph{eigenpurposes}, intrinsic reward functions derived from the learned representations. The options discovered from eigenpurposes traverse the principal directions of the state space. They are useful for multiple tasks because they are discovered without taking the environment's rewards into consideration. Moreover, different options act at different time scales, making them helpful for exploration. We demonstrate features of eigenpurposes in traditional tabular domains as well as in Atari 2600 games.

\end{abstract} 

\section{Introduction}
\label{introduction}

Two important challenges in reinforcement learning (RL) are the problems of representation learning and of automatic discovery of skills. Proto-value functions (PVFs) are a well-known solution for the problem of representation learning~\cite{Mahadevan05,Mahadevan07}; while the problem of skill discovery is generally posed under the options framework~\cite{Sutton99,Precup00}, which models skills as options.

In this paper, we tie together representation learning and option discovery by showing how PVFs implicitly define options. One of our main contributions is to introduce the concepts of \emph{eigenpurpose} and \emph{eigenbehavior}. Eigenpurposes are intrinsic reward functions that incentivize the agent to traverse the state space by following the principal directions of the learned representation. Each intrinsic reward function leads to a different \emph{eigenbehavior}, which is the optimal policy for that reward function. In this paper we introduce an algorithm for option discovery that leverages these ideas. The options we discover are task-independent because, as PVFs, the eigenpurposes are obtained without any information about the environment's reward structure. We first present these ideas in the tabular case and then show how they can be generalized to the function approximation case.

Exploration, while traditionally a separate problem from option discovery, can also be addressed through the careful construction of options~\cite{McGovern01,Simsek05,Solway14,Kulkarni16}. In this paper, we provide evidence that not all options capable of accelerating planning are useful for exploration. We show that options traditionally used in the literature to speed up planning hinder the agents' performance if used for random exploration during learning.  Our options have two important properties that allow them to improve exploration: (i)~they operate at different time scales, and (ii)~they can be easily sequenced. Having options that operate at different time scales allows agents to make finely timed actions while also decreasing the likelihood the agent will explore only a small portion of the state space. Moreover, because our options are defined across the whole state space, multiple options are available in every state, which allows them to be easily sequenced.

\section{Background}
\label{background}

We generally indicate random variables by capital letters (\emph{e.g.}, $R_t$), vectors by bold letters (\emph{e.g.}, $\bm{\theta}$), functions by lowercase letters (\emph{e.g.}, $v$), and sets by calligraphic font (\emph{e.g.}, $\mathscr{S}$).

\subsection{Reinforcement Learning}
In the RL framework \citep{Sutton98}, an agent aims to maximize cumulative reward by taking actions in an environment. These actions affect the agent's next state and the rewards it experiences. We use the MDP formalism throughout this paper. An MDP is a 5-tuple $\langle\mathscr{S}, \mathscr{A}, r, p, \gamma\rangle$. At time $t$ the agent is in state $s_t \in \mathscr{S}$ where it takes action $a_t \in \mathscr{A}$ that leads to the next state $s_{t+1} \in \mathscr{S}$ according to the transition probability kernel $p(s'|s, a)$, which encodes $\Pr(S_{t+1} = s' | S_{t} = s, A_{t} = a)$. The agent also observes a reward $R_{t+1} \sim r(s, a)$. The agent's goal is to learn a policy $\mu : \mathscr{S} \times \mathscr{A} \rightarrow [0,1]$ that maximizes the expected discounted return $G_t \doteq \mathbb{E}_{p,\mu} \big[\sum_{k=0}^{\infty} \gamma^k R_{t+k+1} | s_t\big]$, where $\gamma \in [0, 1)$ is the discount factor. 

It is common to use the policy improvement theorem~\cite{Bellman57} when learning to maximize $G_t$. One technique is to alternate between solving the Bellman equations for the \emph{action-value function} $q_{\mu_k}(s,a)$,
\begin{eqnarray*}
q_{\mu_k}(s,a) \!\!\!\!\! &\doteq& \!\!\!\! \mathbb{E}_{\mu_k, p} \big[G_t | S_t = s, A_t = a\big]\\
               \!\!\!\!\! &=& \!\!\!\! \sum_{s', r} p(s', r | s, a) \big[r + \gamma \sum_{a'} \mu_k(a'|s')q_{\mu_k}(s', a')\big]
\end{eqnarray*}
and making the next policy, $\mu_{k+1}$, greedy w.r.t. $q_{\mu_{k}}$,
\begin{eqnarray*}
\mu_{k+1} \doteq \argmax_{a \in \mathscr{A}} q_{\mu_{k}}(s,a),
\end{eqnarray*}
until converging to an optimal policy $\mu_*$.

Sometimes it is not feasible to learn a value for each state-action pair due to the size of the state space. Generally, this is addressed by parameterizing $q_\mu(s,a)$ with a set of weights $\bm{\theta} \in \mathbb{R}^n$ such that $q_\mu(s,a) \approx q_\mu(s, a, \bm{\theta})$. It is common to approximate $q_\mu$ through a linear function, \emph{i.e.}, $q_\mu(s, a, \bm{\theta}) = \bm{\theta}^\top \bm{\phi}(s,a)$, where $\bm{\phi}(s, a)$ denotes a linear feature representation of state $s$ when taking action $a$.

\subsection{The Options Framework}

The options framework extends RL by introducing temporally extended actions called \emph{skills} or \emph{options}. An option $\omega$ is a 3-tuple $\omega = \langle \mathcal{I}, \pi, \mathcal{T} \rangle$ where $\mathcal{I} \in \mathscr{S}$ denotes the option's initiation set, $\pi : \mathscr{A} \times \mathscr{S} \rightarrow [0,1]$ denotes the option's policy, and $\mathcal{T} \in \mathscr{S}$ denotes the option's termination set. After the agent decides to follow option $\omega$ from a state in $\mathcal{I}$, actions are selected according to $\pi$ until the agent reaches a state in $\mathcal{T}$. Intuitively, options are higher-level actions that extend over several time steps, generalizing MDPs to semi-Markov decision processes (SMDPs)~\cite{Puterman94}. 

Traditionally, options capable of moving agents to \emph{bottleneck} states are sought after. Bottleneck states are those states that connect different densely connected regions of the state space (\emph{e.g.}, doorways) \cite{Simsek04,Solway14}. They have been shown to be very efficient for planning as these states are the states most frequently visited when considering the \emph{shortest} distance between any two states in an MDP~\cite{Solway14}.

\subsection{Proto-Value Functions}

Proto-value functions (PVFs) are learned representations that capture large-scale temporal properties of an environment~\cite{Mahadevan05,Mahadevan07}. They are obtained by diagonalizing a diffusion model, which is constructed from the MDP's transition matrix. A diffusion model captures information flow on a graph, and it is commonly defined by the \emph{combinatorial graph Laplacian} matrix $L = D - A$, where $A$ is the graph's adjacency matrix and $D$ the diagonal matrix whose entries are the row sums of $A$. Notice that the adjacency matrix $A$ easily generalizes to a weight matrix $W$. PVFs are defined to be the eigenvectors obtained after the eigendecomposition of $L$. Different diffusion models can be used to generate PVFs, such as the \emph{normalized graph Laplacian} $L = D^{-\frac{1}{2}}(D - A)D^{-\frac{1}{2}}$, which we use in this paper.

\section{Option Discovery through the Laplacian}
\label{option_discovery}

PVFs capture the large-scale geometry of the environment, such as symmetries and bottlenecks. They are task independent, in the sense that they do not use information related to reward functions. Moreover, they are defined over the whole state space since each eigenvector induces a real-valued mapping over each state. We can imagine that options with these properties should also be useful. In this section we show how to use PVFs to discover options.

Let us start with an example. Consider the traditional 4-room domain depicted in Figure~\ref{fig:4room}. Gray squares represent walls and white squares represent accessible states. Four actions are available: \emph{up}, \emph{down}, \emph{right}, and \emph{left}. The transitions are deterministic and the agent is not allowed to move into a wall. Ideally, we would like to discover options that move the agent from room to room. Thus, we should be able to automatically distinguish between the different rooms in the environment. This is exactly what PVFs do, as depicted in Figure~\ref{fig:eigen_option} (left). Instead of interpreting a PVF as a basis function, we can interpret the PVF in our example as a desire to reach the highest point of the plot, corresponding to the centre of the room. 
Because the sign of an eigenvector is arbitrary, a PVF can also be interpreted as a desire to reach the lowest point of the plot, corresponding to the opposite room. In this paper we use the eigenvectors in both directions (\emph{i.e.}, both signs).

An \emph{eigenpurpose} formalizes the interpretation above by defining an intrinsic reward function. We can see it as defining a \emph{purpose} for the agent, that is, to maximize the discounted sum of these rewards. 

\theoremstyle{definition}
\begin{definition}[Eigenpurpose]
An \emph{eigenpurpose} is the intrinsic reward function $r_i^{\bf{e}}(s, s')$ of a proto-value function $\bf{e} \in \mathbb{R}^{|\mathscr{S}|}$ such that
\begin{eqnarray}
r_i^{\bf{e}}(s, s') &=& {\bf e}^\top (\bm{\phi}(s') - \bm{\phi}(s)),
\end{eqnarray}
where $\bm{\phi}(x)$ denotes the feature representation of state $x$.
\end{definition}

Notice that an eigenpurpose, in the tabular case, can be written as $r_i^{\bf e}(s, s') = {\bf e}[s'] - {\bf e}[s]$.

We can now define a new MDP to learn the option associated with the purpose, $\mathcal{M}_i^{\bf e} = \langle \mathscr{S}, \mathscr{A} \cup \{\bot\}, r_i^{\bf e}, p, \gamma \rangle$, where the reward function is defined as in (1) and the action set is augmented by the action \emph{terminate} ($\bot$), which allows the agent to leave $\mathcal{M}_i^{\bf e}$ without any cost. The state space and the transition probability kernel remain unchanged from the original problem. The discount rate can be chosen arbitrarily, although it impacts the timescale the option encodes.

With $\mathcal{M}_i^{\bf e}$ we define a new state-value function $v_\pi^{\bf e}(s)$, for policy $\pi$, as the expected value of the cumulative discounted intrinsic reward if the agent starts in state $s$ and follows policy $\pi$ until termination. Similarly, we define a new action-value function $q_\pi^{\bf e}(s, a)$ as the expected value of the cumulative discounted intrinsic reward if the agent starts in state $s$, takes action $a$, and then follows policy $\pi$ until termination. We can also describe the optimal value function for any eigenpurpose obtained through $\bf{e}$:
$$ v_*^{\bf e}(s) = \max_{\pi} v_\pi^{\bf e}(s) \ \ \ \ \ \mbox{and} \ \ \ \ \ \ q_*^{\bf e}(s, a) = \max_{\pi} q_\pi^{\bf e}(s,a). $$
These definitions naturally lead us to \emph{eigenbehaviors}.
\begin{definition}[Eigenbehavior]
An \emph{eigenbehavior} is a policy $\chi^{\bf e}: \mathscr{S} \rightarrow \mathscr{A}$ that is optimal with respect to the eigenpurpose $r_i^{\bf e}$, \emph{i.e.}, $\chi^{\bf e}(s) = \argmax_{a \in \mathscr{A}} q_{*}^{\bf e}(s, a)$. 
\end{definition}

\begin{figure}[t]
    \centering
    \begin{subfigure}[b]{0.3\columnwidth}
    \center
        \includegraphics[width=\columnwidth]{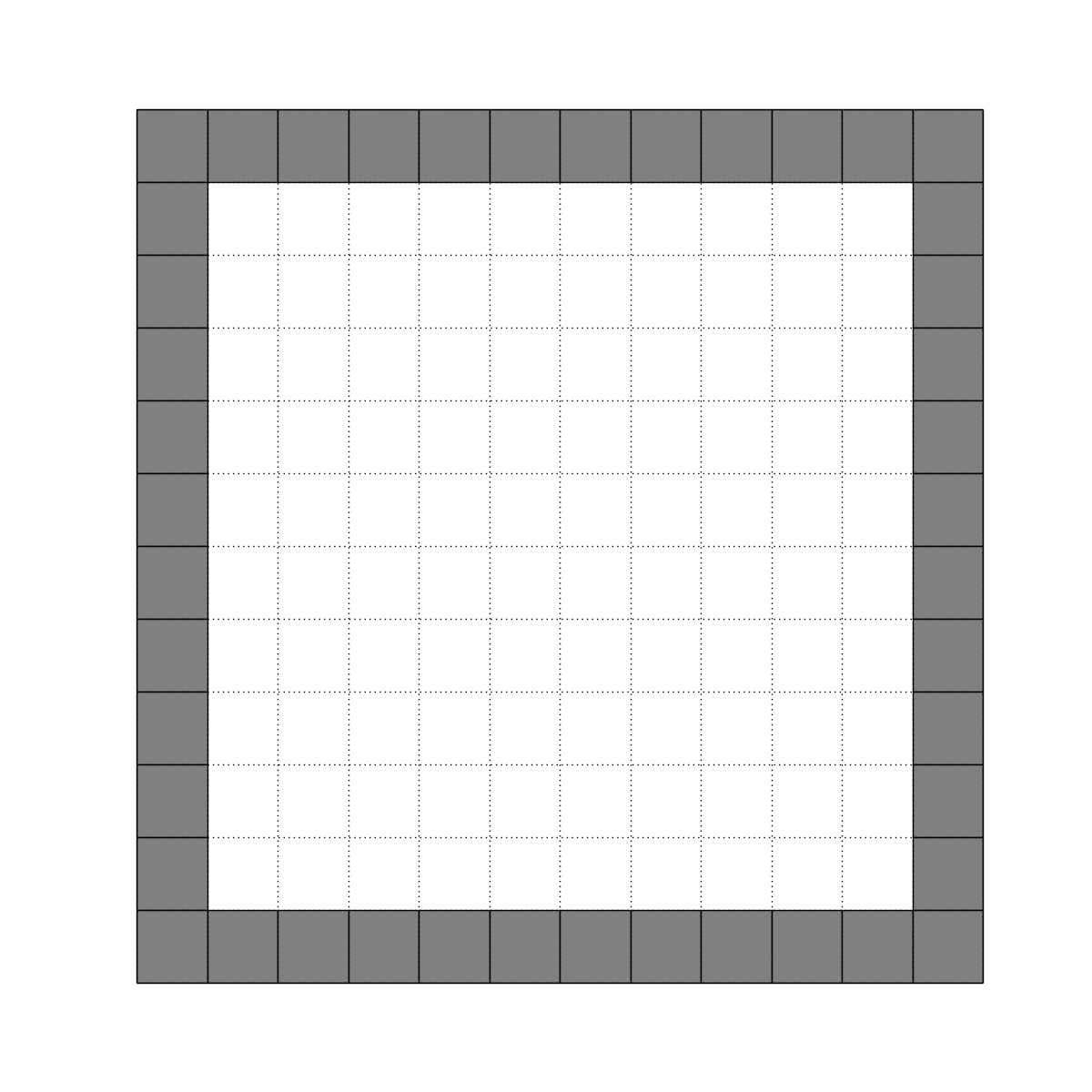}
        \caption{10$\times$10 grid}
        \label{fig:opengrid}
    \end{subfigure}
    ~
    \begin{subfigure}[b]{0.3\columnwidth}
    \center
        \raisebox{8.75mm}{\includegraphics[height=\columnwidth,angle=90]{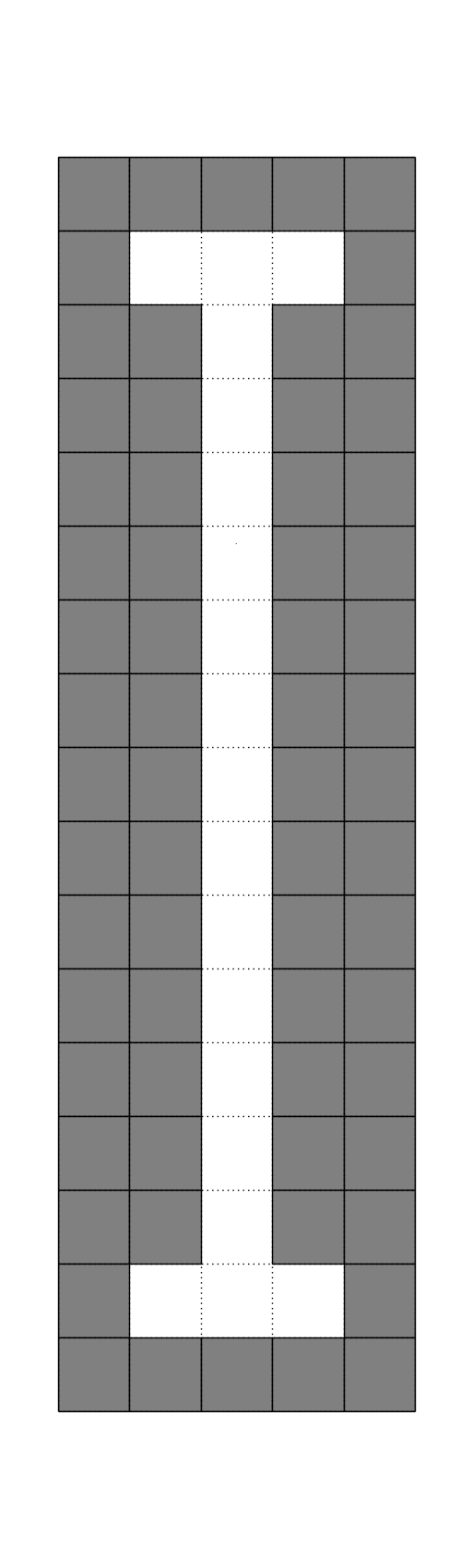}}
        \caption{I-Maze}
        \label{fig:corridor}
    \end{subfigure}
    ~
    \begin{subfigure}[b]{0.3\columnwidth}
    \center
        \includegraphics[width=\columnwidth]{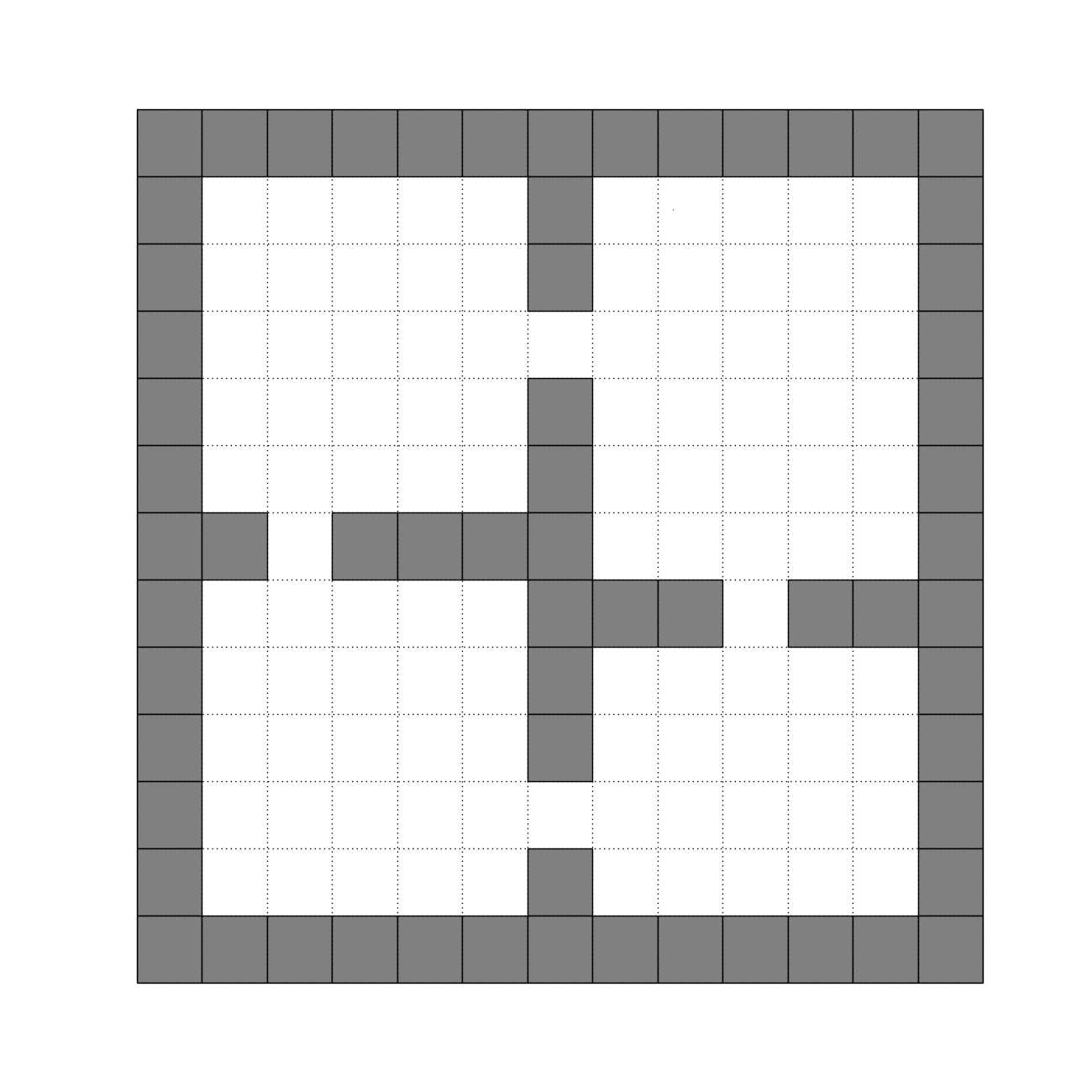}
        \caption{4-room domain}
        \label{fig:4room}
    \end{subfigure}
    \caption{Domains used for evaluation.}\label{fig:domains}
\end{figure}

Finding the optimal policy $\pi_*^{\bf e}$ now becomes a traditional RL problem, with a different reward function. Importantly, this reward function tends to be dense, avoiding challenging situations due to exploration issues. In this paper we use policy iteration to solve for an optimal policy. 

If each eigenpurpose defines an option, its corresponding eigenbehavior is the option's policy. Thus, we need to define the option's initiation and termination set. An option should be available in every state where it is possible to achieve its purpose, and to terminate when it is achieved.

When defining the MDP to learn the option, we augmented the agent's action set with the \emph{terminate} action, allowing the agent to interrupt the option anytime. We want options to terminate when the agent achieves its purpose, \emph{i.e.}, when it is unable to accumulate further positive intrinsic rewards. With the defined reward function, this happens when the agent reaches the state with largest value in the eigenpurpose (or a local maximum when $\gamma < 1$). Any subsequent reward will be negative. We are able to formalize this condition by defining $q_\chi(s,\bot) \doteq 0$ for all $\chi^{\bf e}$. When the terminate action is selected, control is returned to the higher level policy~\cite{Dietterich00}. An option following a policy $\chi^{\bf e}$ terminates when $q_\chi^{\bf e} (s, a) \leq 0$ for all $a \in \mathscr{A}$. We define the initiation set to be all states in which there exists an action $a \in \mathscr{A}$ such that $q_\chi^{\bf e} (s, a) > 0$. Thus, the option's policy is $\pi^{\bf e}(s) = \argmax_{a \in \mathscr{A} \cup \{\bot\}} q_\pi^{\bf e}(s,a)$. We refer to the options discovered with our approach as \emph{eigenoptions}. The eigenoption corresponding to the example at the beginning of this section is depicted in Figure~\ref{fig:eigen_option} (right).

\begin{figure}[t]
    \centering
    \begin{subfigure}[b]{0.49\columnwidth}
        \includegraphics[width=\columnwidth]{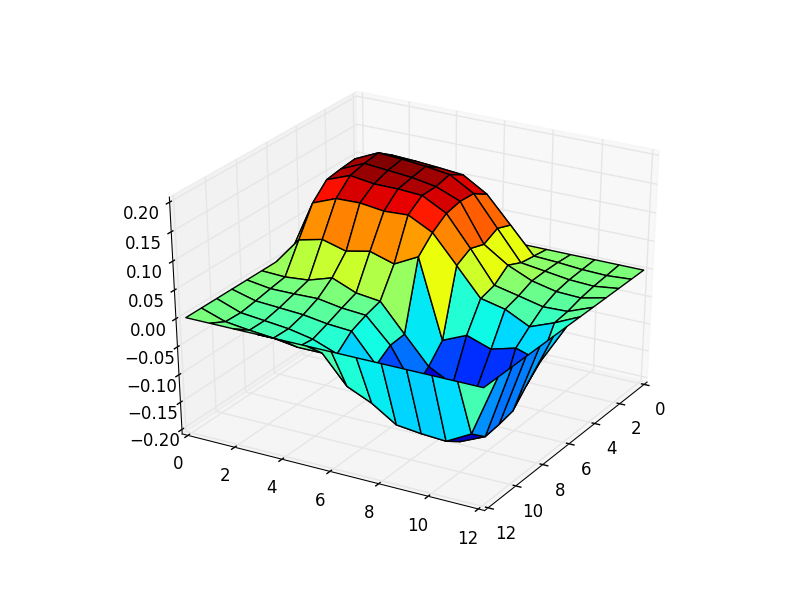}
    \end{subfigure}
    \begin{subfigure}[b]{0.49\columnwidth}
        \includegraphics[width=0.8\columnwidth]{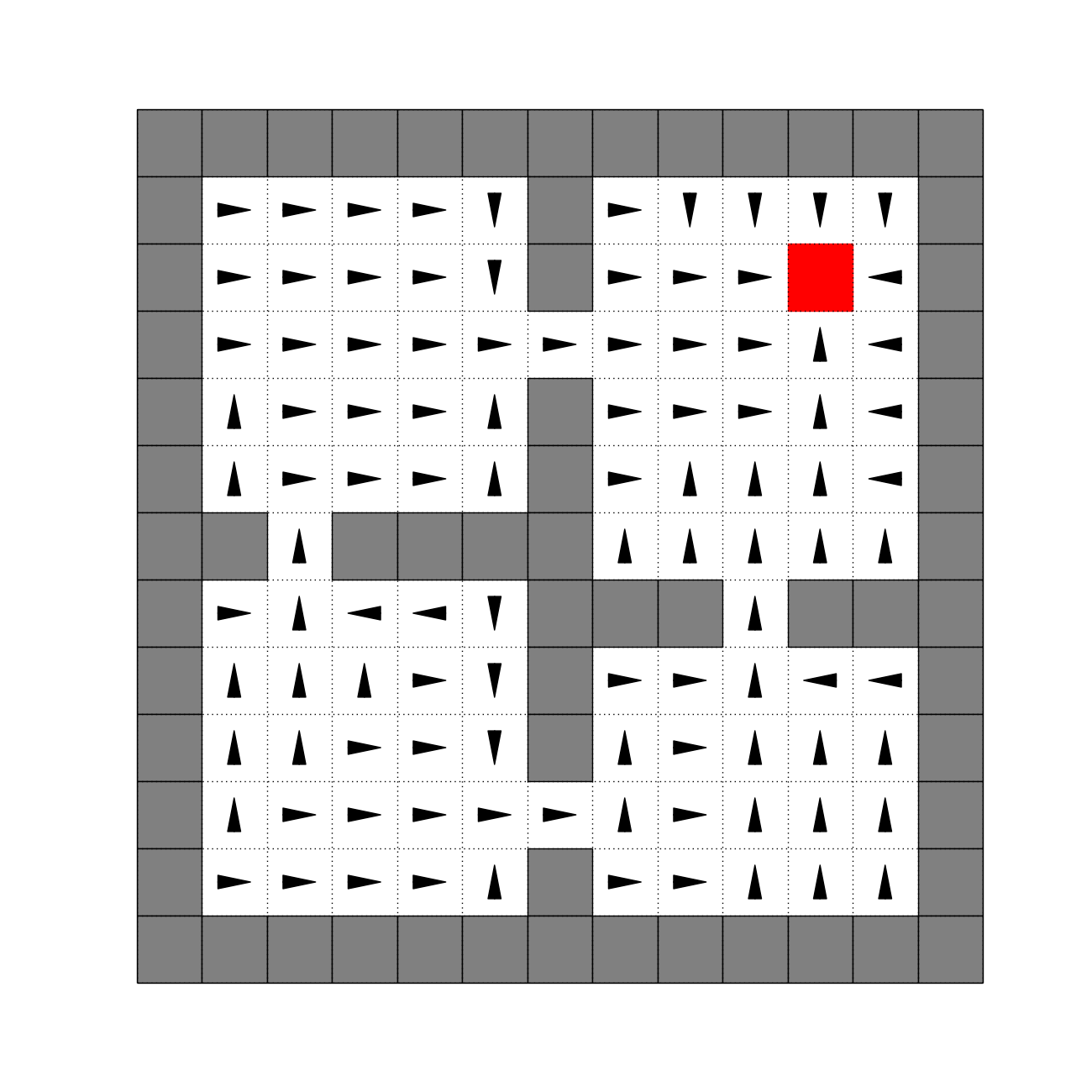}
    \end{subfigure}
    \caption{Second PVF~(left) and its corresponding option~(right) in the 4-room domain. Action \emph{terminate} is depicted in red (top right corner), other actions are depicted as arrows.}\label{fig:eigen_option}
\end{figure}

For any eigenoption, there is always at least one state in which it terminates, as we now show.

\begin{figure*}[t]
    \centering
    \begin{subfigure}[b]{0.24\textwidth}
    \center
        \includegraphics[width=0.8\columnwidth]{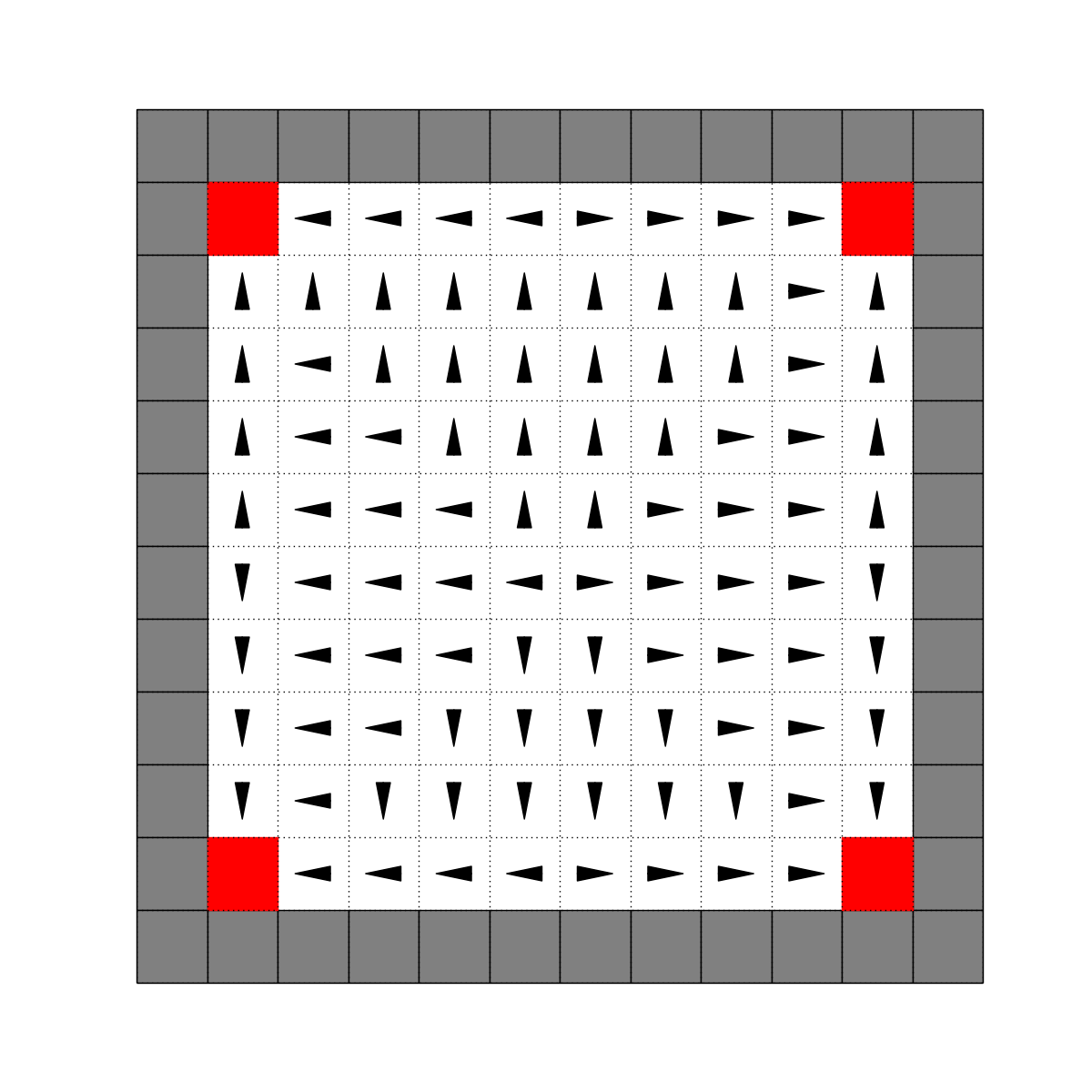}
    \end{subfigure}
    \begin{subfigure}[b]{0.24\textwidth}
    \center
        \includegraphics[width=0.8\columnwidth]{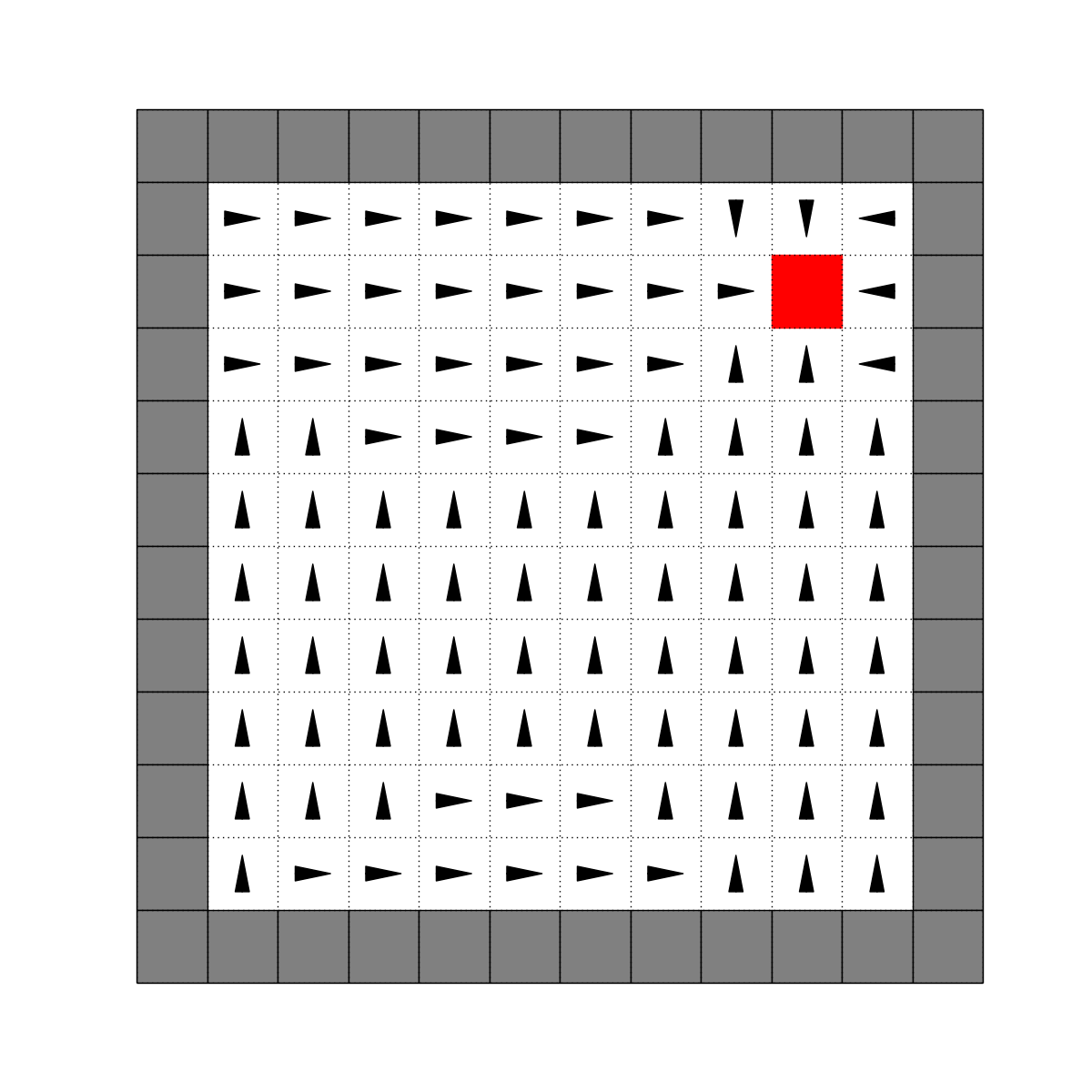}
    \end{subfigure}
    \begin{subfigure}[b]{0.24\textwidth}
    \center
        \includegraphics[width=0.8\columnwidth]{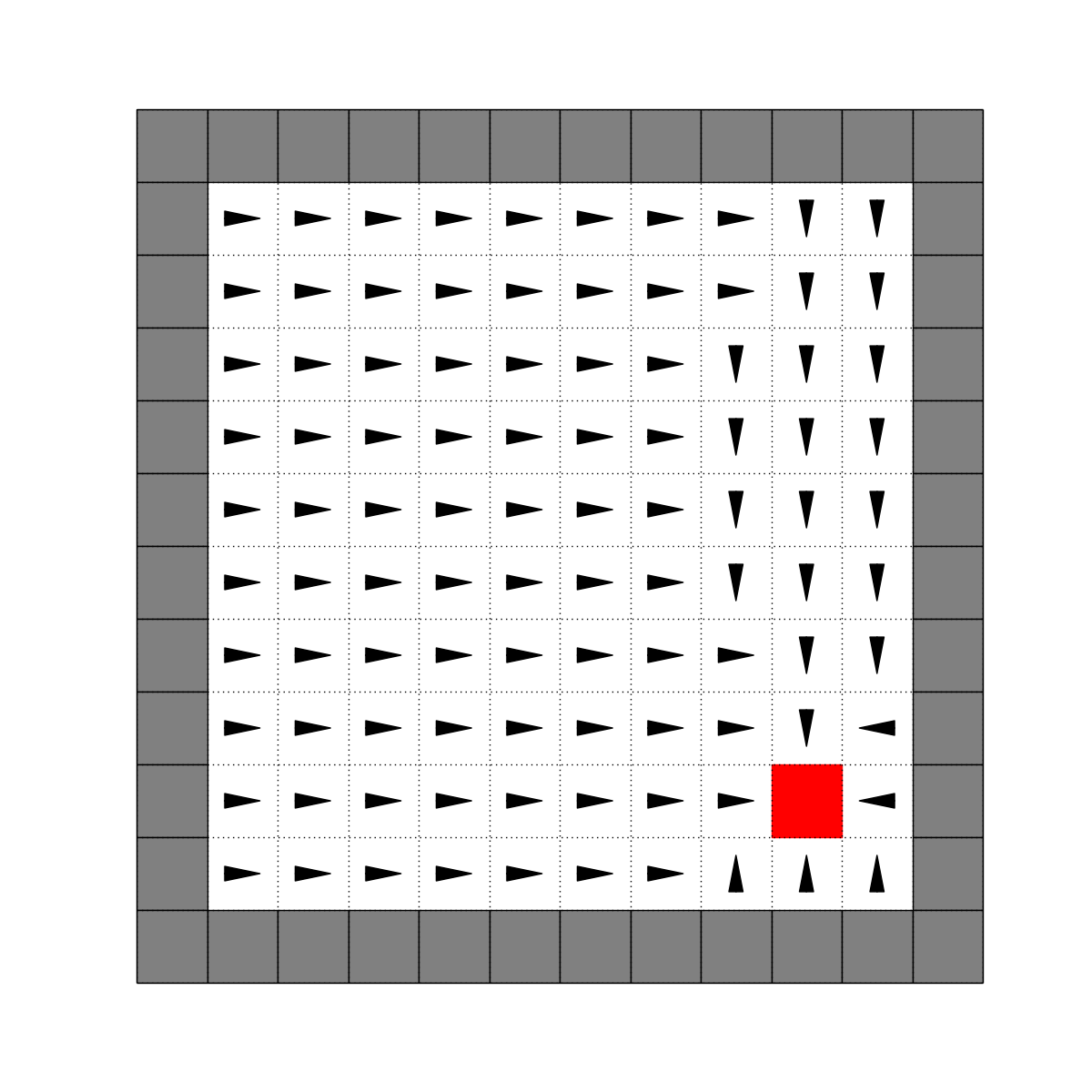}
    \end{subfigure}
    \begin{subfigure}[b]{0.24\textwidth}
    \center
        \includegraphics[width=0.8\columnwidth]{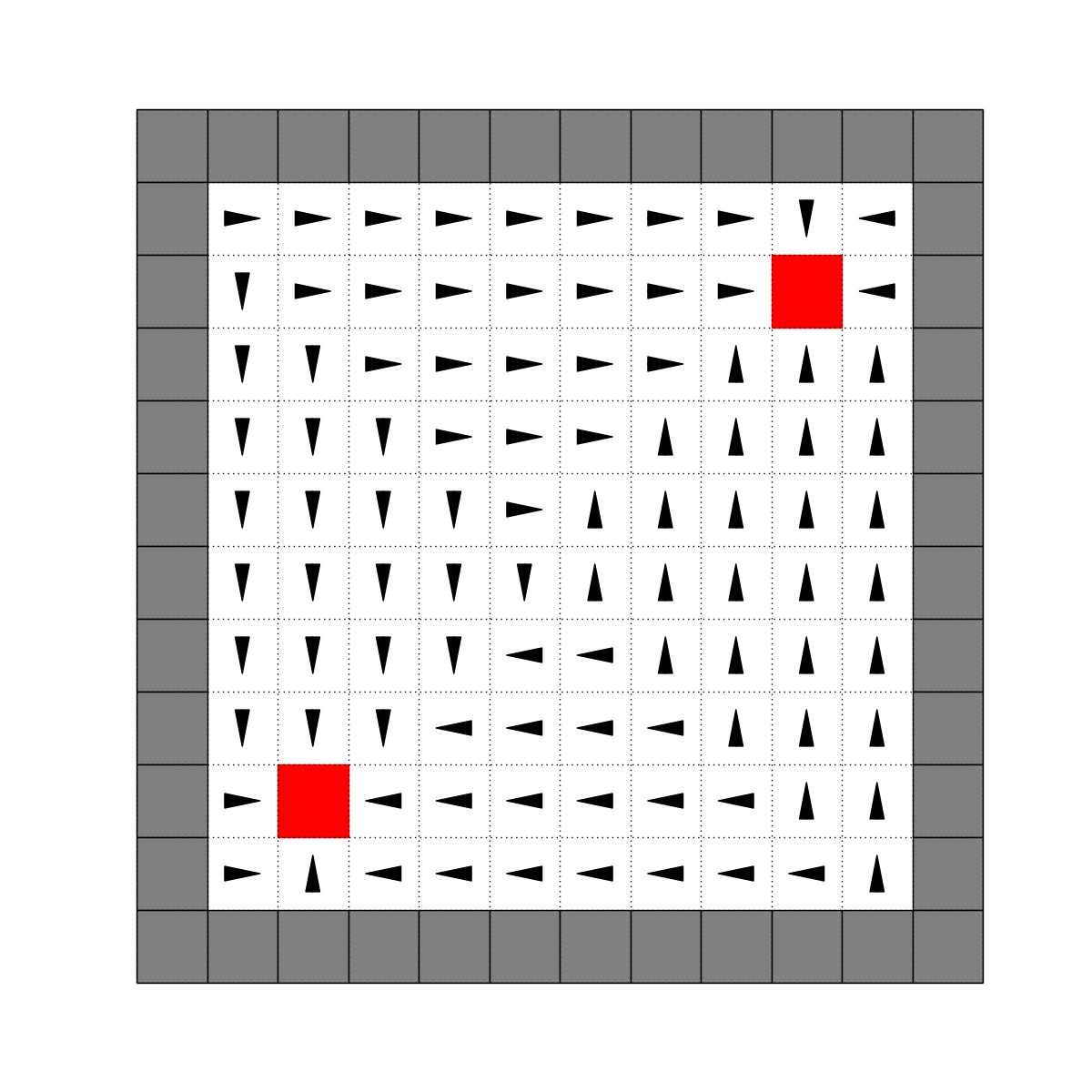}
    \end{subfigure}
    \caption{Options obtained from the four smallest eigenvectors in the 10$\times$10 grid. Action \emph{terminate} is depicted in red.}
    \label{fig:options_opengrid}
\end{figure*}

\begin{figure*}[t]
    \centering
    \begin{subfigure}[b]{0.24\textwidth}
    \center
        \includegraphics[height=0.8\columnwidth,angle=90]{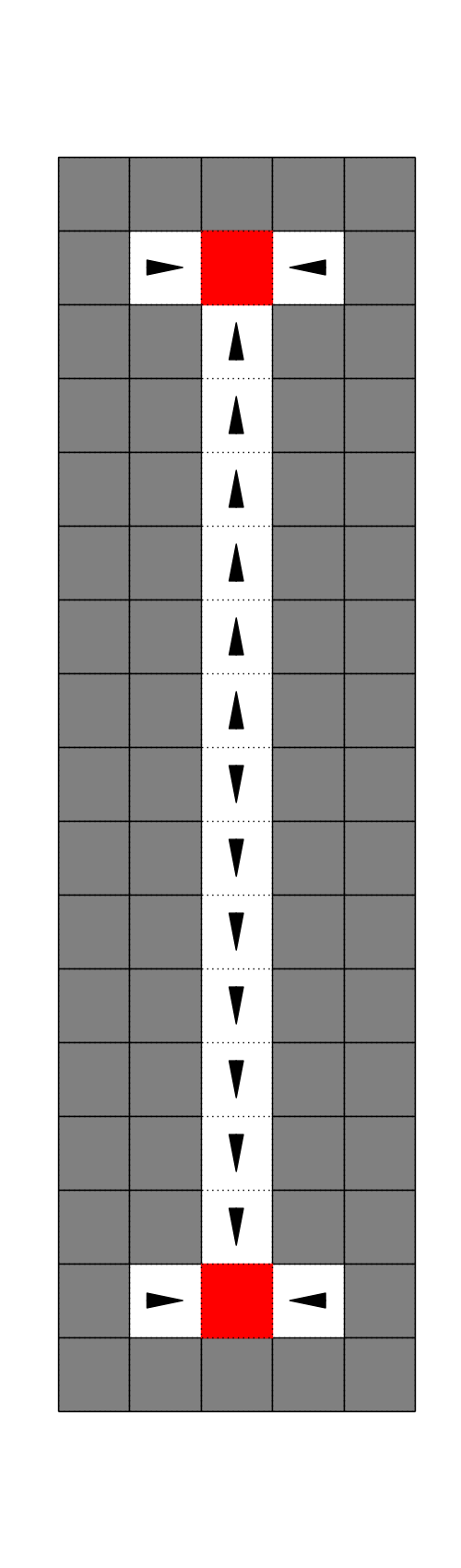}
    \end{subfigure}
    \begin{subfigure}[b]{0.24\textwidth}
    \center
        \includegraphics[height=0.8\columnwidth,angle=90]{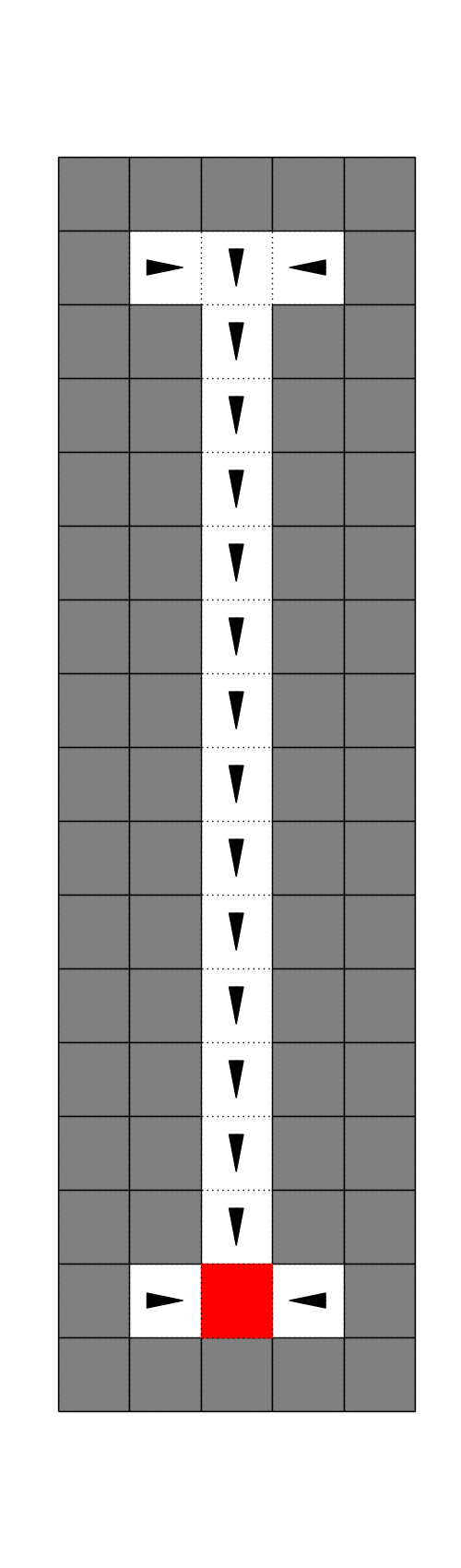}
    \end{subfigure}
    \begin{subfigure}[b]{0.24\textwidth}
    \center
        \includegraphics[height=0.8\columnwidth,angle=90]{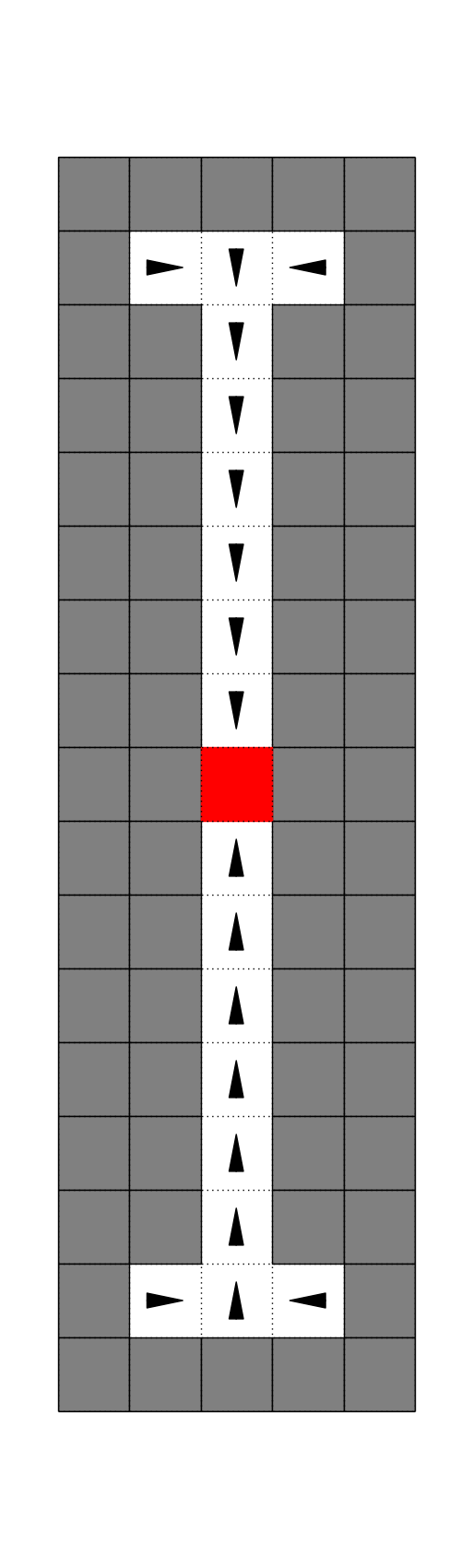}
    \end{subfigure}
    \begin{subfigure}[b]{0.24\textwidth}
    \center
        \includegraphics[height=0.8\columnwidth,angle=90]{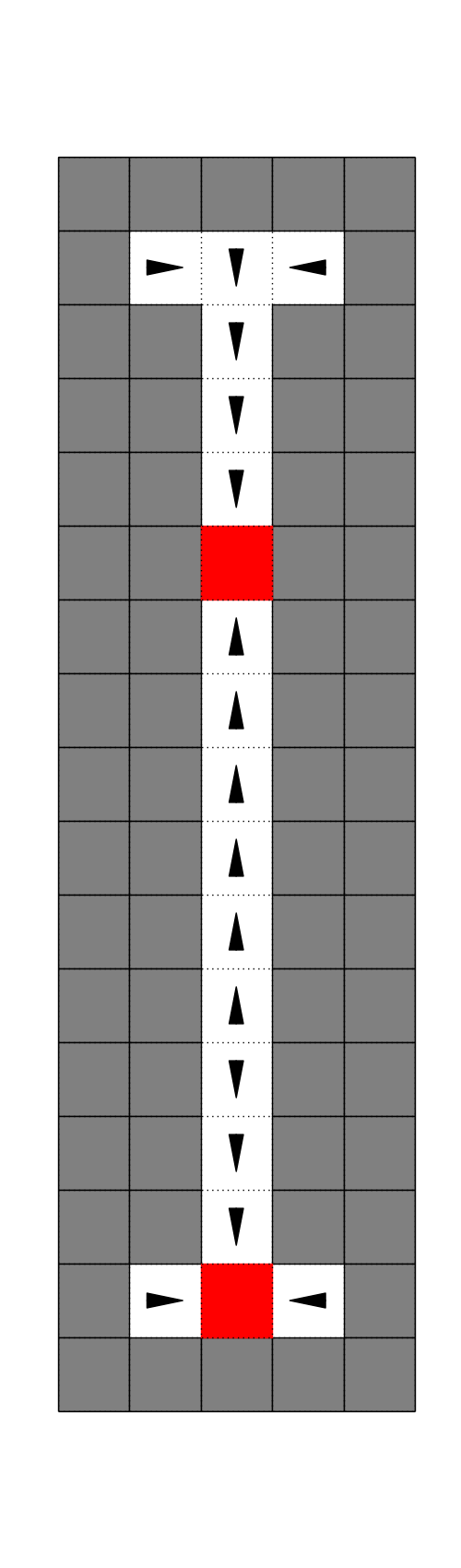}
    \end{subfigure}
    \caption{Options obtained from the four smallest eigenvectors in the I-Maze domain. Action \emph{terminate} is depicted in red.}
    \label{fig:options_longI}
\end{figure*}

\begin{theorem}[Option's Termination]
Consider an eigenoption $o = \langle\mathcal{I}_o, \pi_o, \mathcal{T}_o\rangle$ and $\gamma~<~1$. Then, in an MDP with finite state space, $\mathcal{T}_o$ is nonempty.
\end{theorem}
\begin{proof}
We can write the Bellman equation in the matrix form: ${\bf v} = {\bf r} + \gamma T \bf{v}$, where $\bf{v}$ is a \emph{finite} column vector with one entry per state encoding its value function. From (1) we have ${\bf r} = T\bf{w} - \bf{w}$ with ${\bf w} = \bm{\phi}(s)^\top \bf{e}$, where $\bf{e}$ denotes the eigenpurpose of interest. Therefore:
\begin{align*}
\bf{v} + \bf{w}                              &=      T{\bf w} + \gamma T {\bf v}\\
                                             &=      (1- \gamma) T {\bf w} + \gamma T ({\bf v + w})\\
                                             &=      (1 - \gamma) (I - \gamma T)^{-1} T \bf{w}.
\end{align*}
\begin{align*}
 ||\bf{v} + \bf{w}||_\infty                  &=      (1 - \gamma)||(I - \gamma T)^{-1} T \bf{w}||_\infty\\
 ||\bf{v} + \bf{w}||_\infty                  &\le    (1 - \gamma)||(I - \gamma T)^{-1} T||_\infty ||\bf{w}||_\infty\\
 ||\bf{v} + \bf{w}||_\infty                  &\le    (1 - \gamma) \frac{1}{(1-\gamma)} ||\bf{w}||_\infty\\
 ||\bf{v} + \bf{w}||_\infty                  &\le    ||\bf{w}||_\infty\\
\end{align*}
We can shift $\bf{w}$ by any finite constant without changing the reward, \emph{i.e.}, $T {\bf w \! - \! w} = T({\bf w} \! + \! \bm{\delta}) - ({\bf w} \! + \! \bm{\delta})$ because $T{\bf1}\bm{\delta} = \bf{1}\bm{\delta}$ since $\sum_j T_{i,j} \! = \! 1$. Hence, we can assume $\bf{w} \! \ge \!\bf{0}$. Let $s^* = \argmax_s {\bf w}_{s^*}$, so that ${\bf w}_{s^*} = ||{\bf w}||_\infty$. Clearly ${\bf v}_{s^*} \le \bf{0}$, otherwise $||{\bf v + w}||_\infty \ge |{\bf v}_{s^*} + {\bf w}_{s^*}| = {\bf v}_{s^*} + {\bf w}_{s^*} > {\bf w}_{s^*} = ||{\bf w}||_\infty$, arriving at a contradiction.
\end{proof}

This result is applicable in both the tabular and linear function approximation case. An algorithm that does not rely on knowing the underlying graph is provided in Section~\ref{sec:ale}.

\section{Empirical Evaluation}
\label{sec:experiments}

We used three MDPs in our empirical study (\emph{c.f.}~Figure~\ref{fig:domains}): an open room, an I-Maze, and the 4-room domain. Their transitions are deterministic and gray squares denote walls. Agents have access to four actions: \emph{up}, \emph{down}, \emph{right}, and \emph{left}. When an action that would have taken the agent into a wall is chosen, the agent's state does not change. We demonstrate three aspects of our framework:\footnote{\emph{Python} code can be found at: \\ \indent \indent \url{https://github.com/mcmachado/options}}
\begin{itemize}
\setlength\itemsep{0.1cm}
\item How the eigenoptions present specific purposes. Interestingly, options leading to bottlenecks are not the first ones we discover.
\item How eigenoptions improve exploration by reducing the expected number of steps required to navigate between any two states.
\item How eigenoptions help agents to accumulate reward faster. We show how few options may hurt the agents' performance while enough options speed up learning.
\end{itemize}

\subsection{Discovered Options}~\label{sec:discovered_options}

\vspace{-0.6cm}

In the PVF theory, the ``smoothest'' eigenvectors, corresponding to the smallest eigenvalues, are preferred~\cite{Mahadevan07}. The same intuition applies to eigenoptions, with the eigenpurposes corresponding to the smallest eigenvalues being preferred. Figures~\ref{fig:options_opengrid}, \ref{fig:options_longI}, and~\ref{fig:options_4room} depict the first eigenoptions discovered in the three domains used for evaluation.

\begin{figure*}[t]
    \centering
    \begin{subfigure}[b]{0.24\textwidth}
    \center
        \includegraphics[width=0.8\columnwidth]{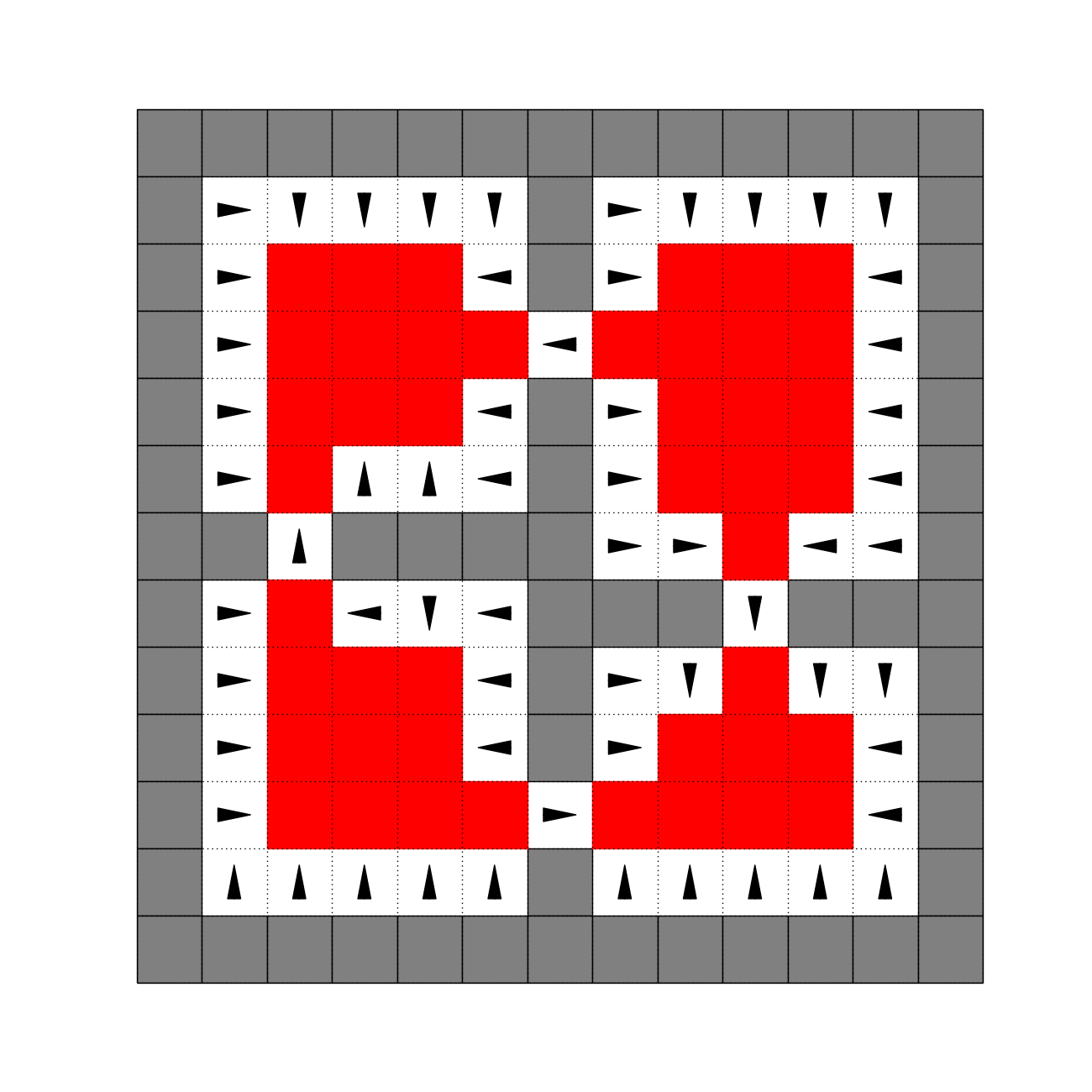}
    \end{subfigure}
    \begin{subfigure}[b]{0.24\textwidth}
    \center
        \includegraphics[width=0.8\columnwidth]{fig/options/4rooms_2}
    \end{subfigure}
    \begin{subfigure}[b]{0.24\textwidth}
    \center
        \includegraphics[width=0.8\columnwidth]{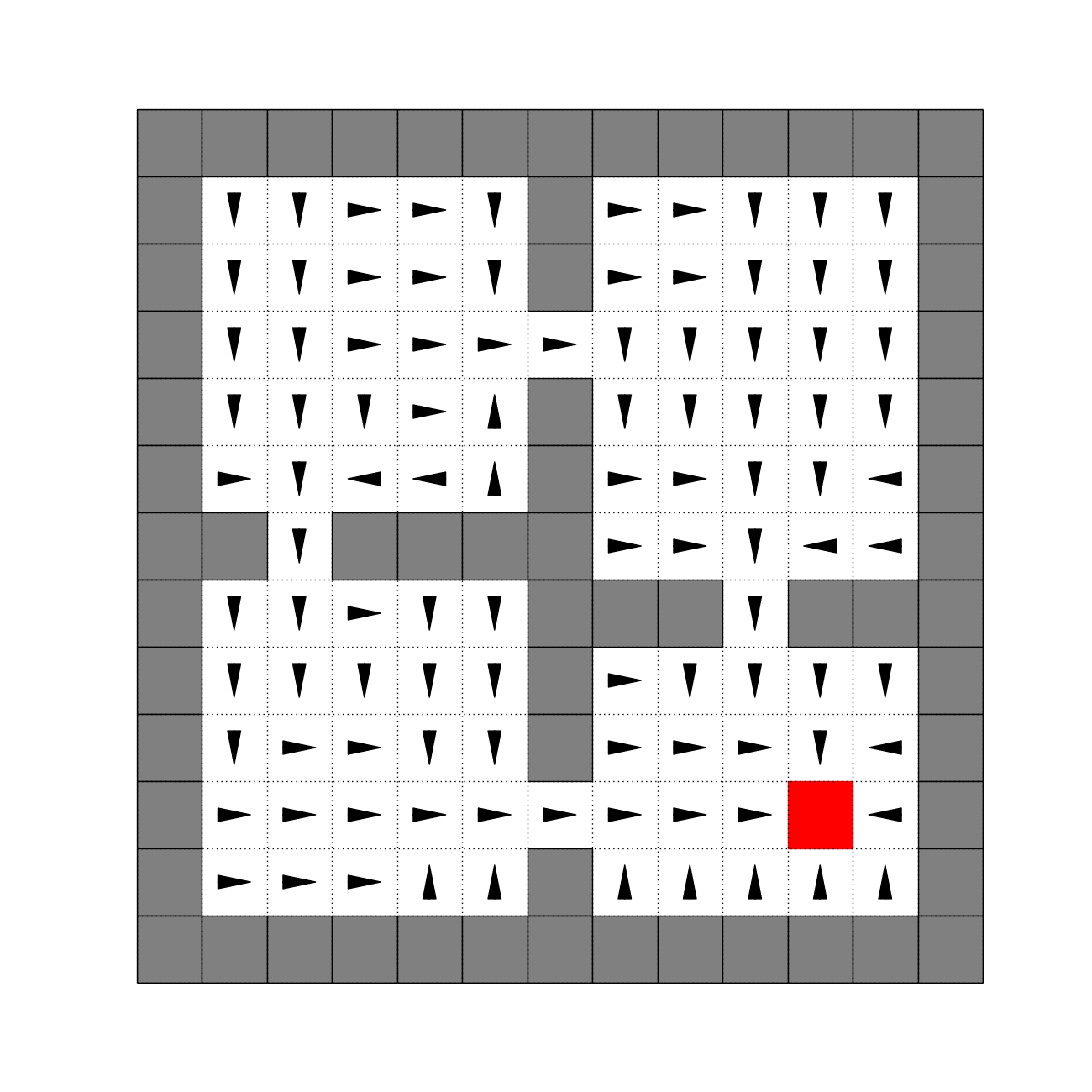}
    \end{subfigure}
    \begin{subfigure}[b]{0.24\textwidth}
    \center
        \includegraphics[width=0.8\columnwidth]{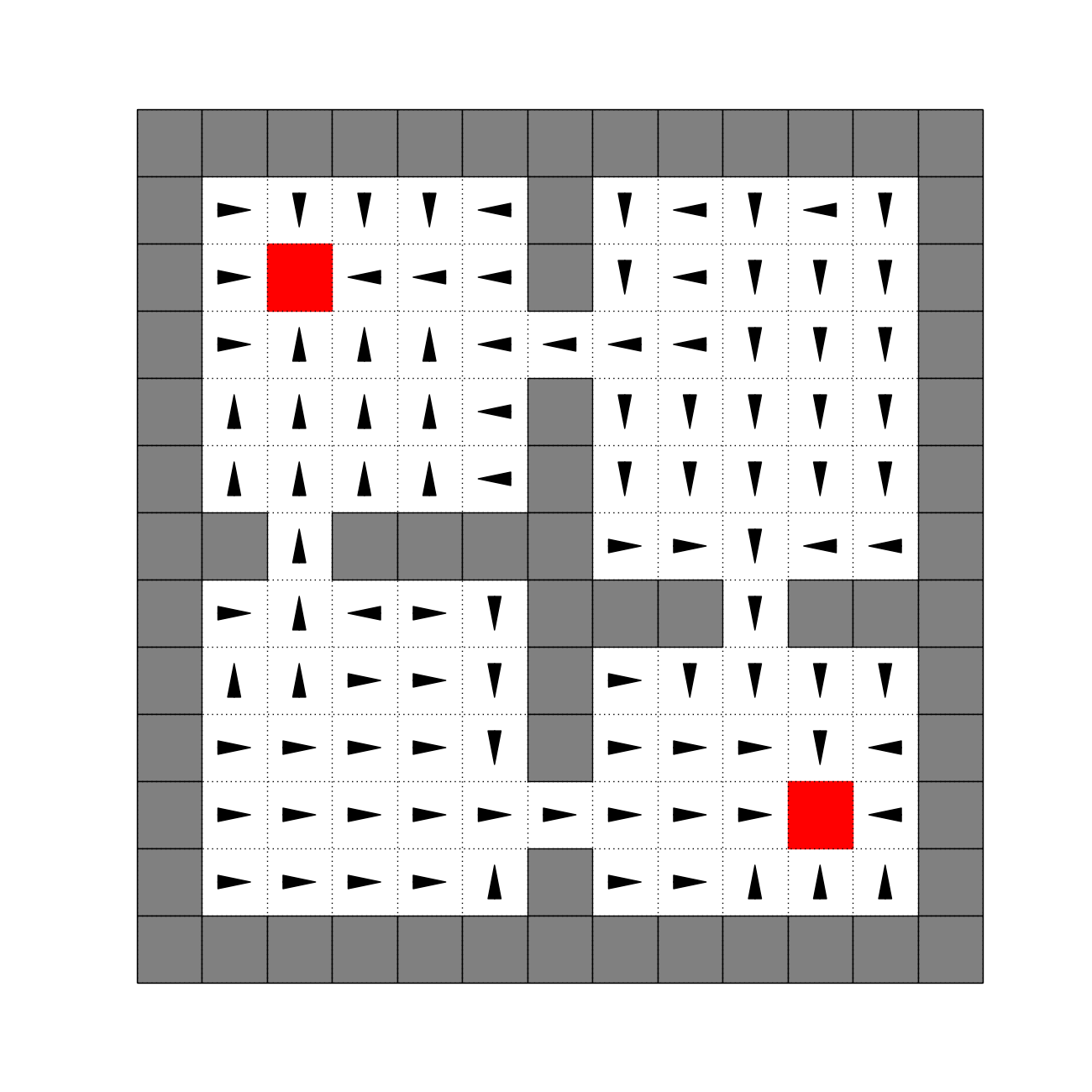}
    \end{subfigure}
    \caption{Options obtained from the four smallest eigenvectors in the 4-room domain. Action \emph{terminate} is depicted in red.}
    \label{fig:options_4room}
\end{figure*}

\begin{figure*}[t]
    \centering
    \begin{subfigure}[b]{0.2\textwidth}
    \center
        \includegraphics[width=\columnwidth]{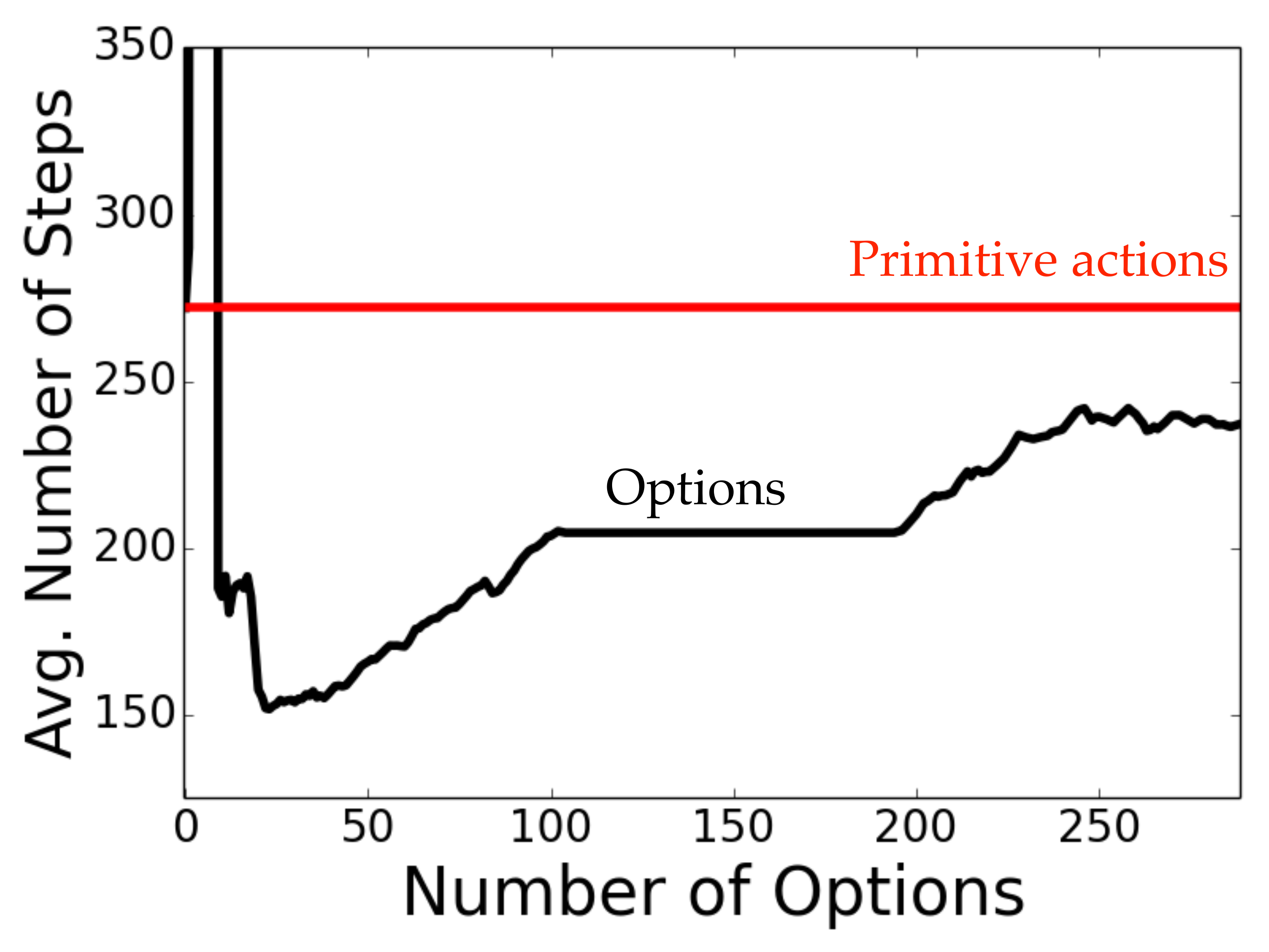}
        \caption{10$\times$10 grid}
    \end{subfigure}
    ~~~~~~~~
    \begin{subfigure}[b]{0.2\textwidth}
    \center
        \includegraphics[width=\columnwidth]{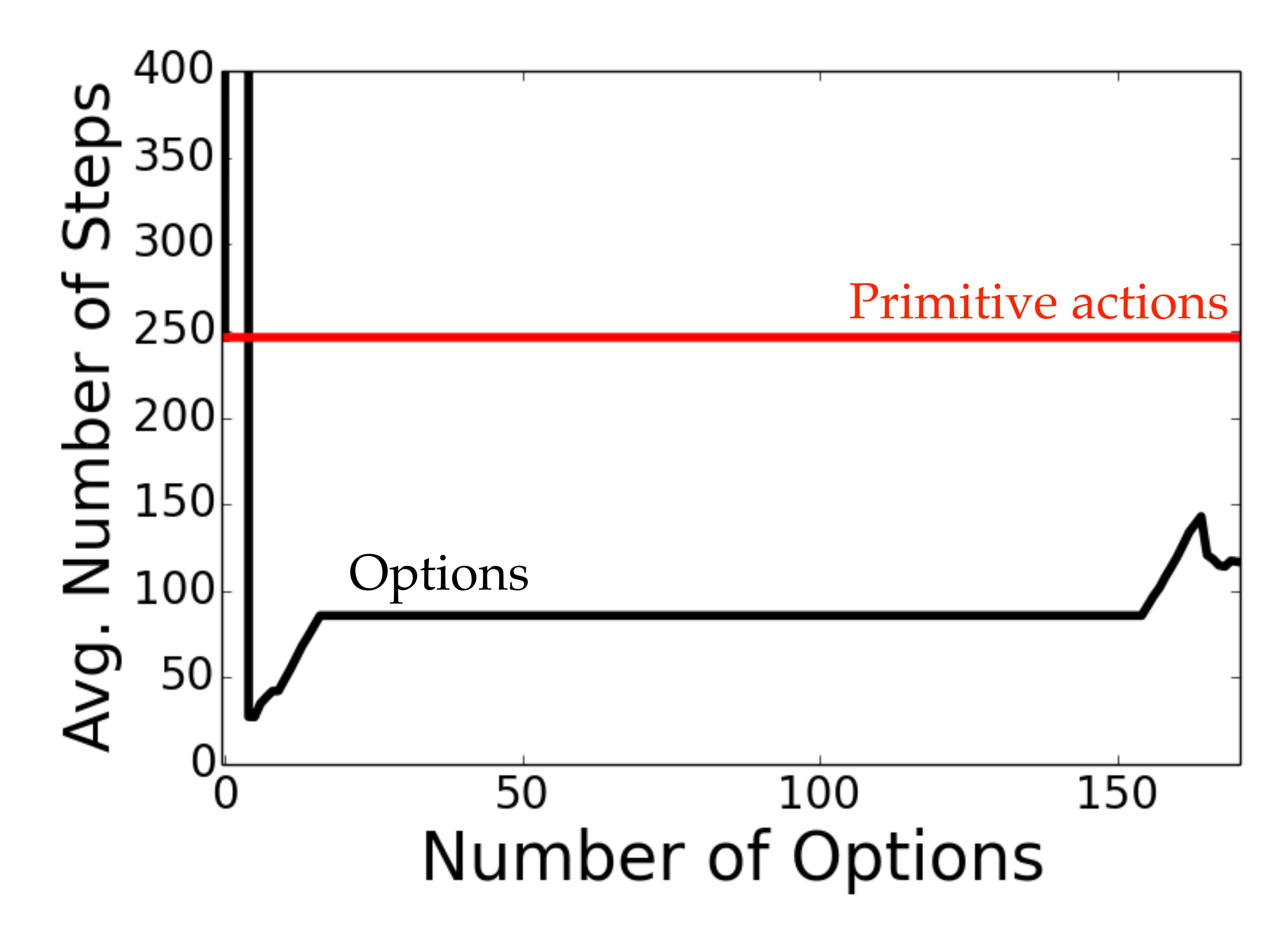}
        \caption{I-Maze}
    \end{subfigure}
    ~~~~~~~~
    \begin{subfigure}[b]{0.2\textwidth}
    \center
        \includegraphics[width=\columnwidth]{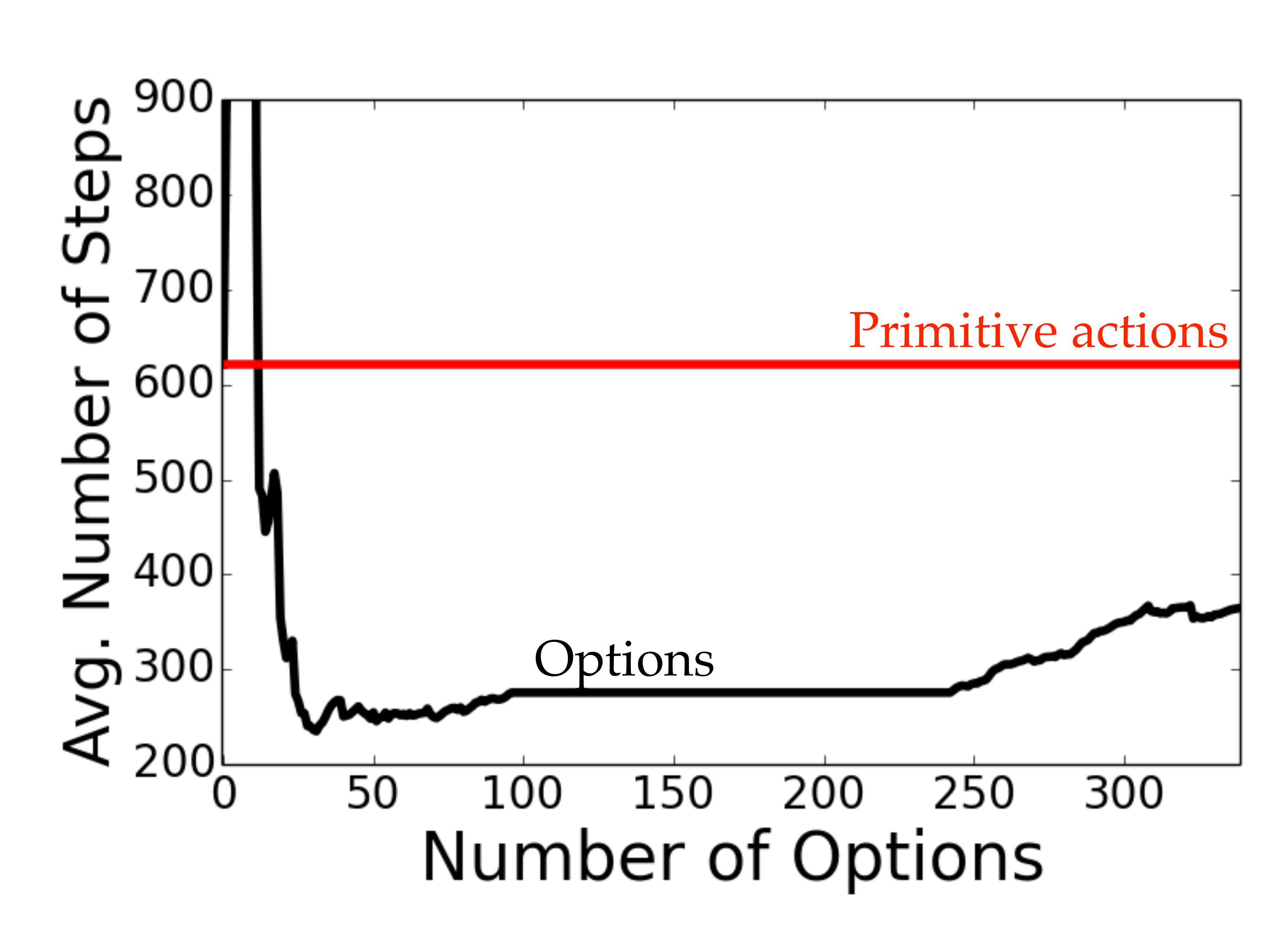}
        \caption{4-room domain}
    \end{subfigure}
    ~~~~~~~~
    \begin{subfigure}[b]{0.2\textwidth}
    \center
        \includegraphics[width=\columnwidth]{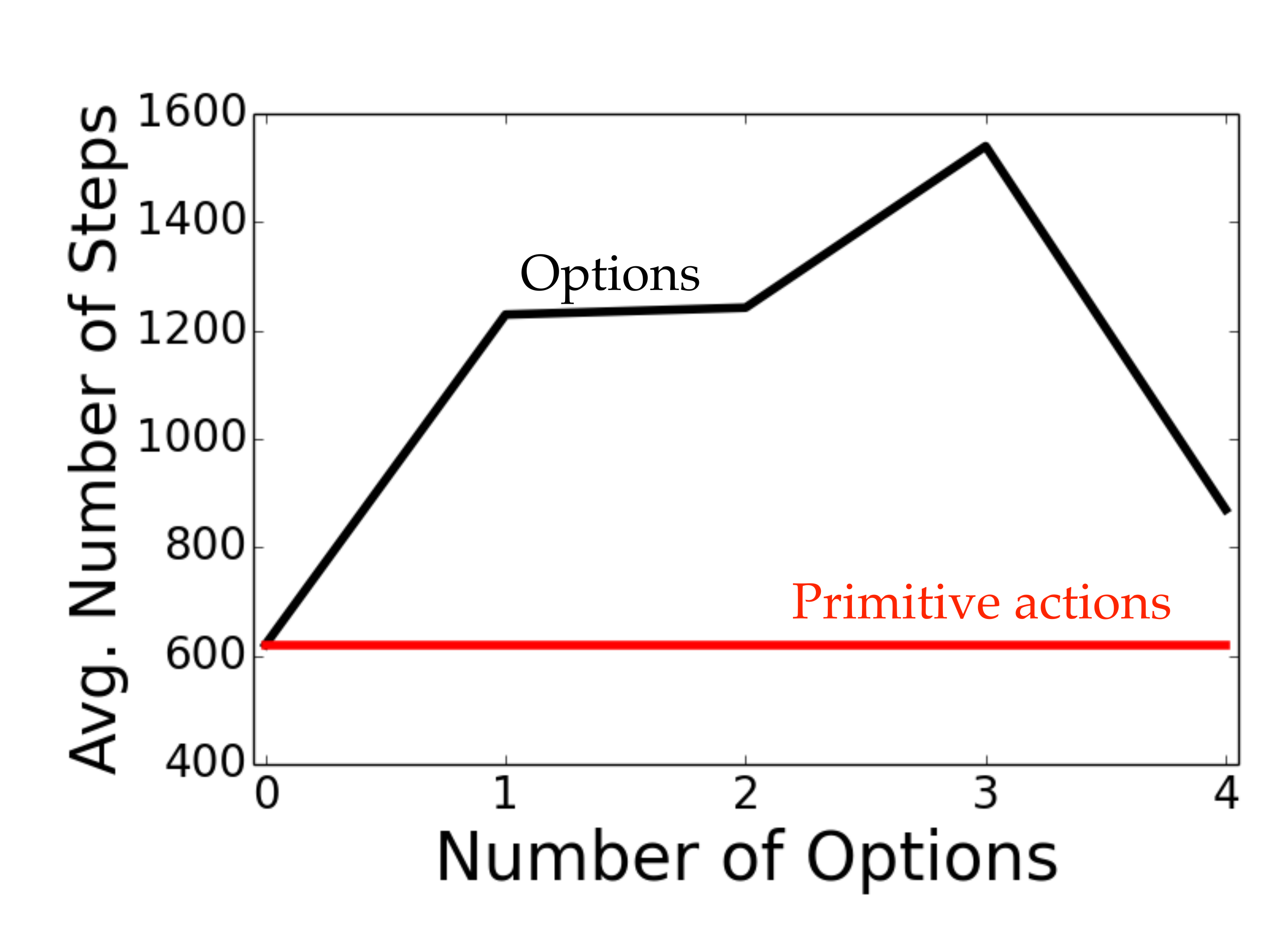}
        \caption{Bottleneck options}\label{fig:expl_bottlenecks}
    \end{subfigure}
    \caption{Expected number of steps between any two states when following a random walk. Figure~\ref{fig:expl_bottlenecks} shows the performance of options that look for doorways in the 4-room domain.}\label{fig:avg_distance}
\end{figure*}

Eigenoptions do not necessarily look for bottleneck states, allowing us to apply our algorithm in many environments in which there are no obvious, or meaningful, bottlenecks. We discover meaningful options in these environments, such as walking down a corridor, or going to the corners of an open room. Interestingly, doorways are not the first options we discover in the 4-room domain (the fifth eigenoption is the first to terminate at the entrance of a doorway). In the next sections we provide empirical evidence that eigenoptions are useful, and often more so than bottleneck options.

\subsection{Exploration}

A major challenge for agents to explore an environment is to be decisive, avoiding the dithering commonly observed in random walks~\cite{Machado16,Osband16}. Options provide such decisiveness by operating in a higher level of abstraction. Agents performing a random walk, when equipped with options, are expected to cover larger distances in the state space, navigating back and forth between subgoals instead of dithering around the starting state. However, options need to satisfy two conditions to improve exploration: (1) they have to be available in several parts of the state space, ensuring the agent always has access to many different options; and (2) they have to operate at different time scales. For instance, in the 4-room domain, it is unlikely an agent randomly selects enough primitive actions leading it to a corner if all options move the agent between doorways. An important result in this section is to show that it is very unlikely for an agent to explore the whole environment if it keeps going back and forth between similar high-level goals.

Eigenoptions satisfy both conditions. As demonstrated in Section~\ref{sec:discovered_options}, eigenoptions are often defined in the whole state space, allowing sequencing. Moreover, PVFs can be seen as a ``frequency'' basis, with different PVFs being associated with different frequencies~\cite{Mahadevan07}. The corresponding eigenoptions also operate at different frequencies, with the length of a trajectory until termination varying. This behavior can be seen when comparing the second and fourth eigenoptions in the $10 \times 10$ grid (Figure~\ref{fig:options_opengrid}). The fourth eigenoption terminates, on expectation, twice as often as the second eigenoption.

In this section we show that eigenoptions improve exploration. We do so by introducing a new metric, which we call \emph{diffusion time}. Diffusion time encodes the expected number of steps required to navigate between two states randomly chosen in the MDP while following a random walk. A small expected number of steps implies that it is more likely that the agent will reach all states with a random walk. We discuss how this metric can be computed in the Appendix.

\begin{figure*}[t]
    \centering
    \begin{subfigure}[b]{0.28\textwidth}
        \includegraphics[width=\columnwidth]{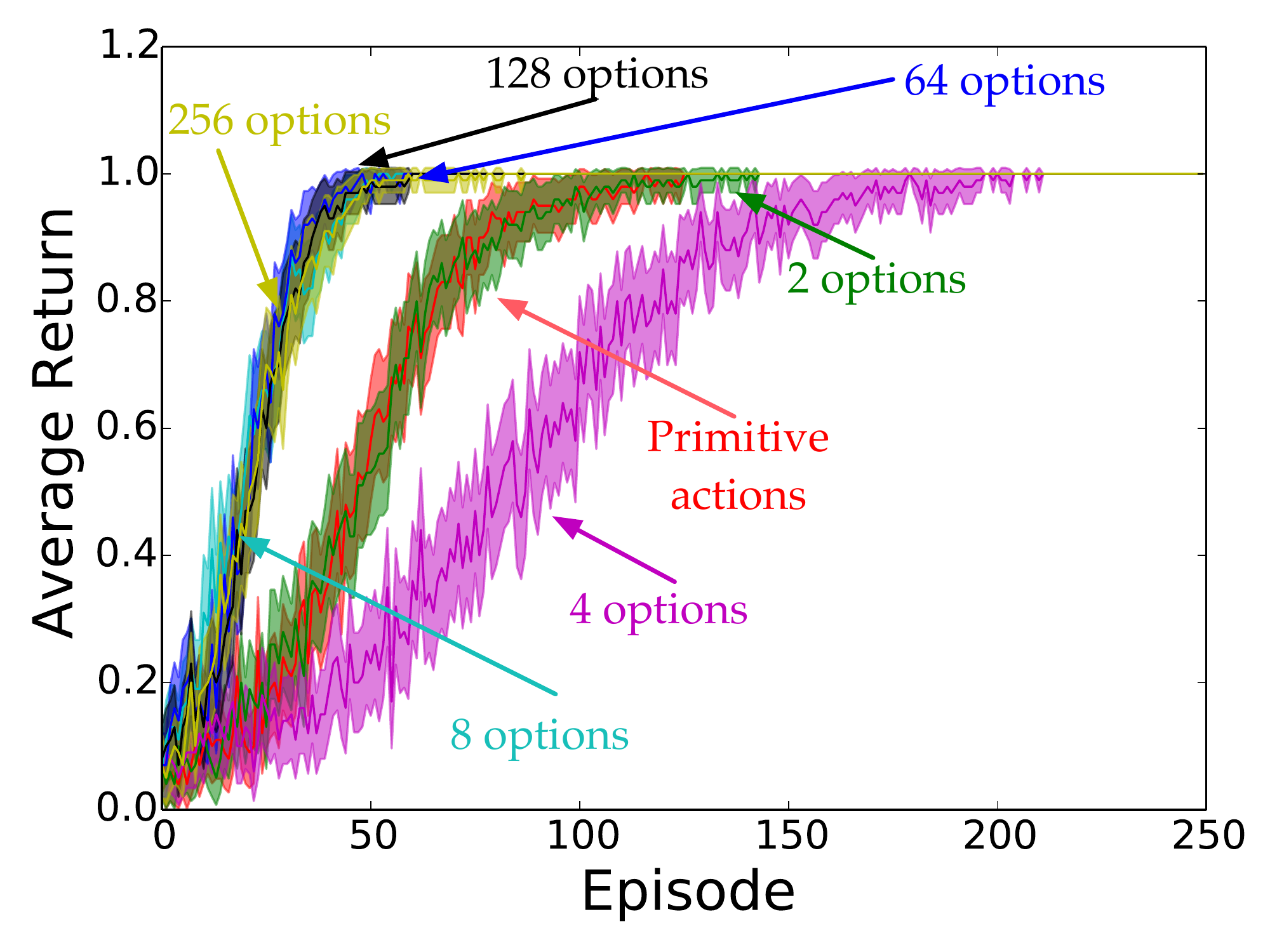}
        \caption{10$\times$10 grid}
    \end{subfigure}
    \begin{subfigure}[b]{0.28\textwidth}
        \includegraphics[width=\columnwidth]{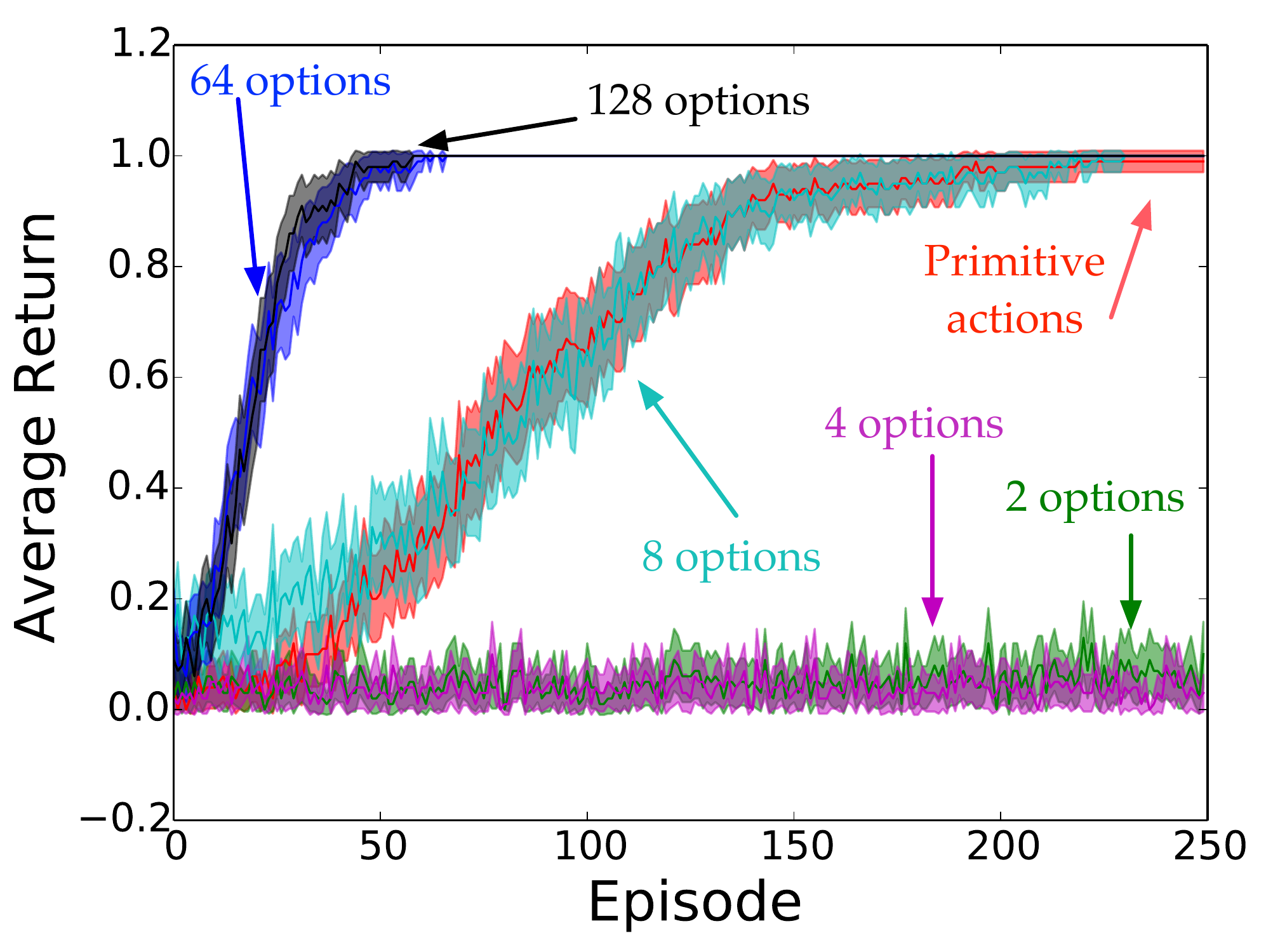}
        \caption{I-Maze}
    \end{subfigure}
    \begin{subfigure}[b]{0.28\textwidth}
        \includegraphics[width=\columnwidth]{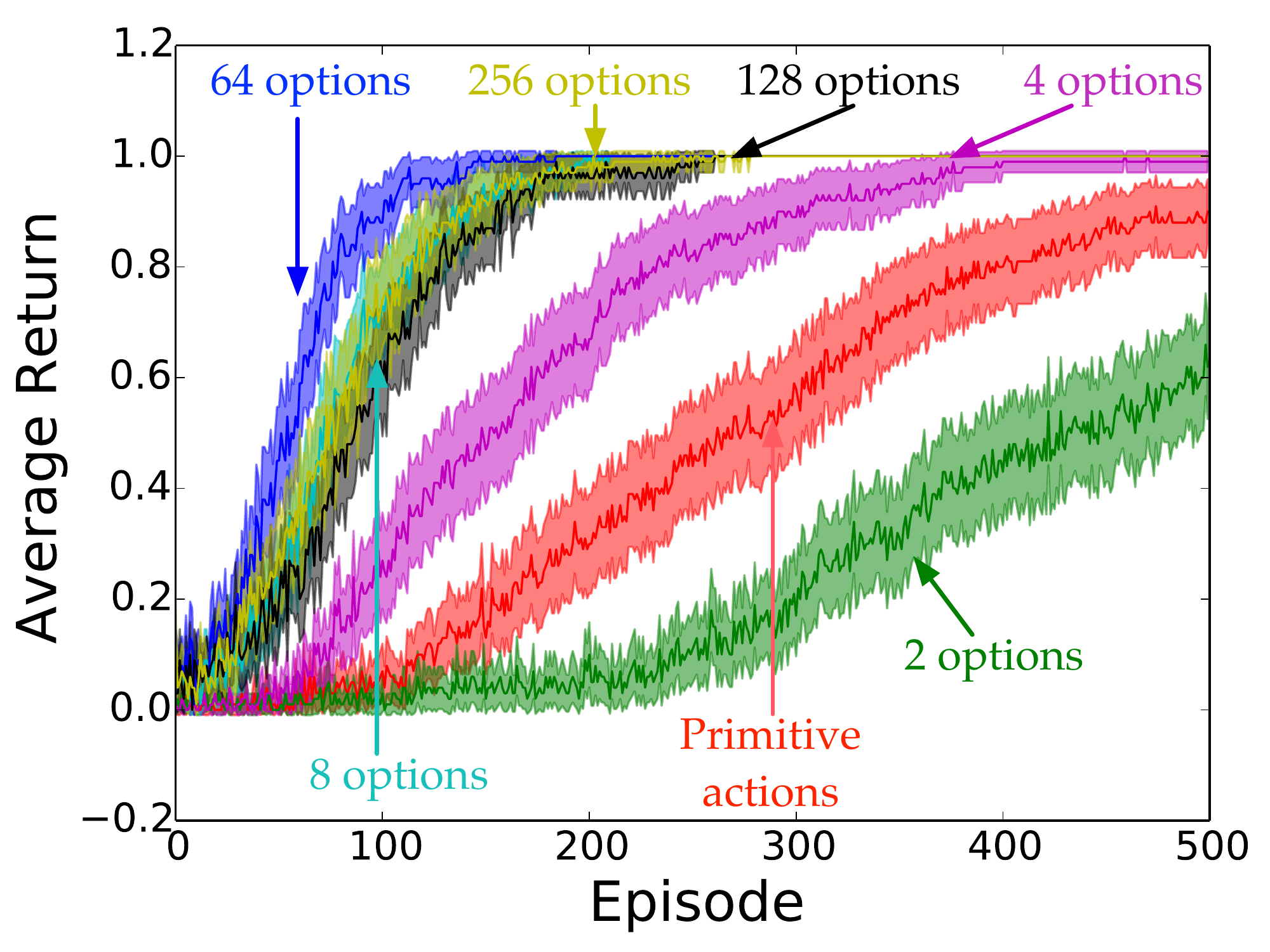}
        \caption{4-room domain}
    \end{subfigure}
    \caption{The agents' performance accumulating reward as options are added to the action set in their behavior policy. These results use the eigenpurposes directly obtained from the eigendecomposition as well as their negation.}\label{fig:acc_reward_neg}
\end{figure*}

Figure~\ref{fig:avg_distance} depicts, for our the three environments, the diffusion time with options and the diffusion time using only primitive actions. We add options incrementally in order of increasing eigenvalue when computing the diffusion time for different sets of options.

The first options added hurt exploration, but when enough options are added, exploration is greatly improved when compared to a random walk using only primitive actions. The fact that few options hurt exploration may be surprising at first, based on the fact that few useful options are generally sought after in the literature. However, this is a major difference between using options for planning and for learning. In planning, options shortcut the agents' trajectories, pruning the search space. All other actions are still taken into consideration. When exploring, a uniformly random policy over options and primitive actions skews where agents spend their time. Options that are much longer than primitive actions reduce the likelihood that an agent will deviate much from the options' trajectories, since sampling an option may undo dozens of primitive actions. This biasing is often observed when fewer options are available.

The discussion above can be made clearer with an example. In the 4-room domain, if the only options available are those leading the agent to doorways (\emph{c.f.} Appendix), it is less likely the agent will reach the outer corners. To do so the agent would have to select enough consecutive primitive actions without sampling an option. Also, it is very likely agents will be always moving between rooms, never really exploring inside a room. These issues are mitigated with eigenoptions. The first eigenoptions lead agents to individual rooms, but other eigenoptions operate in different time scales, allowing agents to explore different parts of rooms.

Figure~\ref{fig:expl_bottlenecks} supports the intuition that options leading to bottleneck states are not sufficient, by themselves, for exploration. It shows how the diffusion time in the 4-room domain is increased when only bottleneck options are used. As in the PVF literature, the ideal number of options to be used by an agent can be seen as a model selection problem.

\subsection{Accumulating Rewards}~\label{sec:acc_reward}
\vspace{-0.6cm}

We now illustrate the usefulness of our options when the agent's goal is to accumulate reward. We also study the impact of an increasing number of options in such a task. In these experiments, the agent starts at the bottom left corner and its goal is to reach the top right corner. The agent observes a reward of $0$ until the goal is reached, when it observes a reward of $+1$. We used Q-Learning~\cite{Watkins92} ($\alpha = 0.1$, $\gamma=0.9$) to learn a policy over primitive actions. The behavior policy chooses uniformly over primitive actions and options, following them until termination. Figure~\ref{fig:acc_reward_neg} depicts, after learning for a given number of episodes, the average over 100 trials of the agents' final performance. Episodes were $100$ time steps long, and we learned for $250$ episodes in the $10 \times 10$ grid and in the I-Maze, and for $500$ episodes in the 4-room domain.

In most scenarios eigenoptions improve performance. As in the previous section, exceptions occur when only a few options are added to the agent's action set. The best results were obtained using $64$ options. Despite being an additional parameter, our results show that the agent's performance is fairly robust across different numbers of options.

Eigenoptions are task-independent by construction. Additional results in the appendix show how the same set of eigenoptions is able to speed-up learning in different tasks. In the appendix we also compare eigenoptions to random options, that is, options that use a random state as subgoal.

\section{Approximate Option Discovery}~\label{sec:ale}
\vspace{-0.6cm}

So far we have assumed that agents have access to the adjacency matrix representing the underlying MDP. However, in practical settings this is generally not true. In fact, the number of states in these settings is often so large that agents rarely visit the same state twice. These problems are generally tackled with sample-based methods and some sort of function approximation.

In this section we propose a sample-based approach for option discovery that asymptotically discovers eigenoptions. We then extend this algorithm to linear function approximation. We provide anecdotal evidence in Atari 2600 games that this relatively na\"ive sample-based approach to function approximation discovers purposeful options.

\subsection{Sample-based Option Discovery}

In the online setting, agents must sample trajectories. Naturally, one can sample trajectories until one is able to perfectly construct the MDP's adjacency matrix, as suggested by \citet{Mahadevan07}. However, this approach does not easily extend to linear function approximation. In this section we provide an approach that does not build the adjacency matrix allowing us to extend the concept of eigenpurposes to linear function approximation.

In our algorithm, a sample transition is added to a matrix $T$ if it was not previously encountered. The transition is added as the difference between the current and previous observations, \emph{i.e.}, $\bm{\phi}(s') - \bm{\phi}(s)$. In the tabular case we define $\bm{\phi}(s)$ to be the one-hot encoding of state $s$. Once enough transitions have been sampled, we perform a singular value decomposition on the matrix $T$ such that $T = U \Sigma V^{\top}$. We use the columns of $V$, which correspond to the right-eigenvectors of $T$, to generate the eigenpurposes. The intrinsic reward and the termination criterion for an eigenbehavior are the same as before.

Matrix $T$ is known as the \emph{incidence matrix}. If all transitions in the graph are sampled once, for tabular representations, this algorithm discovers the same options we obtain with the combinatorial Laplacian. The theorem below states the equivalence between the obtained eigenpurposes.


\begin{theorem}
Consider the SVD of $T = U_T \Sigma_T V_T^{\top}$, with each row of $T$ consisting of the difference between observations, \emph{i.e.}, $\bm{\phi}(s') - \bm{\phi}(s)$. In the tabular case, if all transitions in the MDP have been sampled once, the orthonormal eigenvectors of $L$ are the columns of $V_T^{\top}$.
\end{theorem}
\vspace{-0.15cm}
\begin{proof}
Given the SVD decomposition of a matrix $A = U \Sigma V^\top$, the columns of $V$ are the eigenvectors of $A^\top A$~\cite{Strang05}. We know that $T^\top T = 2 L$, where $L~=~D - W$ (Lemma~5.1, \emph{c.f.} Appendix). Thus, the columns of $V_T$ are the eigenvectors of $T^\top T$, which can be rewritten as $2 (D - W)$. Therefore, the columns of $V_T$ are also the eigenvectors of $L$.
\end{proof}

\vspace{-0.05cm}

There is a trade-off between reconstructing the adjacency matrix and constructing the incidence matrix. In MDPs in which states are sparsely connected, such as the I-Maze, the latter is preferred since it has fewer transitions than states. However, what makes this result interesting is the fact that our algorithm can be easily generalized to linear function approximation.

\subsection{Function Approximation}

An adjacency matrix is not very useful when the agent has access only to features of the state. However, we can use the intuition about the incidence matrix to propose an algorithm compatible with linear function approximation.

In fact, to apply the algorithm proposed in the previous section, we just need to define what constitutes a new transition. We define two vectors, ${\bf t}$ and ${\bf t'}$, to be identical if and only if ${\bf t - t' = 0}$. We then use a \emph{set} data structure to avoid duplicates when storing $\bm{\phi}(s') - \bm{\phi}(s)$. This is a na\"ive approach, but it provides encouraging evidence eigenoptions generalize to linear function approximation. We expect more involved methods to perform even better.


\begin{figure}[t]
    \centering
    \includegraphics[width=0.7\columnwidth]{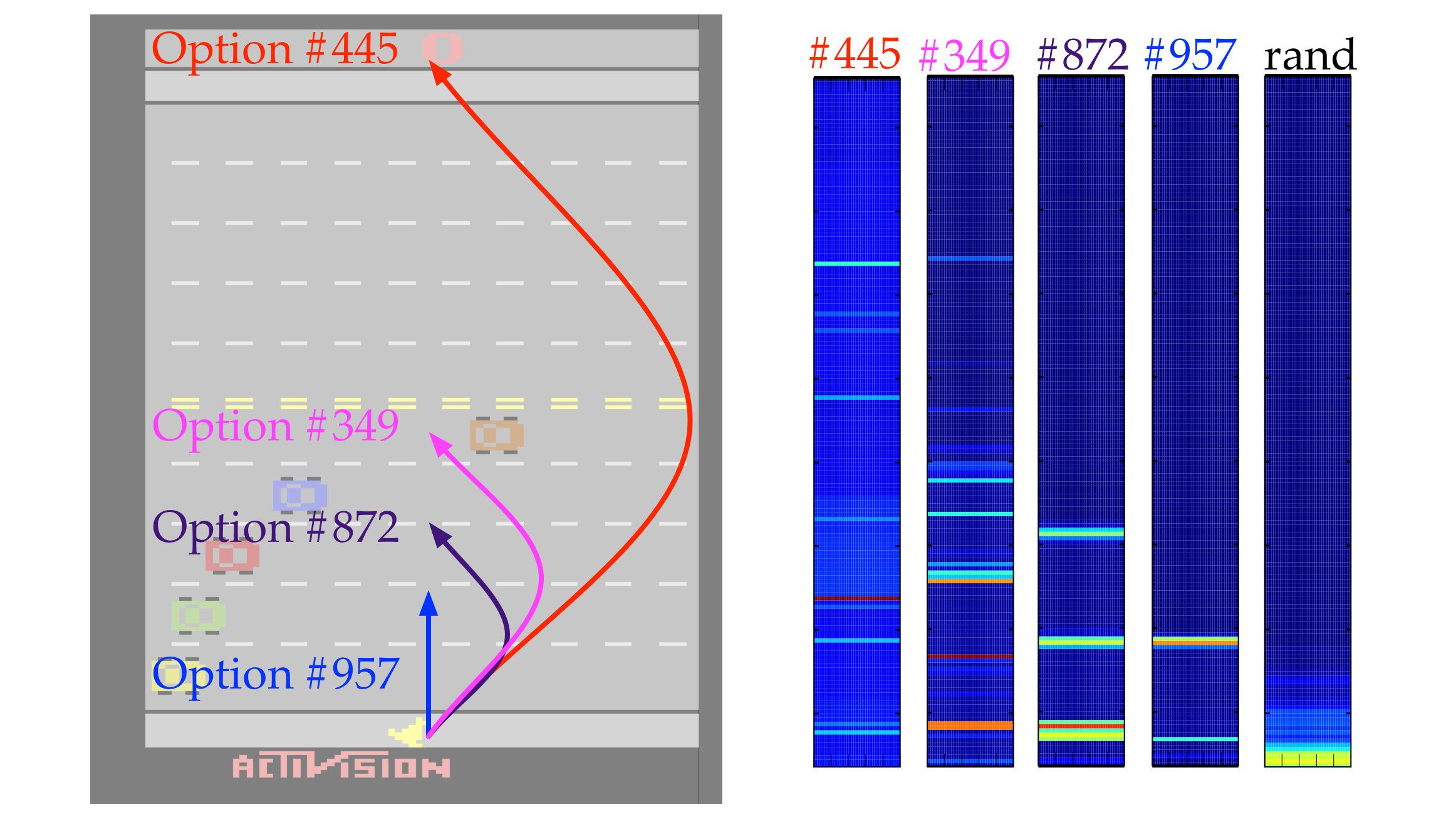}
    \caption{Options in \textsc{Freeway} (\emph{c.f.} text for details).}
    \label{fig:freeway}    
\end{figure}

We tested our method in the ALE~\cite{Bellemare13}. The agent's representation consists of the emulator's RAM state ($1{,}024$ bits). The final incidence matrix in which we ran the SVD had $25{,}000$ rows, which we sampled uniformly from the set of observed transitions. We provide further details of the experimental setup in the appendix.

In the tabular case we start selecting eigenpurposes generated by the eigenvectors with smallest eigenvalue, because these are the ``smoothest'' ones. However, it is not clear such intuition holds here because we are in the function approximation setting and the matrix of transitions does not contain all possible transitions. Therefore, we analyzed, for each game, all $1{,}024$ discovered options.

We approximate these options greedily ($\gamma\!=\!0$) with the ALE emulator's look-ahead. The next action $a'$ for an eigenpurpose ${\bf e}$ is selected as $\argmax_{b \in \mathscr{A}} \int_{s'} p(s'|s,b)\ r^{\bf e}_i(s, s')$.

Even with such a myopic action selection mechanism we were able to obtain options that clearly demonstrate intent. In \textsc{Freeway}, a game in which a chicken is expected to cross the road while avoiding cars, we observe options in which the agent clearly wants to reach a specific lane in the street. Figure~\ref{fig:freeway} (left) depicts where the chicken tends to be when the option is executed. On the right we see a histogram representing the chicken's height during an episode. We can clearly see how the chicken's height varies for different options, and how a random walk over primitive actions (\emph{rand}) does not explore the environment properly. Remarkably, option \emph{\#445} scores $28$ points at the end of the episode, without ever explicitly taking the reward signal into consideration. This performance is very close to those obtained by state-of-the-art algorithms.

In \textsc{Montezuma's Revenge}, a game in which the agent needs to navigate through a room to pickup a key so it can open a door, we also observe the agent having the clear intent of reaching particular positions on the screen, such as staircases, ropes and doors (Figure~\ref{fig:montezuma_revenge}). Interestingly, the options we discover are very similar to those handcrafted by \citet{Kulkarni16} when evaluating the usefulness of options to tackle such a game. A video of the highlighted options can be found online.\footnote{\url{https://youtu.be/2BVicx4CDWA}}

\begin{figure}[t]
    \centering
    \includegraphics[width=0.52\columnwidth]{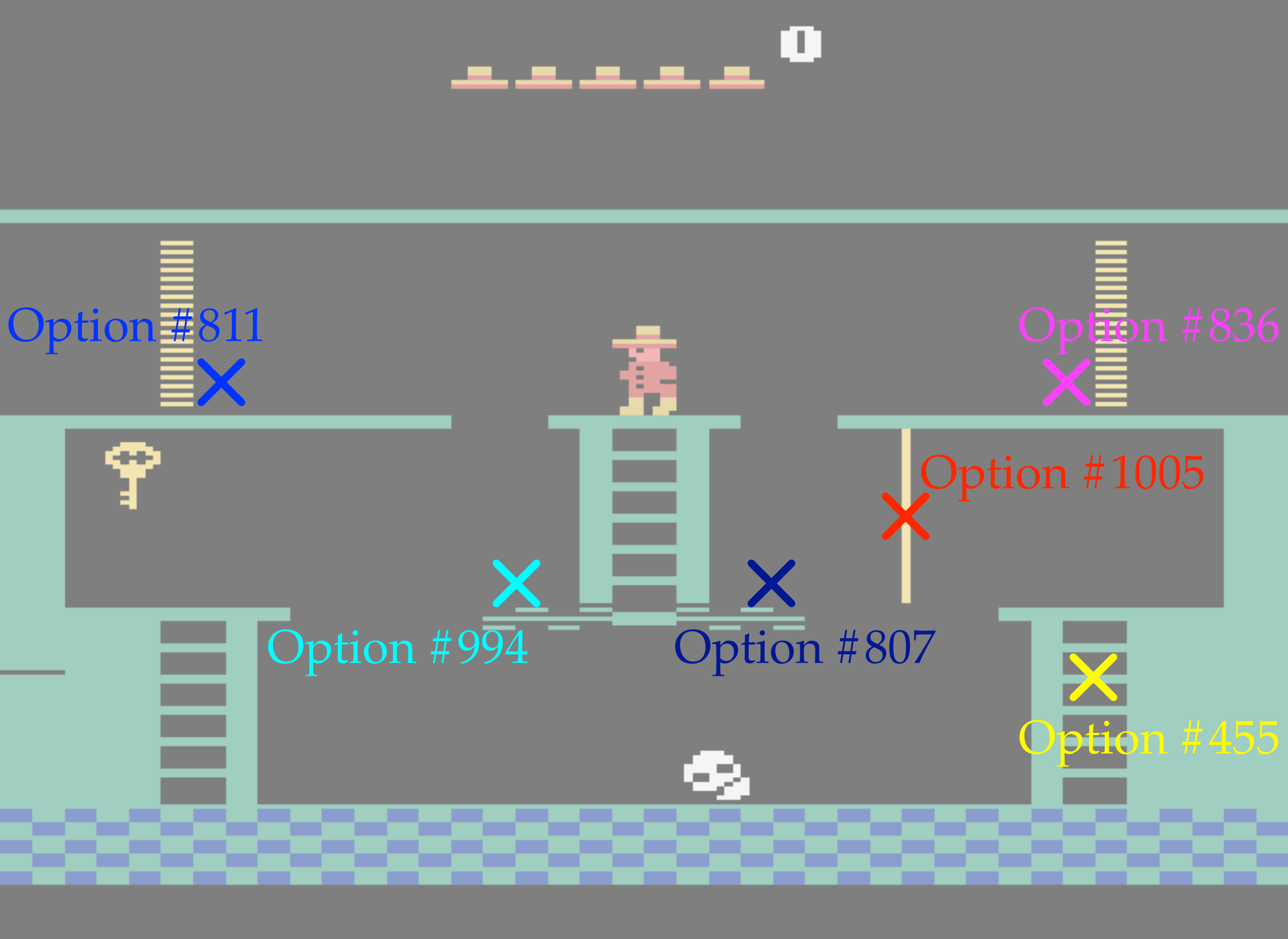}
    \caption{Options in \textsc{Montezuma's Rev.} (\emph{c.f.} text for details).}
    \label{fig:montezuma_revenge}
\end{figure}

\section{Related Work}

Most algorithms for option discovery can be seen as \emph{top-down} approaches. Agents use trajectories leading to informative rewards\footnote{We define an informative reward to be the signal that informs the agent it has reached a goal. For example, when trying to escape from a maze, we consider $0$ to be an informative reward if the agent observes rewards of value $-1$ in every time step it is inside the maze. A different example is a positive reward observed by an agent that typically observes rewards of value $0$.} as a starting point, decomposing and refining them into options. There are many approaches based on this principle, such as methods that use the observed rewards to generate intrinsic rewards leading to new value functions (\emph{e.g.}, \citeeg{McGovern01}; \citeeg{Menache02}; \citeeg{Konidaris09}), methods that use the observed rewards to climb a gradient (\emph{e.g.}, \citeeg{Mankowitz16}; \citeeg{Vezhnevets16}; \citeeg{Bacon17}), or to do probabilistic inference~\cite{Daniel16}. However, such approaches are not applicable in large state spaces with sparse rewards. If informative rewards are unlikely to be found by an agent using only primitive actions, requiring long or specific sequences of actions, options are equally unlikely to be discovered.

Our algorithm can be seen as a \emph{bottom-up} approach, in which options are constructed before the agent observes any informative reward. These options are composed to generate the desired policy. Options discovered this way tend to be independent of an agent's intention, and are potentially useful in many different tasks \cite{Gregor16}. Such options can also be seen as being useful for exploration by allowing agents to commit to a behavior for an extended period of time~\cite{Machado16}. Among the approaches to discover options without using extrinsic rewards are the use of global or local graph centrality measures~\cite{Simsek04,Simsek05,Simsek08} and clustering of states~\cite{Mannor04,Bacon13, Lakshminarayanan16}. Interestingly, \citet{Simsek05} and \citet{Lakshminarayanan16} also use the graph Laplacian in their algorithm, but to identify bottleneck states.

\citet{Baranes13} and \citet{Moulin-Frier13} show how one can build policies to explicitly assist agents to explore the environment. The proposed algorithms self-generate subgoals in order to maximize learning progress. The policies built can be seen as options. Recently, \citet{Solway14} proved that ``optimal hierarchy minimizes the geometric mean number of trial-and-error attempts necessary for the agent to discover the optimal policy for any selected task (...)''. Our experiments confirm this result, although we propose \emph{diffusion time} as a different metric to evaluate how options improve exploration.

The idea of discovering options by learning to control parts of the environment is also related to our work. Eigenpurposes encode different rates of change in the agent’s representation of the world, while the corresponding options aim at maximizing such change. Others have also proposed ways to discover options based on the idea of learning to control the environment. \citet{Hengst02}, for instance, proposes an algorithm that explicitly models changes in the variables that form the agent's representation. Recently, \citet{Gregor16} proposed an algorithm in which agents discover options by maximizing a notion of empowerment~\cite{Salge14}, where the agent aims at getting to states with a maximal set of available intrinsic options.

Continual Curiosity driven Skill Acquisition (CCSA) \cite{Kompella15} is the closest approach to ours. CCSA also discovers skills that maximize an intrinsic reward obtained by some extracted representation. While we use PVFs, CCSA uses Incremental Slow Feature Analysis (SFA)~\cite{Kompella11} to define the intrinsic reward function. 
\citet{Sprekeler11} has shown that, given a specific choice of adjacency function, PVFs are equivalent to SFA~\cite{WiskottS02}. SFA becomes an approximation of PVFs if the function space used in the SFA does not allow arbitrary mappings from the observed data to an embedding. 
Our method differs in how we define the initiation and termination sets, as well as in the objective being maximized. CCSA acquires skills that produce a large variation in the slow-feature outputs, leading to options that seek for bottlenecks. Our approach does not seek for  bottlenecks, focusing on traversing different directions of the learned representation.

\section{Conclusion}

Being able to properly abstract MDPs into SMDPs can reduce the overall expense of learning~\cite{Sutton99,Solway14}, mainly when the learned options are reused in multiple tasks. On the other hand, the wrong hierarchy can hinder the agents' learning process, moving the agent away from desired goal states. 
Current algorithms for option discovery
often depend on an initial informative reward signal, which may not be readily available in large MDPs. In this paper, we introduced an approach that is effective in different environments, for a multitude of tasks.

Our algorithm uses the graph Laplacian, being directly related to the concept of proto-value functions. The learned representation informs the agent what are meaningful options to be sought after. The discovered options can be seen as traversing each one of the dimensions in the learned representation. We believe successful algorithms in the future will be able to simultaneously discover representations and options. Agents will use their learned representation to discover options, which will be used to further explore the environment, improving the agent's representation.

Interestingly, the options first discovered by our approach do not necessarily find bottlenecks, which are commonly sought after. In this paper we showed how bottleneck options can hinder exploration strategies if naively added to the agent's action set, and how the options we discover can help an agent to explore. Also, we have shown how the discovered options can be used to accumulate reward in a multitude of tasks, leveraging their exploratory properties.

There are several exciting avenues for future work. As noted, SFA can be seen as an approximation to PVFs. It would be interesting to compare such an approach to eigenoptions. It would also be interesting to see if the options we discover can be generated incrementally and with incomplete graphs.
Finally, one can also imagine extensions to the proposed algorithm where a hierarchy of options is built. 

\section*{Acknowledgements}
The authors would like to thank Will Dabney, R\'emi Munos and Csaba Szepesv\'ari for useful discussions. This work was supported by grants from Alberta Innovates Technology Futures and the Alberta Machine Intelligence Institute (Amii). Computing resources were provided by Compute Canada through CalculQu\'ebec.

\balance

\bibliography{option_discovery}

\begin{thebibliography}{41}
\providecommand{\natexlab}[1]{#1}
\providecommand{\url}[1]{\texttt{#1}}
\expandafter\ifx\csname urlstyle\endcsname\relax
  \providecommand{\doi}[1]{doi: #1}\else
  \providecommand{\doi}{doi: \begingroup \urlstyle{rm}\Url}\fi

\bibitem[Bacon(2013)]{Bacon13}
Bacon, Pierre-Luc.
\newblock {On the Bottleneck Concept for Options Discovery: Theoretical
  Underpinnings and Extension in Continuous State Spaces}.
\newblock Master's thesis, McGill University, 2013.

\bibitem[Bacon et~al.(2017)Bacon, Harb, and Precup]{Bacon17}
Bacon, Pierre{-}Luc, Harb, Jean, and Precup, Doina.
\newblock The option-critic architecture.
\newblock In \emph{Proceedings of the National Conference on Artificial
  Intelligence (AAAI)}, 2017.

\bibitem[Baranes \& Oudeyer(2013)Baranes and Oudeyer]{Baranes13}
Baranes, Adrien and Oudeyer, Pierre{-}Yves.
\newblock Active learning of inverse models with intrinsically motivated goal
  exploration in robots.
\newblock \emph{{Robotics and Autonomous Systems}}, 61\penalty0 (1):\penalty0
  49--73, 2013.

\bibitem[Bellemare et~al.(2013)Bellemare, Naddaf, Veness, and
  Bowling]{Bellemare13}
Bellemare, Marc~G., Naddaf, Yavar, Veness, Joel, and Bowling, Michael.
\newblock {The Arcade Learning Environment: An Evaluation Platform for General
  Agents}.
\newblock \emph{Journal of Artificial Intelligence Research}, 47:\penalty0
  253--279, 2013.

\bibitem[Bellman(1957)]{Bellman57}
Bellman, Richard~E.
\newblock \emph{{Dynamic Programming}}.
\newblock Princeton University Press, Princeton, NJ, 1957.

\bibitem[\c{S}im\c{s}ek \& Barto(2004)\c{S}im\c{s}ek and Barto]{Simsek04}
\c{S}im\c{s}ek, \"{O}zg\"{u}r and Barto, Andrew~G.
\newblock {Using Relative Novelty to Identify Useful Temporal Abstractions in
  Reinforcement Learning}.
\newblock In \emph{Proceedings of the International Conference on Machine
  Learning (ICML)}, 2004.

\bibitem[\c{S}im\c{s}ek \& Barto(2008)\c{S}im\c{s}ek and Barto]{Simsek08}
\c{S}im\c{s}ek, {\"{O}}zg{\"{u}}r and Barto, Andrew~G.
\newblock {Skill Characterization Based on Betweenness}.
\newblock In \emph{Proceedings of Advances in Neural Information Processing
  Systems (NIPS)}, 2008.

\bibitem[\c{S}im\c{s}ek et~al.(2005)\c{S}im\c{s}ek, Wolfe, and Barto]{Simsek05}
\c{S}im\c{s}ek, \"Ozg\"ur, Wolfe, Alicia~P., and Barto, Andrew~G.
\newblock {Identifying Useful Subgoals in Reinforcement Learning by Local Graph
  Partitioning}.
\newblock In \emph{Proceedings of the International Conference on Machine
  Learning (ICML)}, 2005.

\bibitem[Daniel et~al.(2016)Daniel, van Hoof, Peters, and Neumann]{Daniel16}
Daniel, Christian, van Hoof, Herke, Peters, Jan, and Neumann, Gerhard.
\newblock {Probabilistic Inference for Determining Options in Reinforcement
  Learning}.
\newblock \emph{Machine Learning}, 104\penalty0 (2):\penalty0 337--357, 2016.

\bibitem[Dietterich(2000)]{Dietterich00}
Dietterich, Thomas~G.
\newblock {Hierarchical Reinforcement Learning with the {MAXQ} Value Function
  Decomposition}.
\newblock \emph{Journal of Artificial Intelligence Research {(JAIR)}},
  13:\penalty0 227--303, 2000.

\bibitem[Gregor et~al.(2016)Gregor, Rezende, and Wierstra]{Gregor16}
Gregor, Karol, Rezende, Danilo, and Wierstra, Daan.
\newblock {Variational Intrinsic Control}.
\newblock \emph{CoRR}, abs/1611.07507, 2016.

\bibitem[Gross \& Yellen(2006)Gross and Yellen]{Gross06}
Gross, Jonathan~L. and Yellen, Jay.
\newblock \emph{{Graph Theory and Its Applications}}.
\newblock Chapman and Hall/CRC, 2 edition, 2006.

\bibitem[Hengst(2002)]{Hengst02}
Hengst, Bernhard.
\newblock {Discovering Hierarchy in Reinforcement Learning with HEXQ}.
\newblock In \emph{Proceedings of the International Conference on Machine
  Learning (ICML)}, 2002.

\bibitem[Kompella et~al.(2011)Kompella, Luciw, and Schmidhuber]{Kompella11}
Kompella, Varun~Raj, Luciw, Matthew~D., and Schmidhuber, J{\"{u}}rgen.
\newblock {Incremental Slow Feature Analysis}.
\newblock In \emph{Proceedings of the International Joint Conference on
  Artificial Intelligence (IJCAI)}, pp.\  1354--1359, 2011.

\bibitem[Kompella et~al.(In Press)Kompella, Stollenga, Luciw, and
  Schmidhuber]{Kompella15}
Kompella, Varun~Raj, Stollenga, Marijn, Luciw, Matthew, and Schmidhuber,
  Juergen.
\newblock {Continual Curiosity-Driven Skill Acquisition from High-Dimensional
  Video Inputs for Humanoid Robots}.
\newblock \emph{{Artificial Intelligence}}, In Press.
\newblock ISSN 0004-3702.
\newblock Available online 12 February 2015.

\bibitem[Konidaris \& Barto(2009)Konidaris and Barto]{Konidaris09}
Konidaris, George and Barto, Andrew.
\newblock {Skill Discovery in Continuous Reinforcement Learning Domains using
  Skill Chaining}.
\newblock In \emph{Proceedings of Advances in Neural Information Processing
  Systems (NIPS)}, pp.\  1015--1023, 2009.

\bibitem[Kulkarni et~al.(2016)Kulkarni, Narasimhan, Saeedi, and
  Tenenbaum]{Kulkarni16}
Kulkarni, Tejas~D., Narasimhan, Karthik~R., Saeedi, Ardavan, and Tenenbaum,
  Joshua~B.
\newblock {Hierarchical Deep Reinforcement Learning: Integrating Temporal
  Abstraction and Intrinsic Motivation}.
\newblock \emph{ArXiv e-prints}, 2016.

\bibitem[Lakshminarayanan et~al.(2016)Lakshminarayanan, Krishnamurthy, Kumar,
  and Ravindran]{Lakshminarayanan16}
Lakshminarayanan, Aravind, Krishnamurthy, Ramnandan, Kumar, Peeyush, and
  Ravindran, Balaraman.
\newblock {Option Discovery in Hierarchical Reinforcement Learning using
  Spatio-Temporal Clustering}.
\newblock \emph{CoRR}, abs/1605.05359, 2016.
\newblock Presented at the ICML-16 Workshop on Abstraction in Reinforcement
  Learning.

\bibitem[Machado \& Bowling(2016)Machado and Bowling]{Machado16}
Machado, Marlos~C. and Bowling, Michael.
\newblock {Learning Purposeful Behaviour in the Absence of Rewards}.
\newblock \emph{CoRR}, abs/1410.4604, 2016.
\newblock Presented at the ICML-16 Workshop on Abstraction in Reinforcement
  Learning.

\bibitem[Mahadevan(2005)]{Mahadevan05}
Mahadevan, Sridhar.
\newblock {Proto-Value Functions: Developmental Reinforcement Learning}.
\newblock In \emph{Proceedings of the International Conference on Machine
  Learning (ICML)}, pp.\  553--560, 2005.

\bibitem[Mahadevan \& Maggioni(2007)Mahadevan and Maggioni]{Mahadevan07}
Mahadevan, Sridhar and Maggioni, Mauro.
\newblock {Proto-value Functions: {A} Laplacian Framework for Learning
  Representation and Control in Markov Decision Processes}.
\newblock \emph{{Journal of Machine Learning Research (JMLR)}}, 8:\penalty0
  2169--2231, 2007.

\bibitem[Mankowitz et~al.(2016)Mankowitz, Mann, and Mannor]{Mankowitz16}
Mankowitz, Daniel~J., Mann, Timothy~Arthur, and Mannor, Shie.
\newblock {Adaptive Skills Adaptive Partitions {(ASAP)}}.
\newblock In \emph{Proceedings of Advances in Neural Information Processing
  Systems (NIPS)}, pp.\  1588--1596, 2016.

\bibitem[Mannor et~al.(2004)Mannor, Menache, Hoze, and Klein]{Mannor04}
Mannor, Shie, Menache, Ishai, Hoze, Amit, and Klein, Uri.
\newblock {Dynamic Abstraction in Reinforcement Learning via Clustering}.
\newblock In \emph{Proceedings of the International Conference on Machine
  Learning (ICML)}, 2004.

\bibitem[McGovern \& Barto(2001)McGovern and Barto]{McGovern01}
McGovern, Amy and Barto, Andrew~G.
\newblock {Automatic Discovery of Subgoals in Reinforcement Learning using
  Diverse Density}.
\newblock In \emph{Proceedings of the International Conference on Machine
  Learning (ICML)}, 2001.

\bibitem[Menache et~al.(2002)Menache, Mannor, and Shimkin]{Menache02}
Menache, Ishai, Mannor, Shie, and Shimkin, Nahum.
\newblock {Q-Cut~-~Dynamic Discovery of Sub-goals in Reinforcement Learning}.
\newblock In \emph{Proceedings of the European Conference on Machine Learning
  (ECML)}, 2002.

\bibitem[Moulin{-}Frier \& Oudeyer(2013)Moulin{-}Frier and
  Oudeyer]{Moulin-Frier13}
Moulin{-}Frier, Cl\'ement and Oudeyer, Pierre{-}Yves.
\newblock {Exploration Strategies in Developmental Robotics: A Unified
  Probabilistic Framework}.
\newblock In \emph{Proceedings of the Joint {IEEE} International Conference on
  Development and Learning and Epigenetic Robotics (ICDL-EpiRob)}, pp.\  1--6,
  2013.

\bibitem[Oh et~al.(2016)Oh, Chockalingam, Singh, and Lee]{Oh16}
Oh, Junhyuk, Chockalingam, Valliappa, Singh, Satinder~P., and Lee, Honglak.
\newblock {Control of Memory, Active Perception, and Action in Minecraft}.
\newblock In \emph{Proceedings of the International Conference on Machine
  Learning (ICML)}, pp.\  2790--2799, 2016.

\bibitem[Osband et~al.(2016)Osband, Roy, and Wen]{Osband16}
Osband, Ian, Roy, Benjamin~Van, and Wen, Zheng.
\newblock {Generalization and Exploration via Randomized Value Functions}.
\newblock In \emph{Proceedings of the International Conference on Machine
  Learning (ICML)}, pp.\  2377--2386, 2016.

\bibitem[Precup(2000)]{Precup00}
Precup, Doina.
\newblock \emph{{Temporal Abstraction in Reinforcement Learning}}.
\newblock PhD thesis, University of Massachusetts Amherst, 2000.

\bibitem[Puterman(1994)]{Puterman94}
Puterman, Martin~L.
\newblock \emph{{Markov Decision Processes: Discrete Stochastic Dynamic
  Programming}}.
\newblock John Wiley \& Sons, Inc., New York, NY, USA, 1994.

\bibitem[Salge et~al.(2014)Salge, Glackin, and Polani]{Salge14}
Salge, Christoph, Glackin, Cornelius, and Polani, Daniel.
\newblock {Empowerment -- An Introduction}.
\newblock In \emph{{Guided Self-Organization: Inception}}, pp.\  67--114.
  Springer, 2014.

\bibitem[Solway et~al.(2014)Solway, Diuk, C\'ordova, Yee, Barto, Niv, and
  Botvinick]{Solway14}
Solway, Alec, Diuk, Carlos, C\'ordova, Natalia, Yee, Debbie, Barto, Andrew~G.,
  Niv, Yael, and Botvinick, Matthew~M.
\newblock {Optimal Behavioral Hierarchy}.
\newblock \emph{{PLOS Computational Biology}}, 10\penalty0 (8):\penalty0 1--10,
  2014.

\bibitem[Sprekeler(2011)]{Sprekeler11}
Sprekeler, Henning.
\newblock {On the Relation of Slow Feature Analysis and Laplacian Eigenmaps}.
\newblock \emph{{Neural Computation}}, 23\penalty0 (12):\penalty0 3287--3302,
  2011.

\bibitem[Strang(2005)]{Strang05}
Strang, Gilbert.
\newblock \emph{{Linear Algebra and Its Applications}}.
\newblock Brooks Cole, 2005.

\bibitem[Sutton \& Barto(1998)Sutton and Barto]{Sutton98}
Sutton, Richard~S. and Barto, Andrew~G.
\newblock \emph{{Reinforcement Learning: An Introduction}}.
\newblock MIT Press, 1998.

\bibitem[Sutton et~al.(1999)Sutton, Precup, and Singh]{Sutton99}
Sutton, Richard~S., Precup, Doina, and Singh, Satinder.
\newblock {Between MDPs and semi-MDPs: A Framework for Temporal Abstraction in
  Reinforcement Learning}.
\newblock \emph{Artificial Intelligence}, 112\penalty0 (1–2):\penalty0 181 --
  211, 1999.

\bibitem[Szepesv\'ari(2010)]{Szepesvari10}
Szepesv\'ari, Csaba.
\newblock \emph{Algorithms for Reinforcement Learning}.
\newblock Synthesis Lectures on Artificial Intelligence and Machine Learning.
  Morgan \& Claypool, 2010.

\bibitem[Vezhnevets et~al.(2016)Vezhnevets, Mnih, Osindero, Graves, Vinyals,
  Agapiou, and Kavukcuoglu]{Vezhnevets16}
Vezhnevets, Alexander, Mnih, Volodymyr, Osindero, Simon, Graves, Alex, Vinyals,
  Oriol, Agapiou, John, and Kavukcuoglu, Koray.
\newblock {Strategic Attentive Writer for Learning Macro-Actions}.
\newblock In \emph{Proceedings of Advances in Neural Information Processing
  Systems (NIPS)}, pp.\  3486--3494, 2016.

\bibitem[Watkins \& Dayan(1992)Watkins and Dayan]{Watkins92}
Watkins, Christopher J. C.~H. and Dayan, Peter.
\newblock {Technical Note: \cal Q-Learning}.
\newblock \emph{Machine Learning}, 8\penalty0 (3-4), May 1992.

\bibitem[Weber et~al.(2004)Weber, Rungsarityotin, and Schliep]{Weber04}
Weber, Marcus, Rungsarityotin, Wasinee, and Schliep, Alexander.
\newblock {Perron Cluster Analysis and Its Connection to Graph Partitioning for
  Noisy Data}.
\newblock Technical Report 04-39, ZIB, Takustr.7, 14195 Berlin, 2004.

\bibitem[Wiskott \& Sejnowski(2002)Wiskott and Sejnowski]{WiskottS02}
Wiskott, Laurenz and Sejnowski, Terrence~J.
\newblock {Slow Feature Analysis: Unsupervised Learning of Invariances}.
\newblock \emph{{Neural Computation}}, 14\penalty0 (4):\penalty0 715--770,
  2002.

\end{thebibliography}
\bibliographystyle{icml2017}

\onecolumn


\section*{Appendix: Supplementary Material}

This supplementary material contains details omitted from the main text due to space constraints. The list of contents is below:
\begin{itemize}
    \item Supporting lemmas and their respective proofs, as well as a more detailed proof of Theorem 3.1;
    \item Description of how to easily compute the \emph{diffusion time} in tabular MDPs;
    \item The options leading to bottleneck states (doorways) we used in our experiments;
    \item Performance comparisons between eigenoptions and options generated to reach randomly selected states;
    \item Demonstration of the applicability of eigenoptions in multiple tasks with a new set of experiments;
    \item Further details on the empirical setting used in the Arcade Learning Environment.
\end{itemize} 

\section*{A. Lemmas and Proofs}

\addtocounter{section}{4}

\begin{lemma}
 Suppose $(I + A)$ is a non-singular matrix, with $||A|| \leq 1$. We have:
$$||(I + A)^{-1}|| \leq \frac{1}{1 - ||A||}.$$
\end{lemma}

\begin{proof}\footnote{Our proof follows closely the proof of Parnell in lecture notes available at \url{http://www-solar.mcs.st-and.ac.uk/~clare/Lectures/num-analysis.html}.}
\begin{align*}
(I + A)(I + A)^{-1} &= I\\
I(I + A)^{-1} + A(I + A)^{-1} &= I\\
(I + A)^{-1} &= I - A(I + A)^{-1}\\
||(I + A)^{-1}|| &= ||I - A(I+A)^{-1}||\\
                 &\leq ||I|| + ||A (I + A)^{-1}|| && \text{because} \ \ ||A + B|| \leq ||A|| + ||B||\\
                 &\leq 1 + ||A||||(I + A)^{-1}|| && \text{because} \ \ ||A B|| \leq ||A|| \cdot ||B||\\
||(I + A)^{-1}|| - ||A||||(I + A)^{-1}||&\leq 1\\
(1-||A||) ||(I + A)^{-1}|| &\leq 1\\
||(I + A)^{-1}|| &\leq \frac{1}{1 - ||A||} && \text{if} \ \ ||A|| \leq 1.
\end{align*}
\end{proof}

\begin{lemma}
The induced infinity norm of $(I - \gamma T)^{-1}T$ is bounded by
$$||(I - \gamma T)^{-1}T||_\infty \leq \frac{1}{(1 - \gamma)}.$$
\end{lemma}

\begin{proof}
\begin{align*}
||(I - \gamma T)^{-1}T||_\infty & \leq ||(I - \gamma T)^{-1}||_\infty||T||_\infty  && \text{because} \ \  ||AB||_{\infty} \leq ||A||_{\infty} \cdot ||B||_{\infty}\\
||(I - \gamma T)^{-1}T||_\infty & \leq \frac{1}{1 - ||-\gamma T||_\infty} ||T||_\infty && \text{Lemma 3.1} \\
||(I - \gamma T)^{-1}T||_\infty & \leq \frac{1}{1 - \gamma ||T||_\infty} ||T||_\infty && \text{because}  \ \ ||\lambda B|| = |\lambda| ||B|| \\
||(I - \gamma T)^{-1}T||_\infty & \leq \frac{1}{(1 - \gamma)}
\end{align*}
\end{proof}

\begin{theorem}[Option's Termination]
Consider an eigenoption $o = \langle\mathcal{I}_o, \pi_o, \mathcal{T}_o\rangle$ and $\gamma~<~1$. Then, in an MDP with finite state space, $\mathcal{T}_o$ is nonempty.
\end{theorem}

\begin{proof}
This proof is more detailed than the one presented in the main paper. We can write the Bellman equation in the matrix form: ${\bf v} = {\bf r} + \gamma T \bf{v}$, where $\bf{v}$ is a \emph{finite} column vector with one entry per state encoding its value function. From equation (1) in the main paper we have ${\bf r} = T{\bf w - w}$ with ${\bf w} = \bm{\phi}(s)^\top {\bf e}$, where ${\bf e}$ denotes the eigenpurpose of interest. Therefore:

\begin{align*}
{\bf v}                                      &=      T{\bf w - w} + \gamma T {\bf v}\\
{\bf v + w}                                  &=      T{\bf w} + \gamma T {\bf v}\\
                                             &=      T{\bf w} + \gamma T {\bf v} + \gamma T {\bf w} - \gamma T {\bf w}\\
                                             &=      (1- \gamma) T {\bf w} + \gamma T ({\bf v + w})\\
{\bf v + w} - \gamma T ({\bf v + w})         &=      (1 - \gamma) T {\bf w}\\
(I - \gamma T) ({\bf v + w})                 &=      (1 - \gamma) T {\bf w}\\
 \bf{v} + \bf{w}                             &=      (1 - \gamma) (I - \gamma T)^{-1} T {\bf w} && (I - \gamma T)^{-1} \ \text{is guaranteed to be nonsigular because}\\
                                             &      && ||T|| \leq 1 \text{, where } \ ||T|| = \sup_{\mathbf{v}:||\mathbf{v}||_\infty = 1} ||T {\bf v}||_\infty \text{. By }\\
                                             &      && \text{Neumann series we have } (I - \gamma T)^{-1} = \sum_{n=0}^\infty \gamma^nT^n\\
 ||{\bf v + w}||_\infty                      &=      (1 - \gamma)||(I - \gamma T)^{-1} T {\bf w}||_\infty && \text{using the induced norm}\\
 ||{\bf v + w}||_\infty                      &\le    (1 - \gamma)||(I - \gamma T)^{-1} T||_\infty ||{\bf w}||_\infty && \text{because $||A{\bf x}||\leq||A||\cdot||{\bf x}||$}\\
 ||{\bf v + w}||_\infty                      &\le    (1 - \gamma) \frac{1}{(1-\gamma)} ||{\bf w}||_\infty && \text{Lemma~3.2}\\
 ||{\bf v + w}||_\infty                      &\le    ||{\bf w}||_\infty\\
\end{align*}
We can shift ${\bf w}$ by any finite constant without changing the reward, \emph{i.e.} $T{\bf w - w} = T({\bf w} + \bm{\delta}) - ({\bf w} + \bm{\delta})$ because $T{\bf 1}\bm{\delta} = {\bf 1}\bm{\delta}$ since $\sum_j T_{i,j} = 1$. Therefore, we can assume ${\bf w} \ge {\bf 0}$. Let $s^* = \argmax_s {\bf w}_{s^*}$, so that ${\bf w}_{s^*} = ||{\bf w}||_\infty$. Clearly ${\bf v}_{s^*} \le {\bf 0}$, otherwise $||{\bf v + w}||_\infty \ge |{\bf v}_{s^*} + {\bf w}_{s^*}| = {\bf v}_{s^*} + {\bf w}_{s^*} > {\bf w}_{s^*} = ||{\bf w}||_\infty$, arriving at a contradiction.\\
\end{proof}

\addtocounter{section}{1}
\addtocounter{lemma}{-2}

\begin{lemma}
In the tabular case, if all transitions in the MDP have been sampled once, $T^\top T = 2 L$.
\end{lemma}
\begin{proof}
Let $t_{ij}$ and $tt_{ij}$ denote the entries in the $i$-th row and $j$-th column of matrices $T$ and $T^\top T$. We can write $tt_{ij}$ as:
\begin{equation}
tt_{ij} = \sum_k t_{ik} \times t_{jk}.
\end{equation}
In the tabular case, $t_{ij}$ has three possible values:
\begin{itemize}
    \item $t_{ij} = +1$, meaning that the agent arrived in state $j$ at time step $i$, 
    \item $t_{ij} = -1$, meaning that the agent left state $j$ at time step $i$,
    \item $t_{ij} = 0$, meaning that the agent did not arrive nor leave state $j$ at time step~$i$.
\end{itemize}
We decompose $T^\top T$ in two matrices, $K$ and $Z$, such that $T^\top T = K + Z$. Here $Z$ is a diagonal matrix such that $z_{ii} = tt_{ii}$, for all $i$; and $K$ contains all elements from $T^\top T$ that lie outside the main diagonal.

When computing the elements of $Z$ we have $i = j$. Thus $z_{ii} = \sum_k t_{ik}^2$. Because we square all elements, we are in fact summing over all transitions leaving ($-1^2$) \underline{and} arriving ($1^2$) in state $i$, counting the node's degree twice. Thus, $Z = 2D$.

When not computing the elements in the main diagonal, for the element $tt_{ij}$, we add all transitions that leave state $i$ arriving in state $j$ ($-1 \times 1$), \underline{and} those that leave state $j$ arriving in state $i$ ($1 \times -1$). We assume each transition has been sampled once, thus:
$$tt_{ij} = \left\{\begin{array}{rl}
-2, &\mbox{if the transition between states $i$ and $j$ exists},\\
0, &\mbox{otherwise}.
\end{array}\right.
$$
Therefore, we have $K = -2 W$ and $T^\top T = K + Z = 2 (D - W)$.
\end{proof}

\section*{B. Diffusion Time Computation}

In the main paper we introduced \emph{diffusion time} as a new metric to evaluate exploration, but we did not discuss how it can be computed. Diffusion time encodes the expected number of time steps required to navigate between any two states in the MDP when following a random walk. In tabular domains, we can easily compute the diffusion time with dynamic programming. To do so we define a new MDP such that the value function of a state $s$, under a uniform random policy, encodes the expected number of steps required to navigate between state $s$ and a chosen goal state. We can then compute the expected number of steps between any two states by averaging, for each possible goal, the value of all other states.

The MDP in which the value function of state $s$ encodes the expected number of time steps from $s$ to a goal state has $\gamma = 1$ and a reward function where the agent observes $+1$ at every time step in which it is not in the goal state. Policy evaluation in this case encodes the expected number of time steps the agent will take before arriving to the goal state. To compute the diffusion time we iterate over all possible states, defining them as terminal states, and averaging the value function of the other states in that MDP.

\section*{C. Options Leading to Doorways in the 4-room Domain}

Figure~\ref{fig:bottleneck_options} depicts the four options we refer to in Section~4 as the options leading to bootleneck states, \emph{i.e.}, doorways. Each option is defined in a room and it moves the agent toward the closest doorway. These options were inspired by Solway et al. (2014)'s discussion about the optimal options discovered by their algorithm.
\begin{figure*}[h]
    \centering
    \begin{subfigure}[b]{0.24\textwidth}
    \center
        \includegraphics[width=0.8\columnwidth]{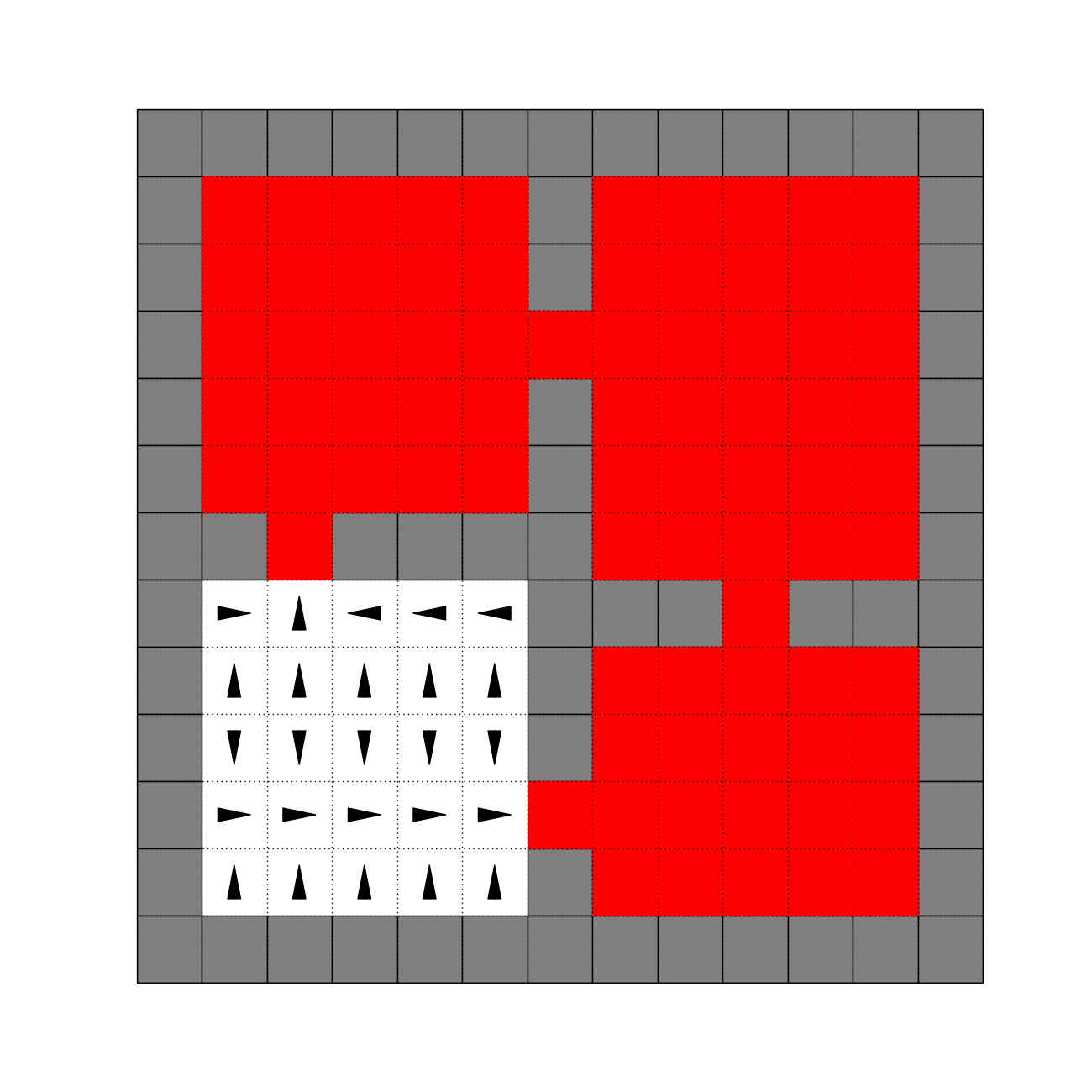}
    \end{subfigure}
    \begin{subfigure}[b]{0.24\textwidth}
    \center
        \includegraphics[width=0.8\columnwidth]{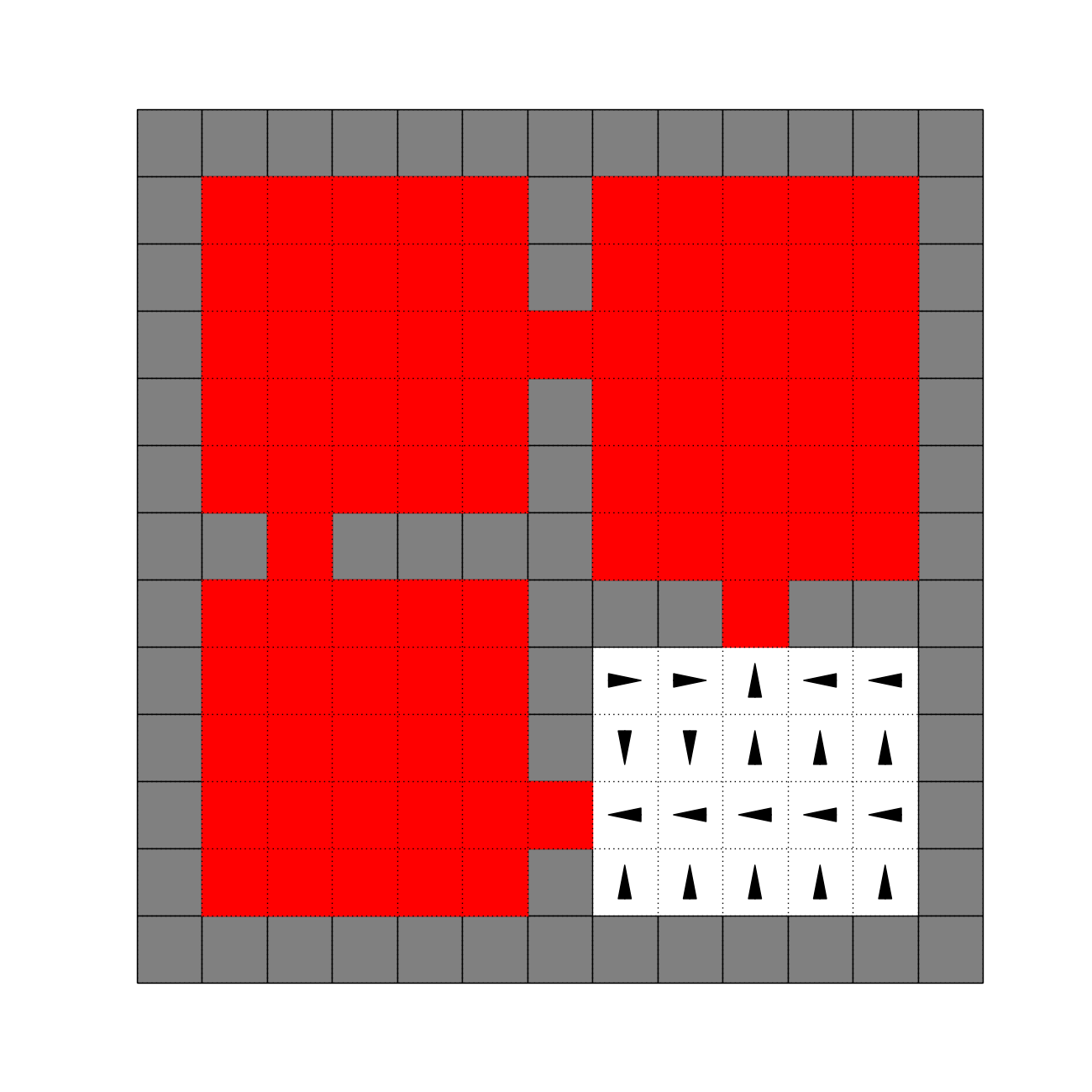}
    \end{subfigure}
    \begin{subfigure}[b]{0.24\textwidth}
    \center
        \includegraphics[width=0.8\columnwidth]{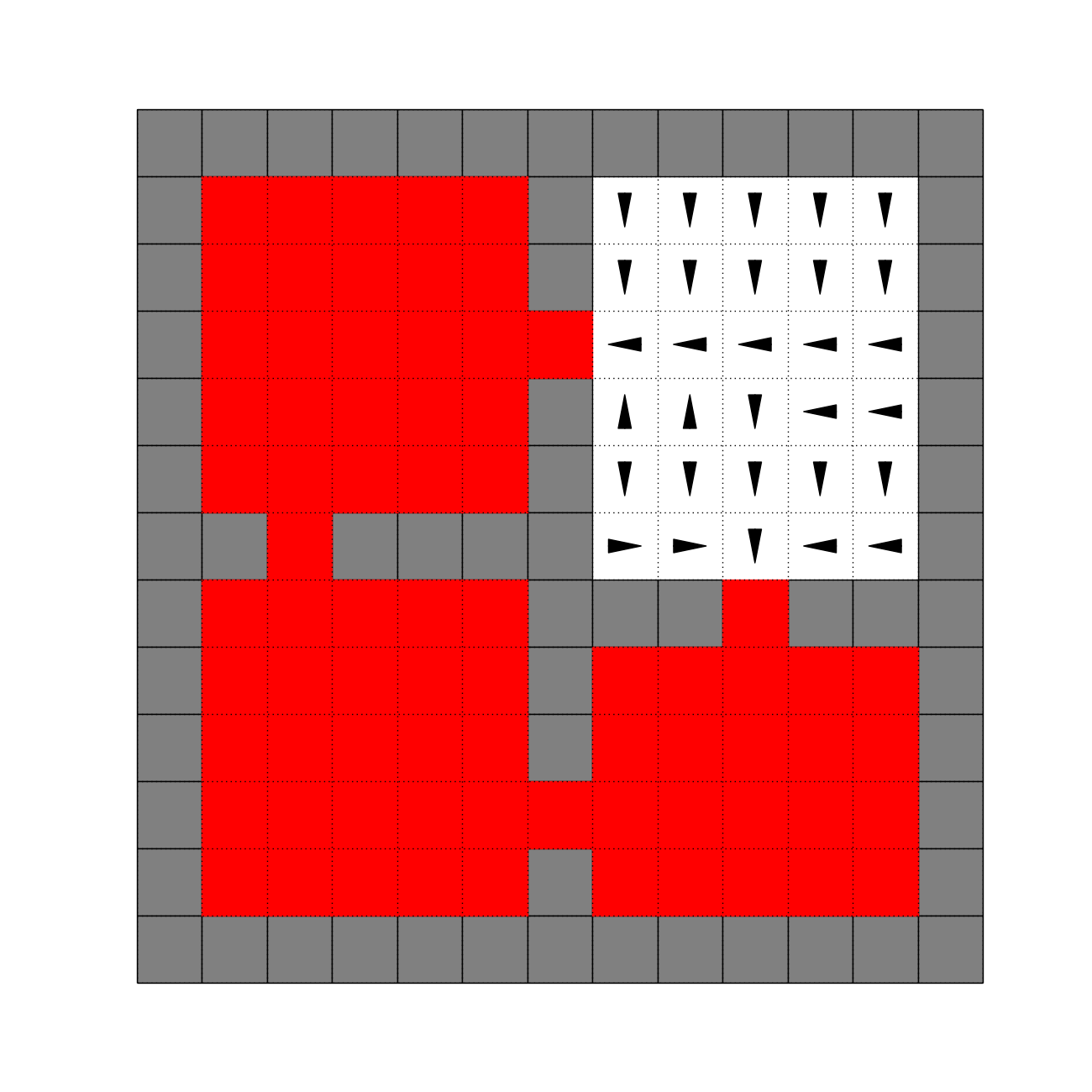}
    \end{subfigure}
    \begin{subfigure}[b]{0.24\textwidth}
    \center
        \includegraphics[width=0.8\columnwidth]{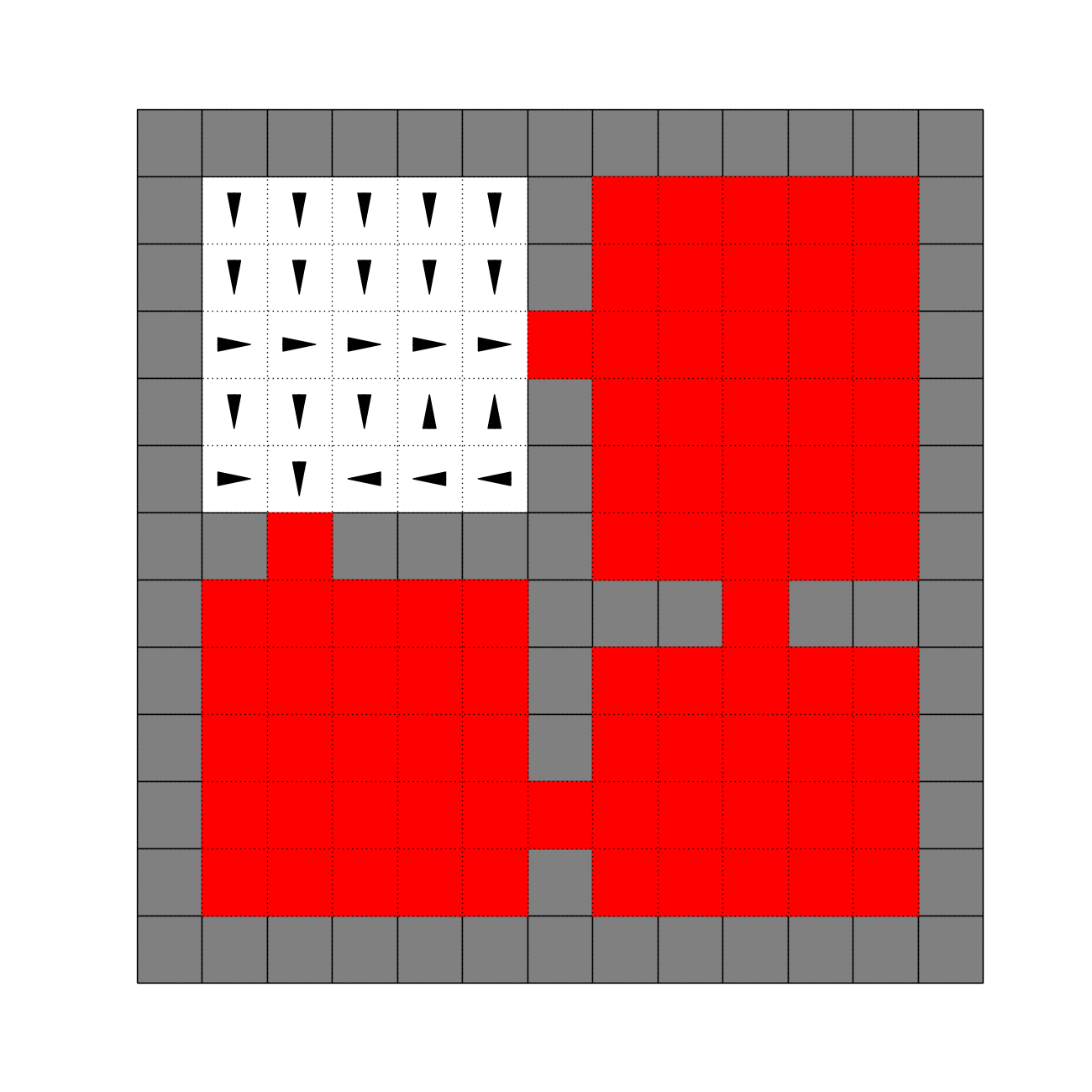}
    \end{subfigure}
    \caption{Options leading to bottleneck states. Each option is defined in a single room, moving the agent to the closest doorway.}
    \label{fig:bottleneck_options}
\end{figure*}

\section*{D. Comparison to Random Options}

In this section we show the importance of using information about diffusion in the environment to define the option's purposes. This information impacts the sequence of subgoal locations the options' seek after, as well as the time scales they operate at. The ordering in which the eigenoptions are discovered and the different time scales they operate at can have a major impact on the agents' performance. 

We demonstrate the importance of using the environment's diffusion information by comparing our approach to \emph{random options}, a simple baseline that does not use such information. This baseline defines an option to be the policy, defined in the whole state space, that terminates in a randomly selected state of the environment. We performed our experiments in the tabular case because it is not clear how we can extend this baseline to settings in which states cannot be enumerated.

Figure~\ref{fig:random_diffusion_time} depicts the diffusion time (\emph{c.f.} Section B) of random options and eigenoptions in the 4-room domain. We used the same method described in Section~4.2 to obtain the eigenoptions' performance. For the random options results, we added them incrementally to the agent's action set until having added all possible options. We repeated this process $24$ times to verify the impact of adding random options in a different order. Each blue line represents the performance of one of the evaluated sequences. The results clearly show that eigenoptions do more than going to a randomly selected state. Most of the obtained sequences of random options fail to reduce the agent's diffusion time. They increase it by several orders of magnitude (notice the $y$-axis is in logarithmic scale) until having enough options available to the point that the graph is almost fully connected, that is, when the agent basically has an option leading it to each possible state in the MDP.

Figure~\ref{fig:random_learning_curve} was generated following the protocol described in Section 4.3. It depicts the learning curve of agents equipped with eigenoptions and of agents equipped with random options. As before, the blue lines indicate the agent's performance in individual runs. We can see that no individual run is competitive to eigenoptions. When fewer options are used (not shown), the variance across individual runs is even larger, depending on whether one of the random options terminates near the goal state. In some runs the agent never even learns to reach the goal. Therefore, as in the diffusion time, on average, random options are not competitive to eigenoptions, demonstrating the importance of the diffusion model we use.

\begin{figure}[t]
    \centering
    \begin{subfigure}[b]{0.4\columnwidth}
        \includegraphics[width=\columnwidth]{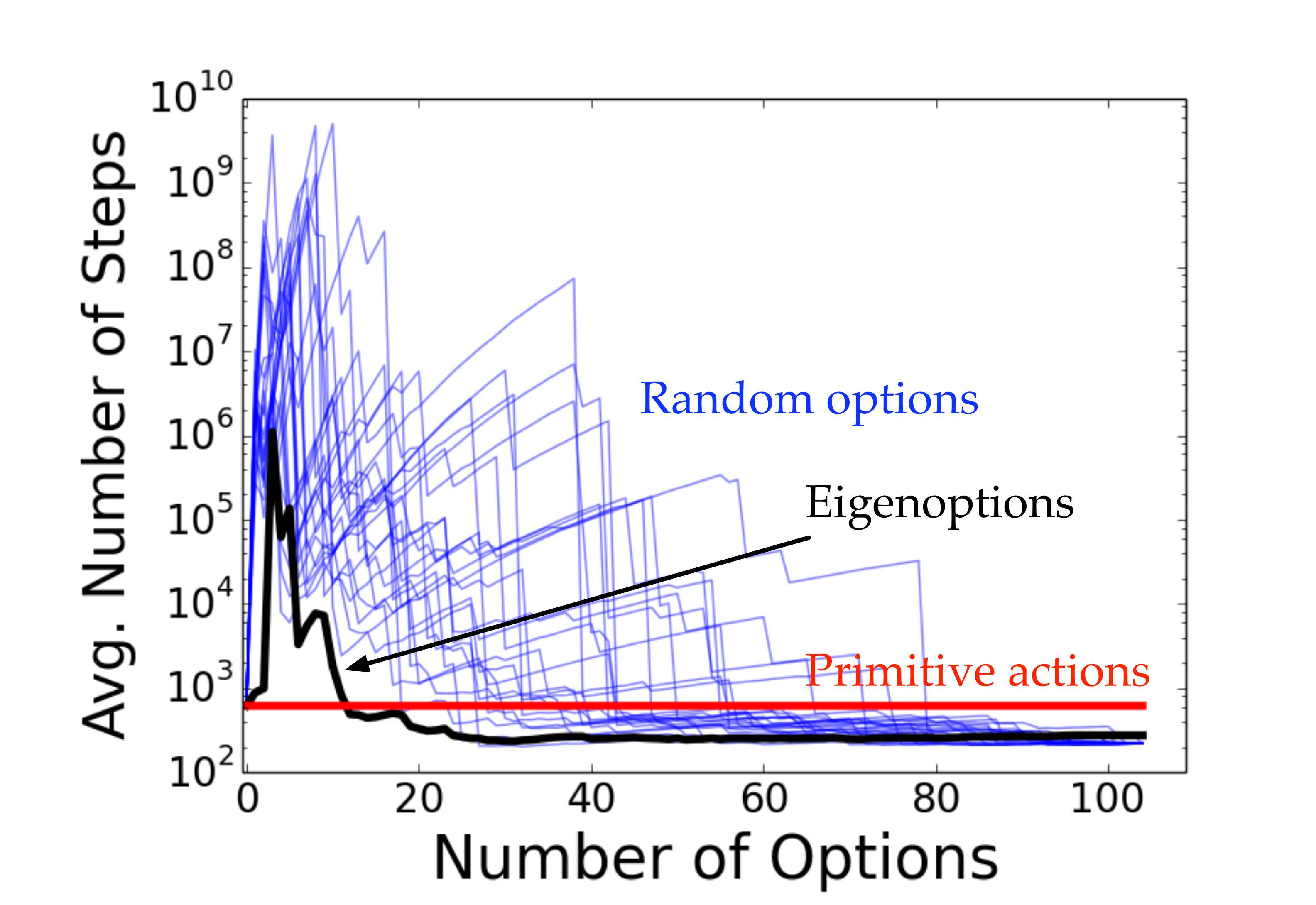}
        \caption{Diffusion time. The $y$-axis is in logarithmic scale.}
        \label{fig:random_diffusion_time}
    \end{subfigure}
    ~~~~~~~~
    \begin{subfigure}[b]{0.4\columnwidth}
        \raisebox{1.6mm}{\includegraphics[width=0.84\columnwidth]{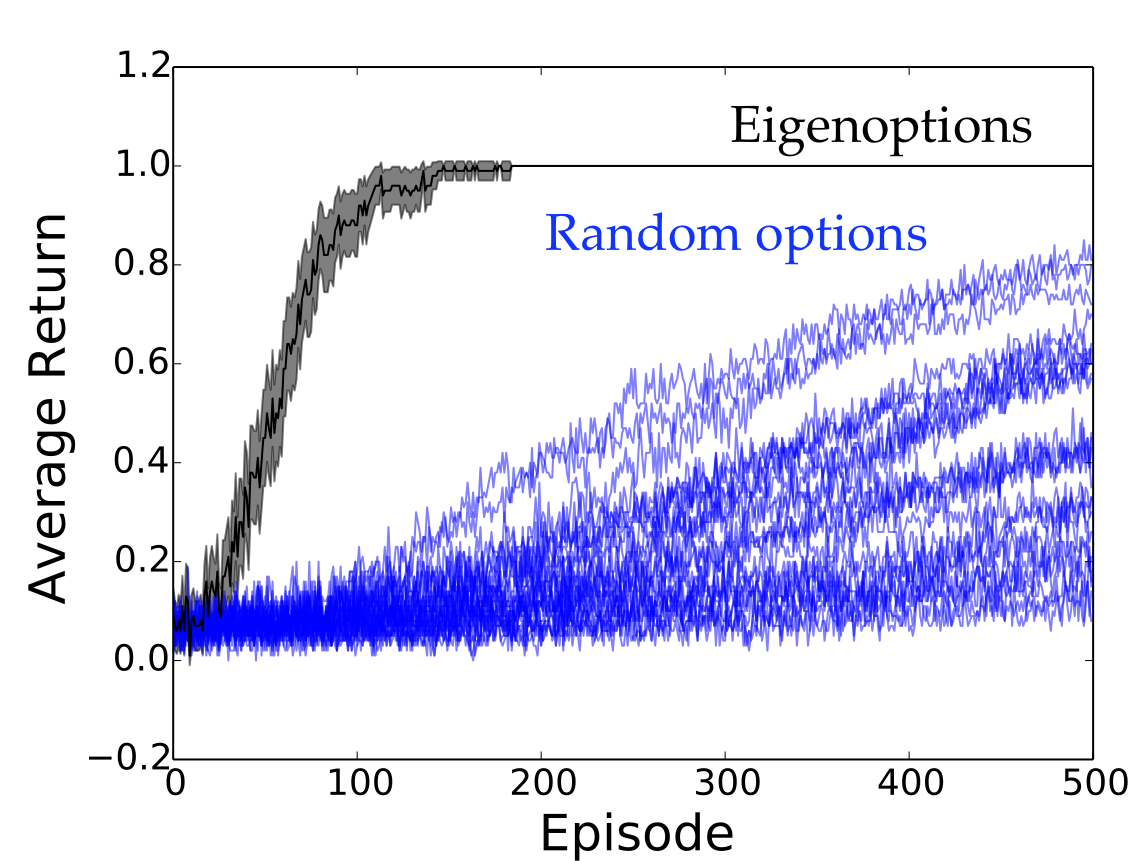}}
        \caption{Learning curve using 64 options.}
        \label{fig:random_learning_curve}
    \end{subfigure}
    \caption{Diffusion time and learning performance of eigenoptions and of random options in the 4-room domain.}\label{fig:random_baseline}
\end{figure}

\section*{D. Empirical Evaluation of the Agent's Performance in Multiple Tasks}

In Section~4 we argued that eigenoptions are useful for multiple tasks, based on results showing that eigenoptions allow us to find and to accumulated rewards faster. Here we explicit demonstrate the uselfuness of eigenoptions to multiple tasks. We evaluate the agents' performance for different starting and goal states in the 4-room domain. As in Section~4.3, we use Q-Learning ($\alpha = 0.1, \gamma = 0.9$) to learn a policy over primitive actions. The behavior policy chooses uniformly over primitive actions and options, following them until termination. Episodes were $100$ time steps long, and we learned for $250$ episodes. For clarity, we zoom in the plots on the interval in which agents are still learning.

Figure~\ref{fig:multi} depicts, after learning for a pre-determined number of episodes, the average over 100 trials of the agents' final performance, as well as the starting (S) and goal (G) states. Based on our previous results, we fixed the number of used eigenoptions to $64$ ($32$ options and their negations). In this set of experiments we also compare our approach to traditional bottleneck options (Figure~\ref{fig:bottleneck_options}).

The obtained results show that switching the positions of the starting and goal states have no effect in the performance of our algorithm. Also, in almost all settings, the agents augmented by eigenoptions outperfom those equipped only with primitive actions. The comparison between eigenoptions and options that look for bottleneck states is more subtle. As expected, agents equipped with eigenoptions outperform agents equipped with options leading to bottleneck states in settings in which the goal state is far from the doorways, as discussed in the main paper. In scenarios where the goal state is closer to bottleneck states, the options leading to doorways are more competitive. Importantly, this analysis is based on the results when using 64 eigenoptions, which may not encode all options required to go to a specific region of the state space.

\clearpage

\section*{E. Experimental Setup in the Arcade Learning Environment}

We defined six different starting states in each Atari 2600 game, letting the agent take random actions from that point until termination. The agent follows a pre-determined sequence of actions leading it to each starting state. We store the observed transitions leading the agent to the start states as well as those obtained from the random actions. In the main paper we provided results for \textsc{Freeway} and \textsc{Montezuma's Revenge}. In this section we also provide results for \textsc{Ms Pac-Man}. The starting states for all three games are depicted in Figure~\ref{fig:start_states}.

\begin{figure}[t]
    \centering
    \begin{subfigure}[b]{0.3\columnwidth}
        \includegraphics[width=\columnwidth]{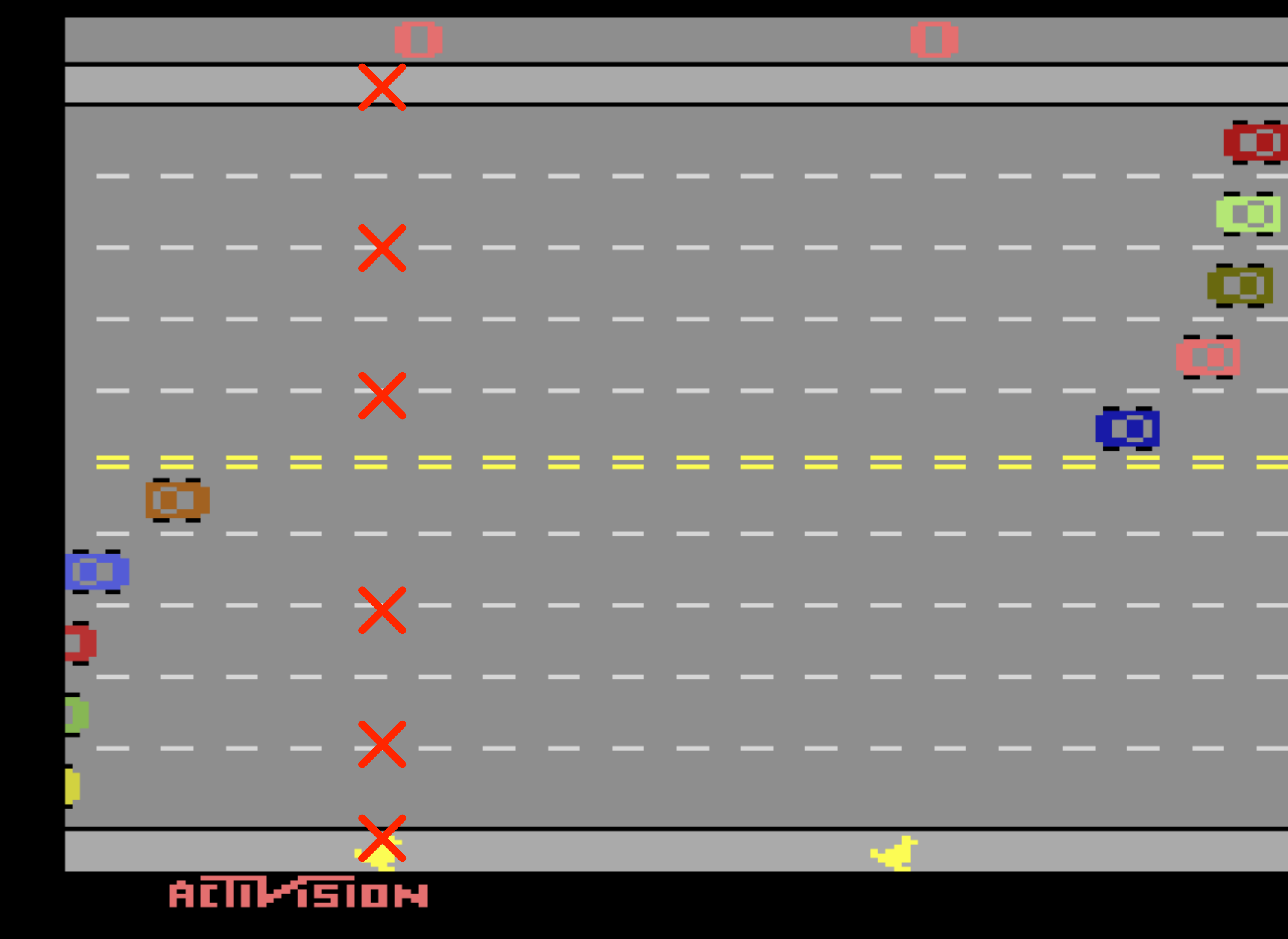}
        \caption{\textsc{Freeway}}
    \end{subfigure}
    ~
    \begin{subfigure}[b]{0.3\columnwidth}
        \includegraphics[width=\columnwidth]{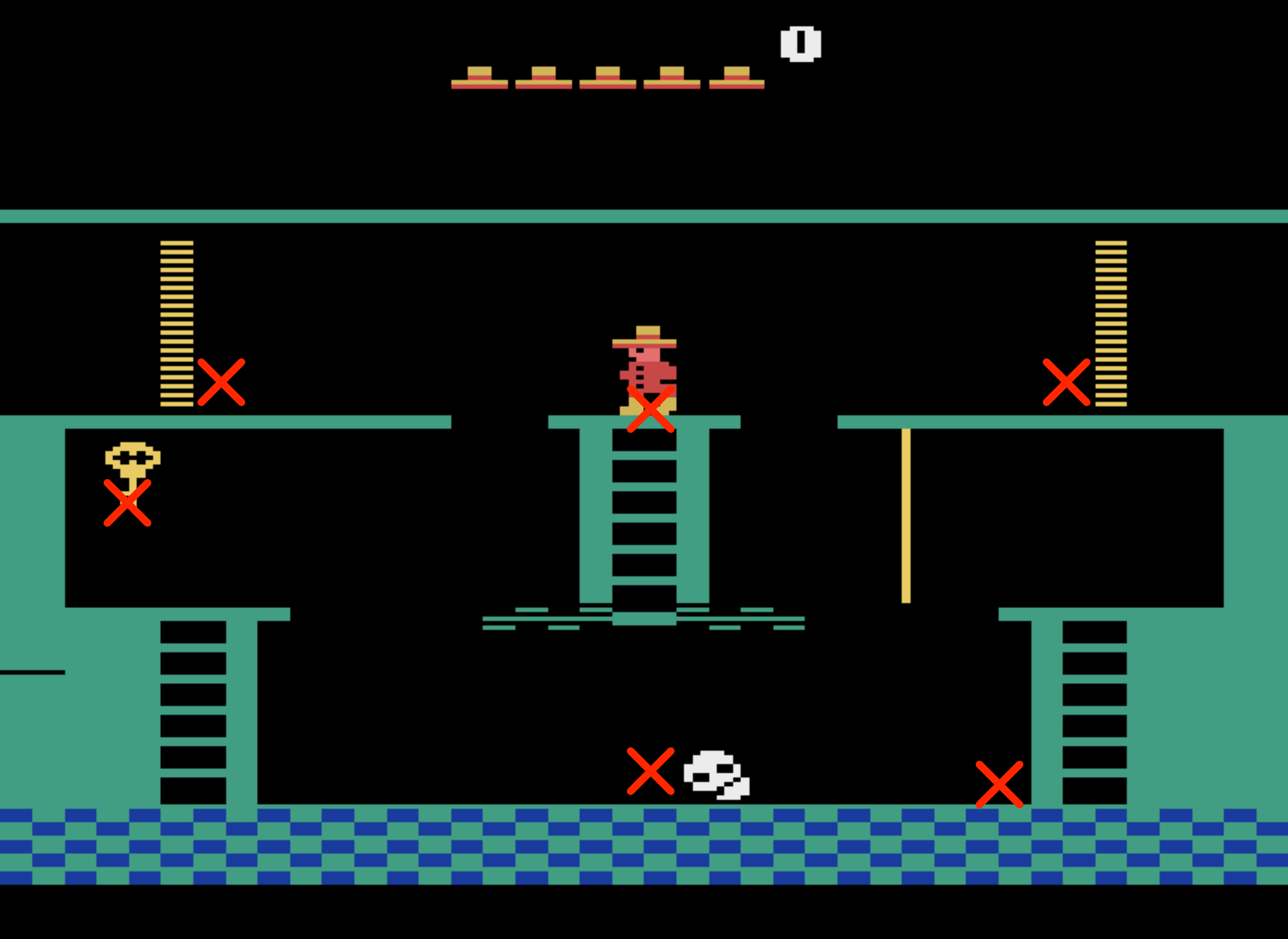}
        \caption{\textsc{Montezuma's Revenge}}
    \end{subfigure}
    ~
    \begin{subfigure}[b]{0.3\columnwidth}
        \includegraphics[width=\columnwidth]{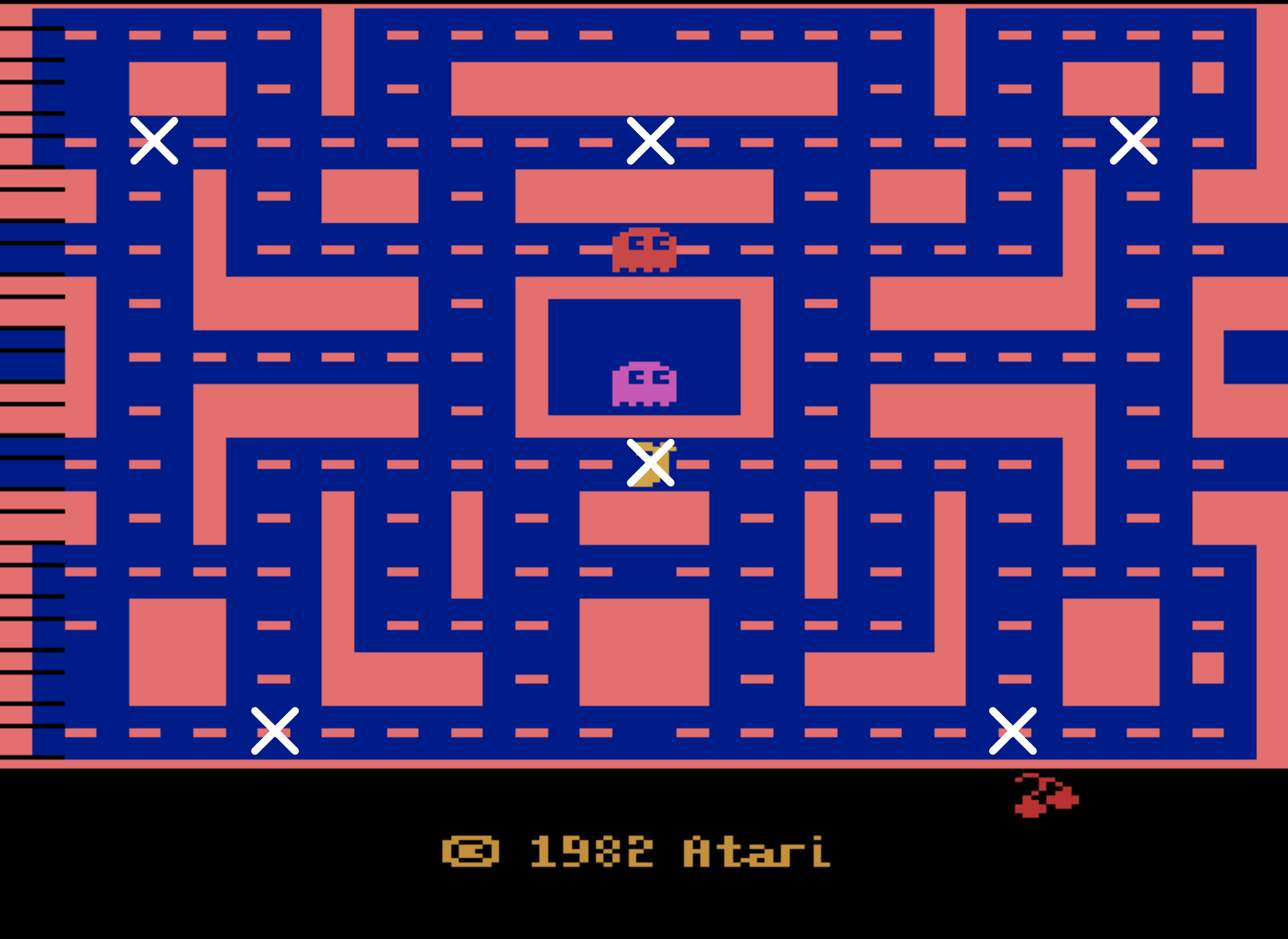}
        \caption{\textsc{Ms Pac-Man}}
    \end{subfigure}
    \caption{Pre-defined start states in Atari 2600 games.}\label{fig:start_states}
\end{figure}

The agent plays rounds of six episodes, with each episode starting from a different start state, until it observes at least $25{,}000$ new transitions. The final incidence matrix in which we ran the SVD had $25{,}000$ rows, which we sampled uniformly from the set of observed transitions. The agent used the deterministic version of the Arcade Learning Environment (ALE), the games' minimal action set and, a frame skip of $1$.

\begin{wrapfigure}{r}{0.42\textwidth}
    \begin{center}
    \includegraphics[width=0.35\columnwidth]{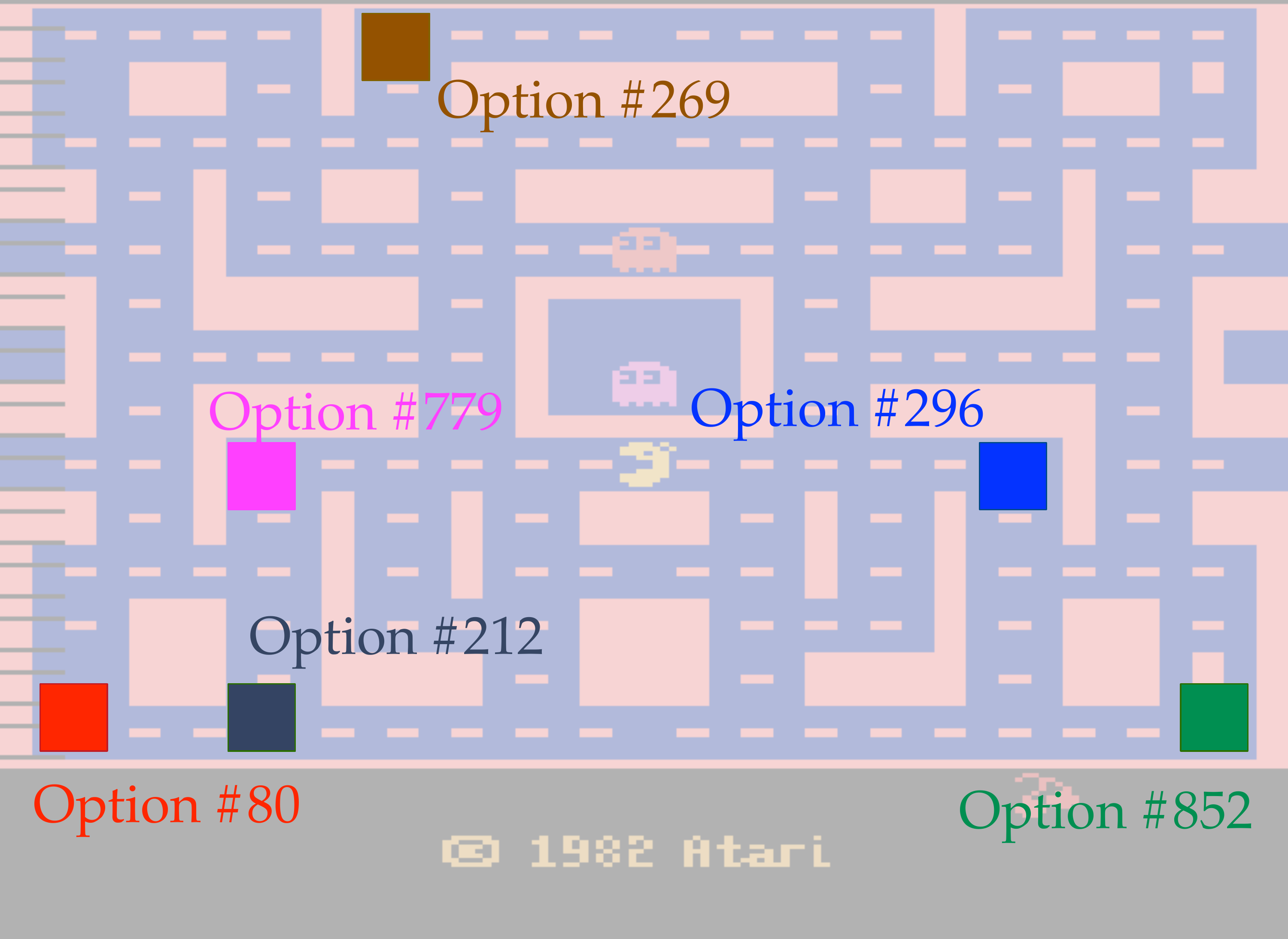}
    \caption{Options in \textsc{Ms. Pac-Man} (\emph{c.f.} text for details).}
    \end{center}
    \label{fig:ms_pacman}
\end{wrapfigure}

We used three games to evaluate the options we discover in the sample-based setting with linear function approximation. We discussed the results for \textsc{Freeway} and \textsc{Montezuma's Revenge} in the main paper. The results we obtained in \textsc{Ms. Pac-Man} are similar to those we already discussed. \textsc{Ms. Pac-Man} is a game in which the agent needs to navigate through a maze eating pellets while avoiding ghosts. As in the other games, the agent has the clear intent of reaching particular positions in the screen, such as corners and intersections. Figure~4 depicts the positions in which agents tend to spend most of their time on.  A video of the highlighted options can be found online.\footnote{\url{https://youtu.be/2BVicx4CDWA}}

\begin{figure}[t]
    \centering
    \begin{subfigure}[b]{0.24\columnwidth}
        \includegraphics[width=\columnwidth]{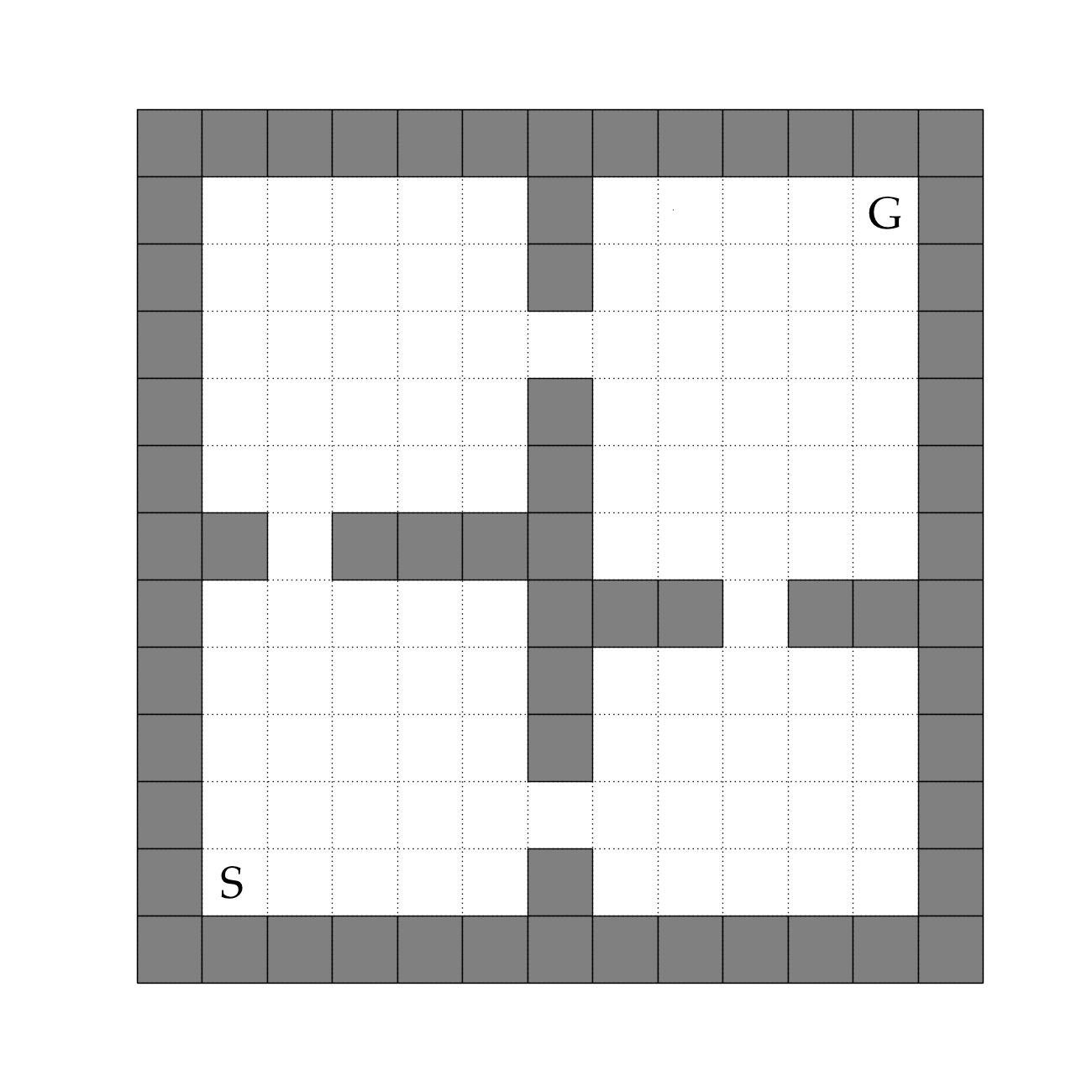}
    \end{subfigure}
    \begin{subfigure}[b]{0.24\columnwidth}
        \raisebox{4.00mm}{\includegraphics[width=\columnwidth]{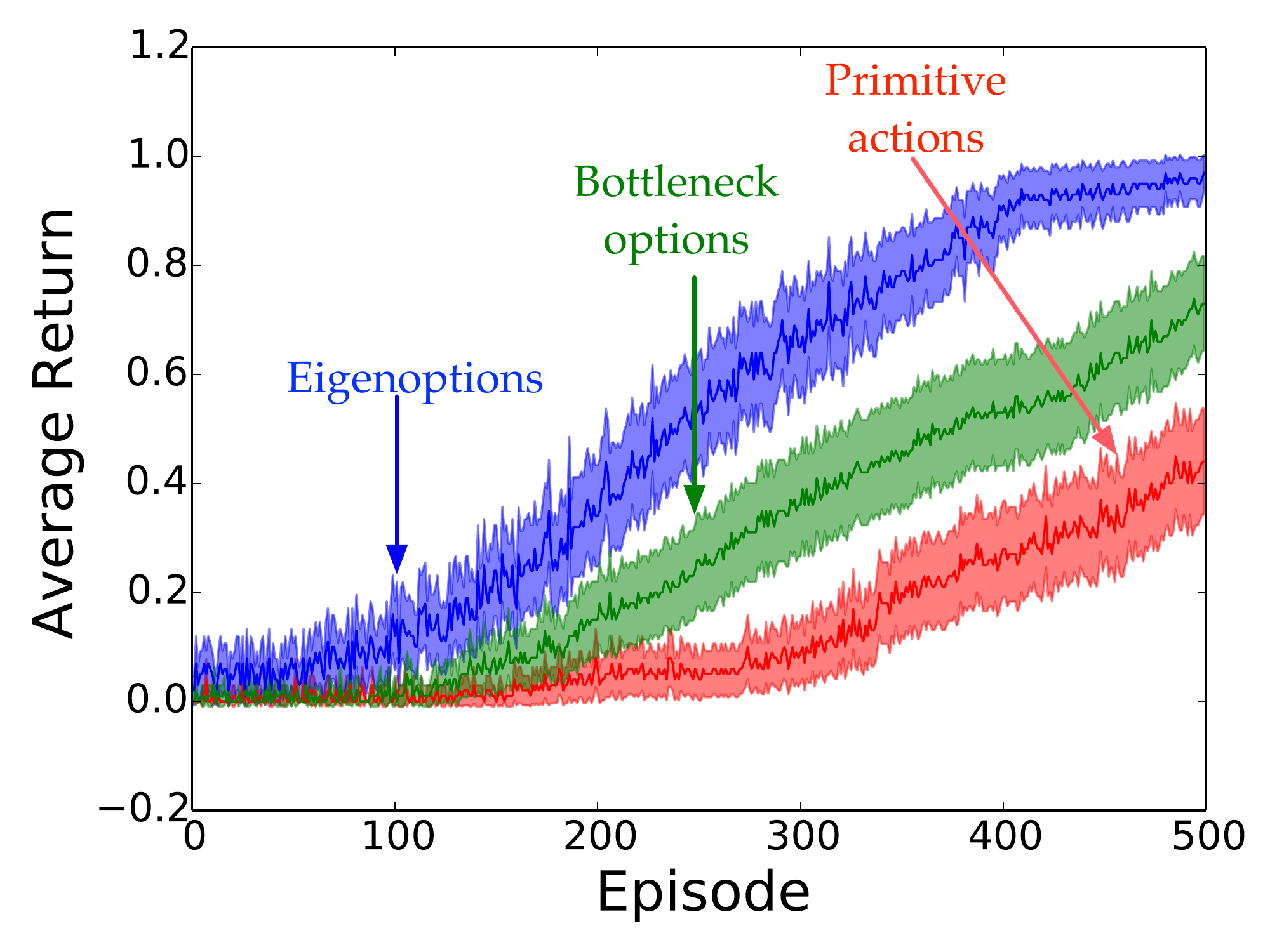}}
    \end{subfigure}
    \begin{subfigure}[b]{0.24\columnwidth}
        \includegraphics[width=\columnwidth]{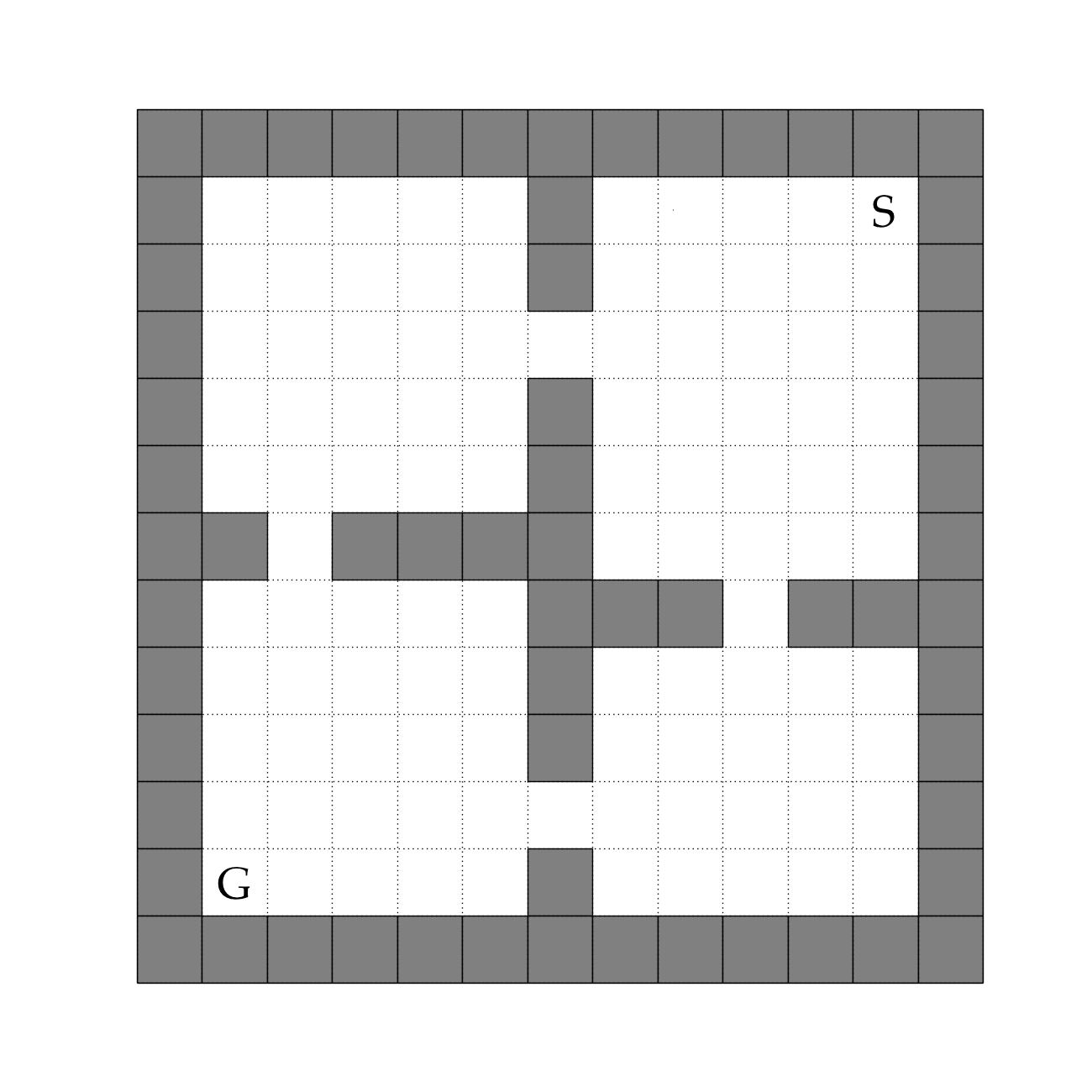}
    \end{subfigure}
    \begin{subfigure}[b]{0.24\columnwidth}
        \raisebox{4.00mm}{\includegraphics[width=\columnwidth]{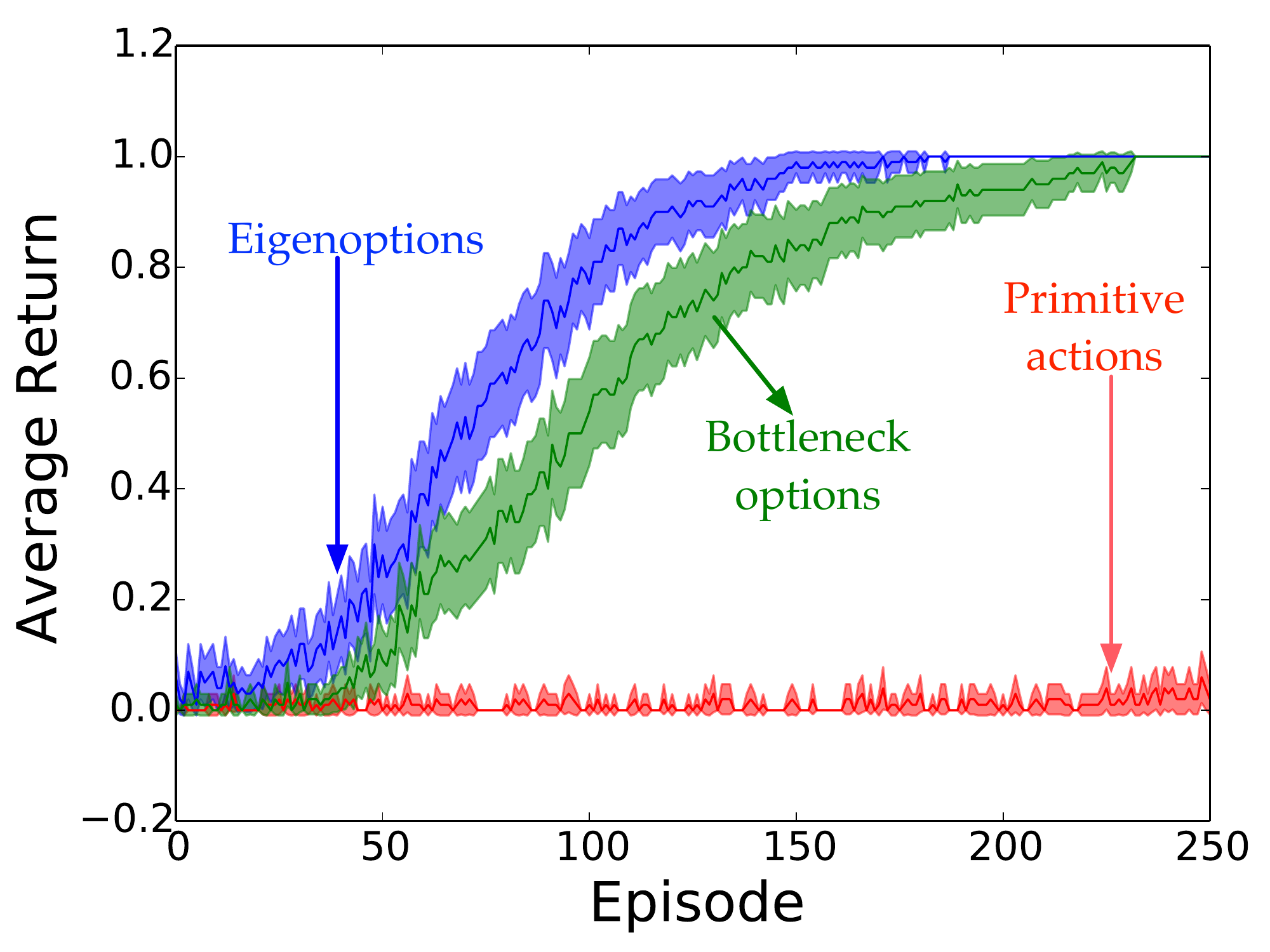}}
    \end{subfigure}
    \begin{subfigure}[b]{0.24\columnwidth}
        \includegraphics[width=\columnwidth]{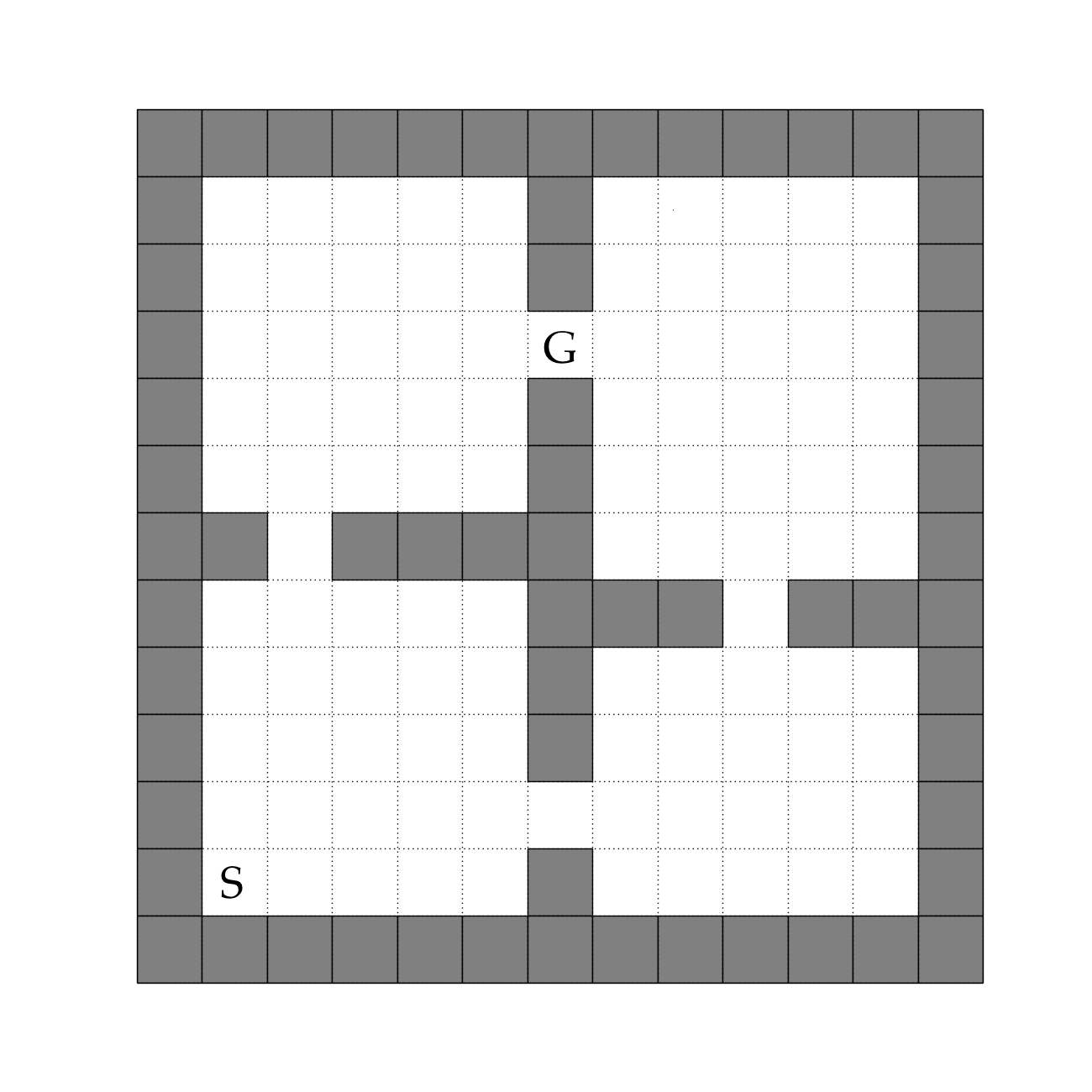}
    \end{subfigure}
    \begin{subfigure}[b]{0.24\columnwidth}
        \raisebox{4.00mm}{\includegraphics[width=\columnwidth]{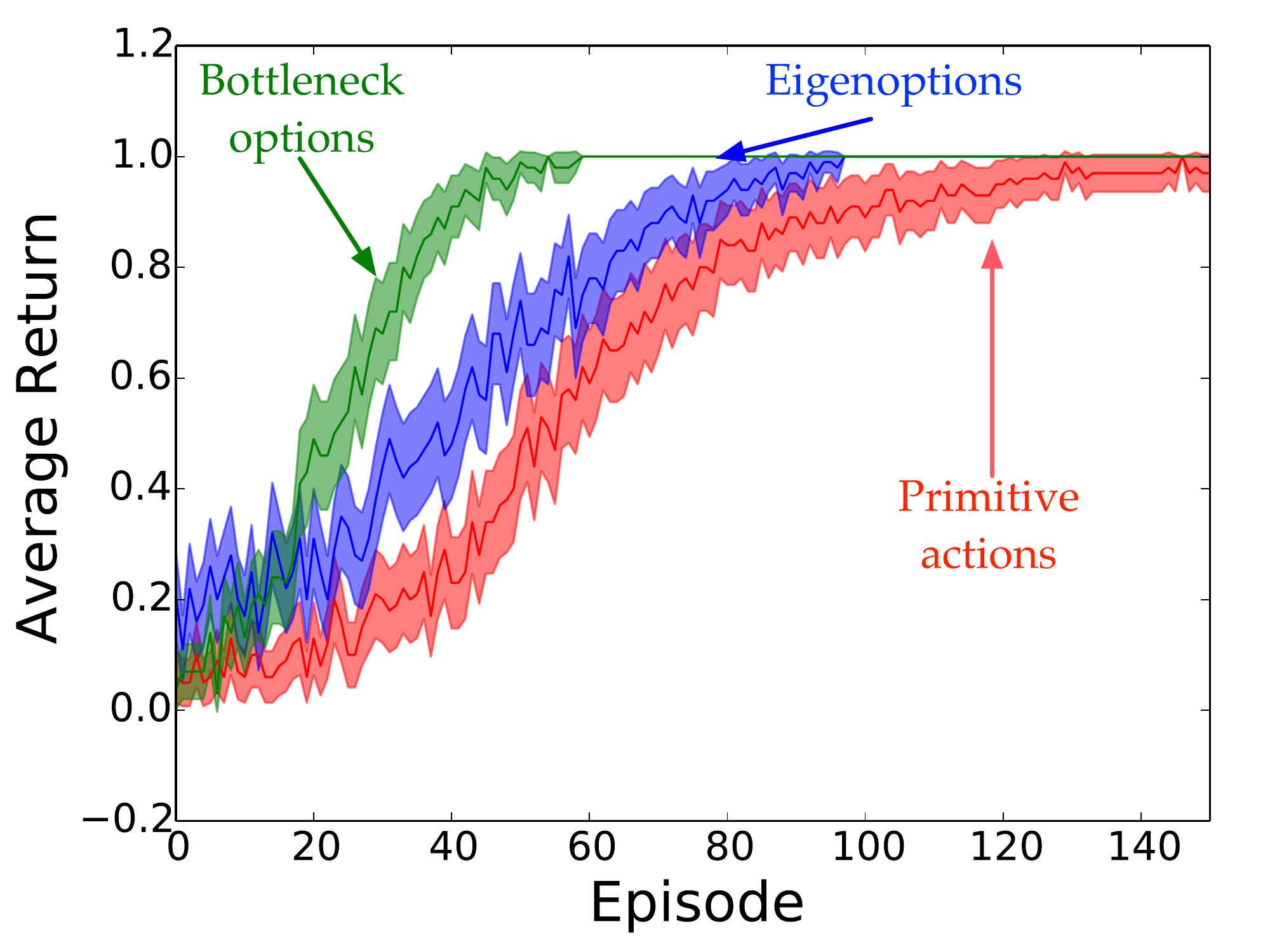}}
    \end{subfigure}
    \begin{subfigure}[b]{0.24\columnwidth}
        \includegraphics[width=\columnwidth]{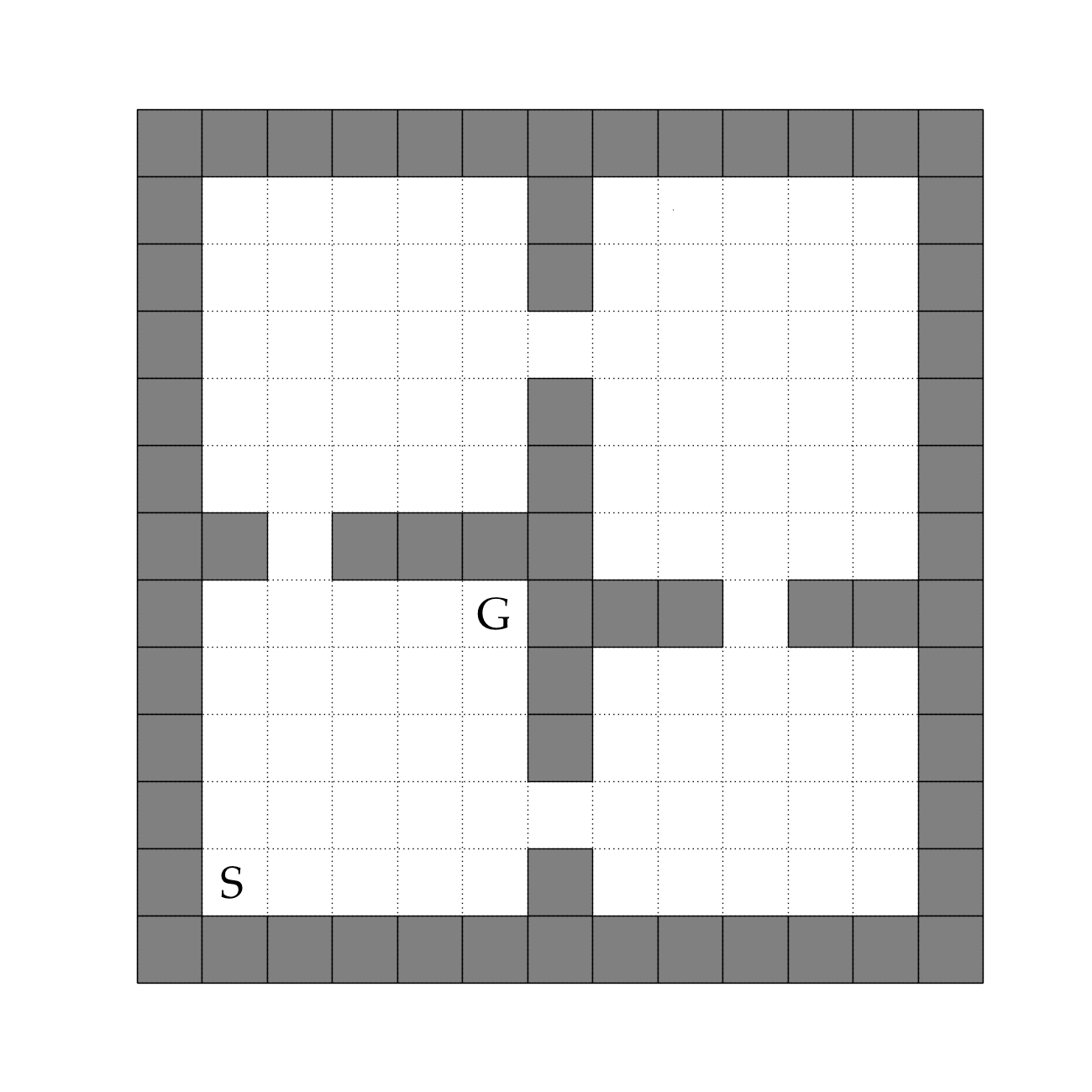}
    \end{subfigure}
    \begin{subfigure}[b]{0.24\columnwidth}
        \raisebox{4.00mm}{\includegraphics[width=\columnwidth]{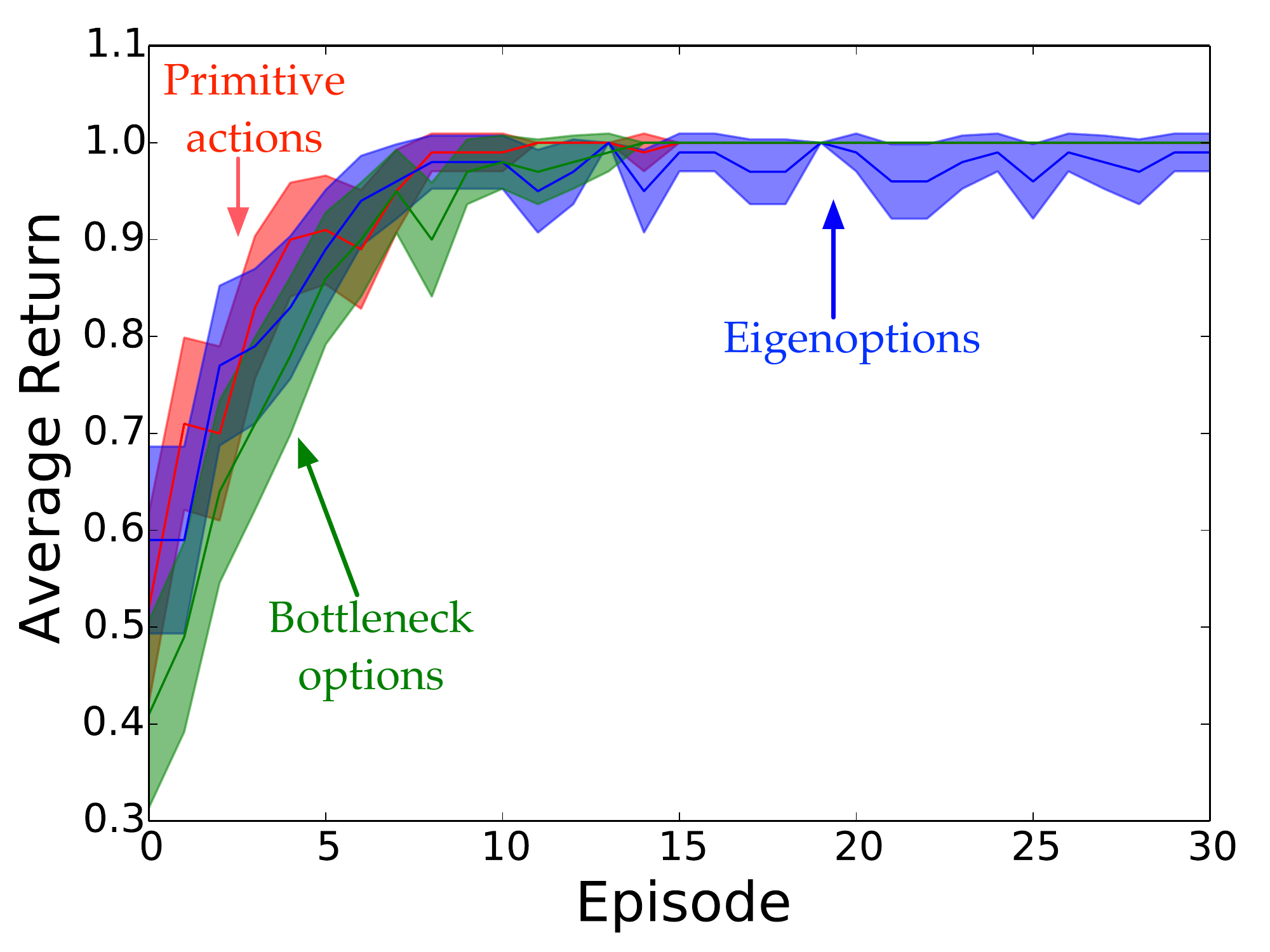}}
    \end{subfigure}
    \begin{subfigure}[b]{0.24\columnwidth}
        \includegraphics[width=\columnwidth]{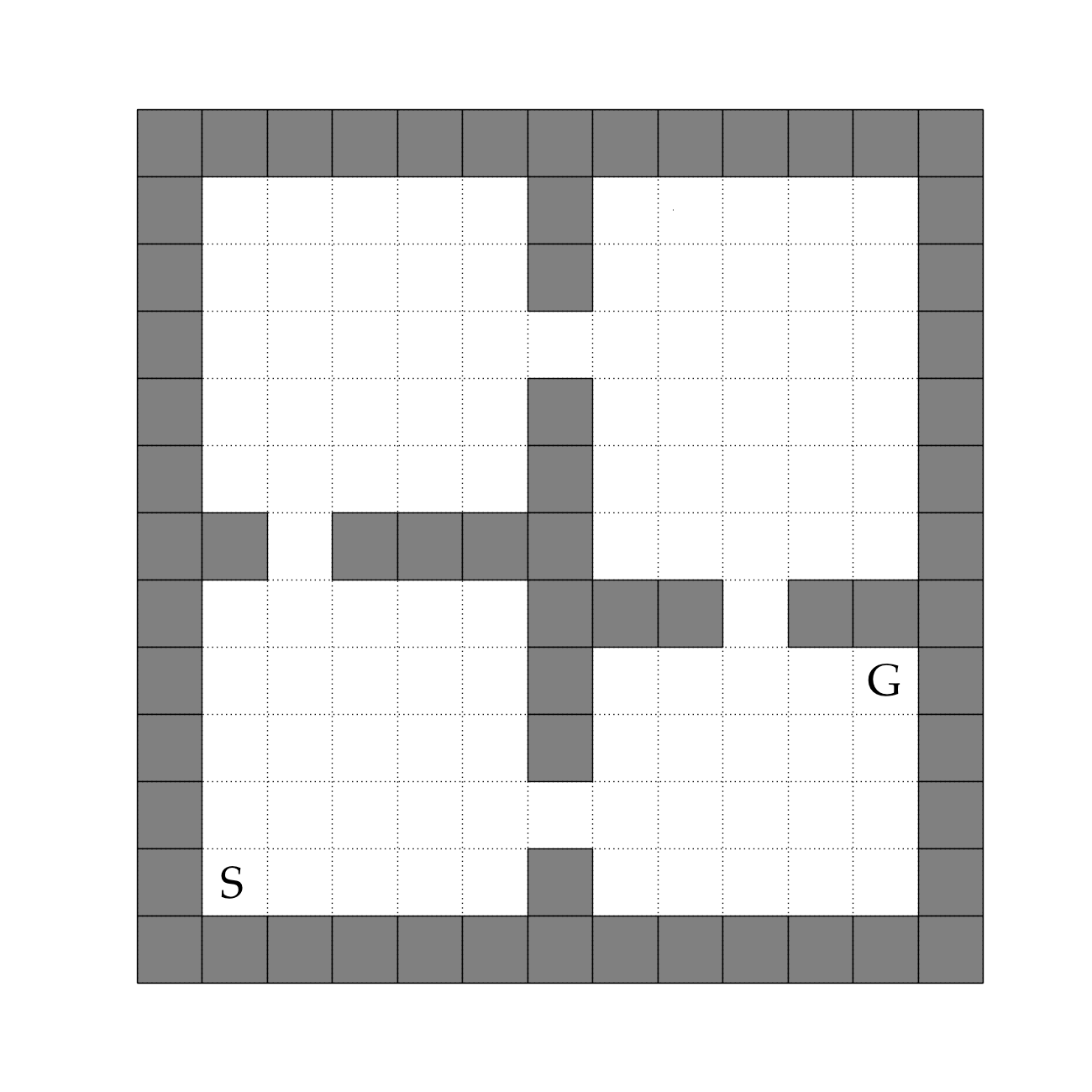}
    \end{subfigure}
    \begin{subfigure}[b]{0.24\columnwidth}
        \raisebox{4.00mm}{\includegraphics[width=\columnwidth]{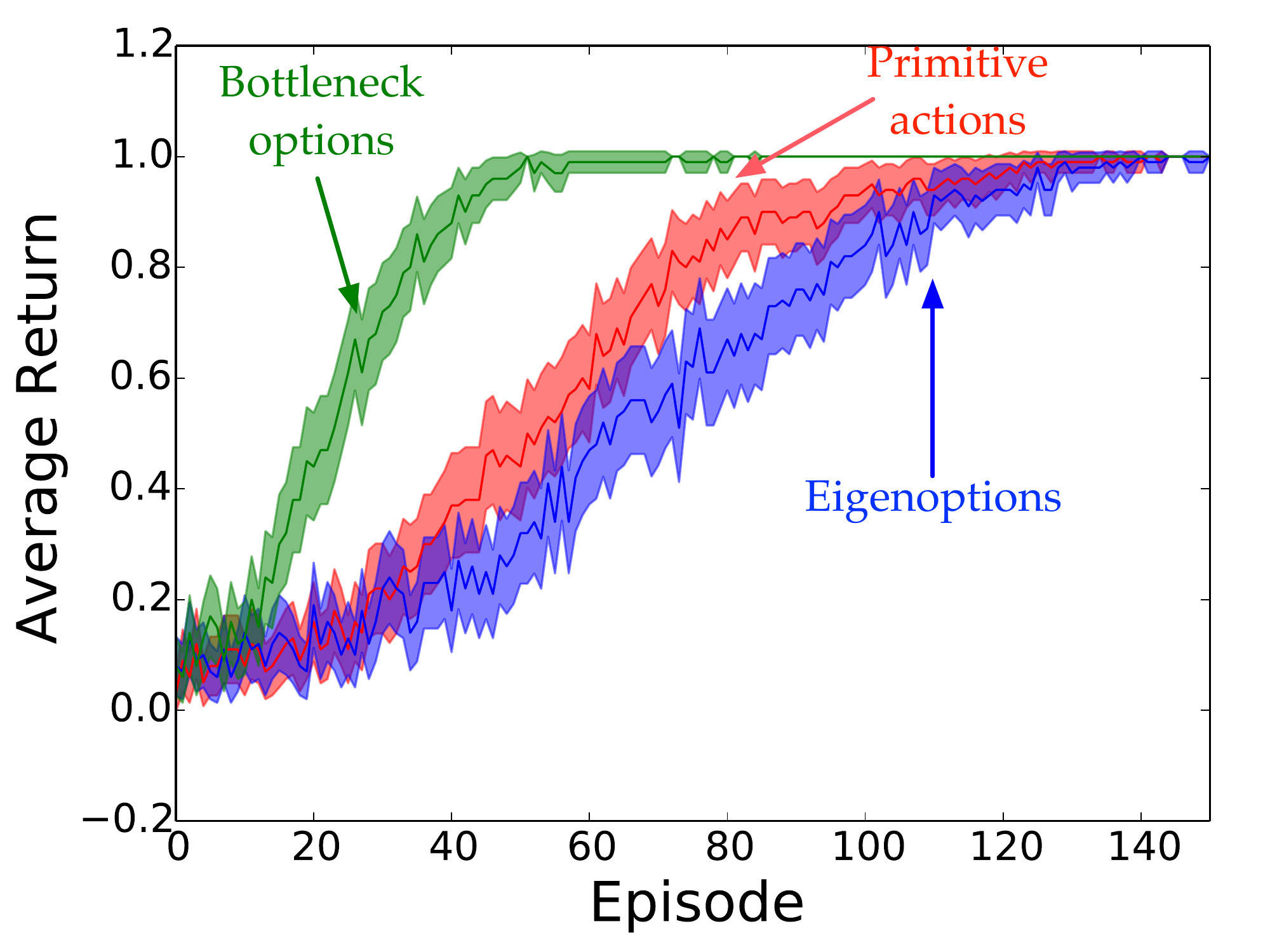}}
    \end{subfigure}
    \begin{subfigure}[b]{0.24\columnwidth}
        \includegraphics[width=\columnwidth]{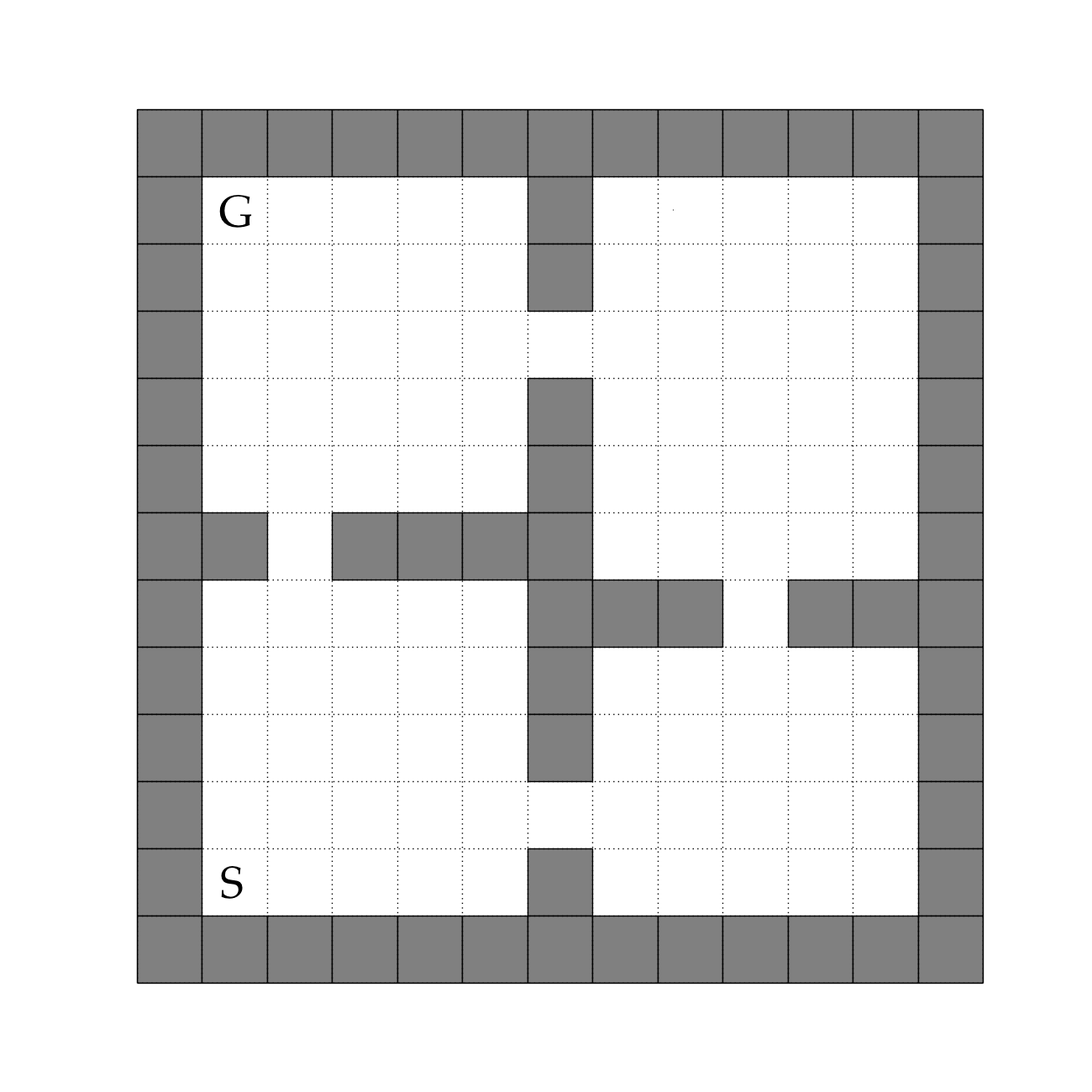}
    \end{subfigure}
    \begin{subfigure}[b]{0.24\columnwidth}
        \raisebox{4.00mm}{\includegraphics[width=\columnwidth]{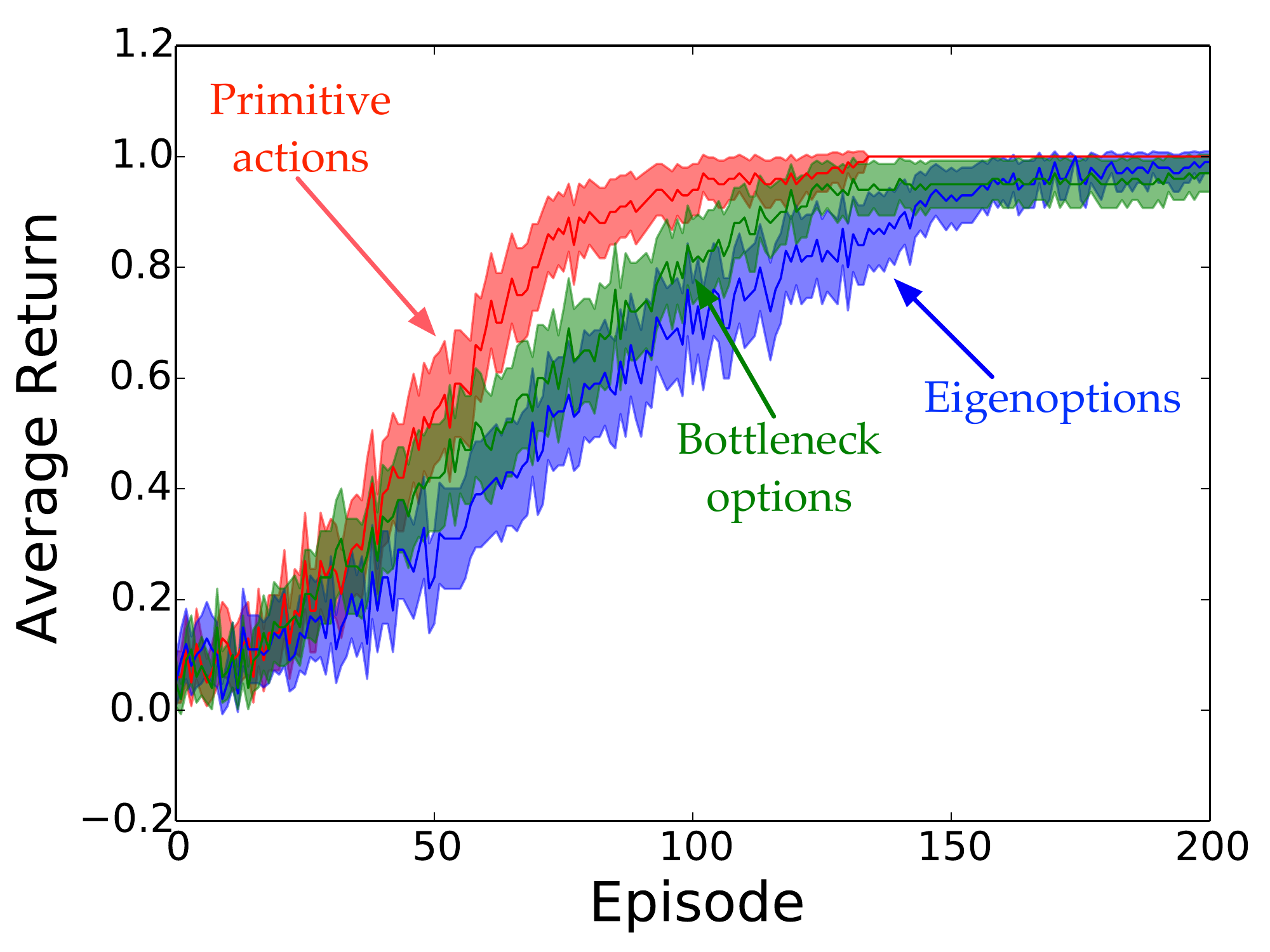}}
    \end{subfigure}
    \begin{subfigure}[b]{0.24\columnwidth}
        \includegraphics[width=\columnwidth]{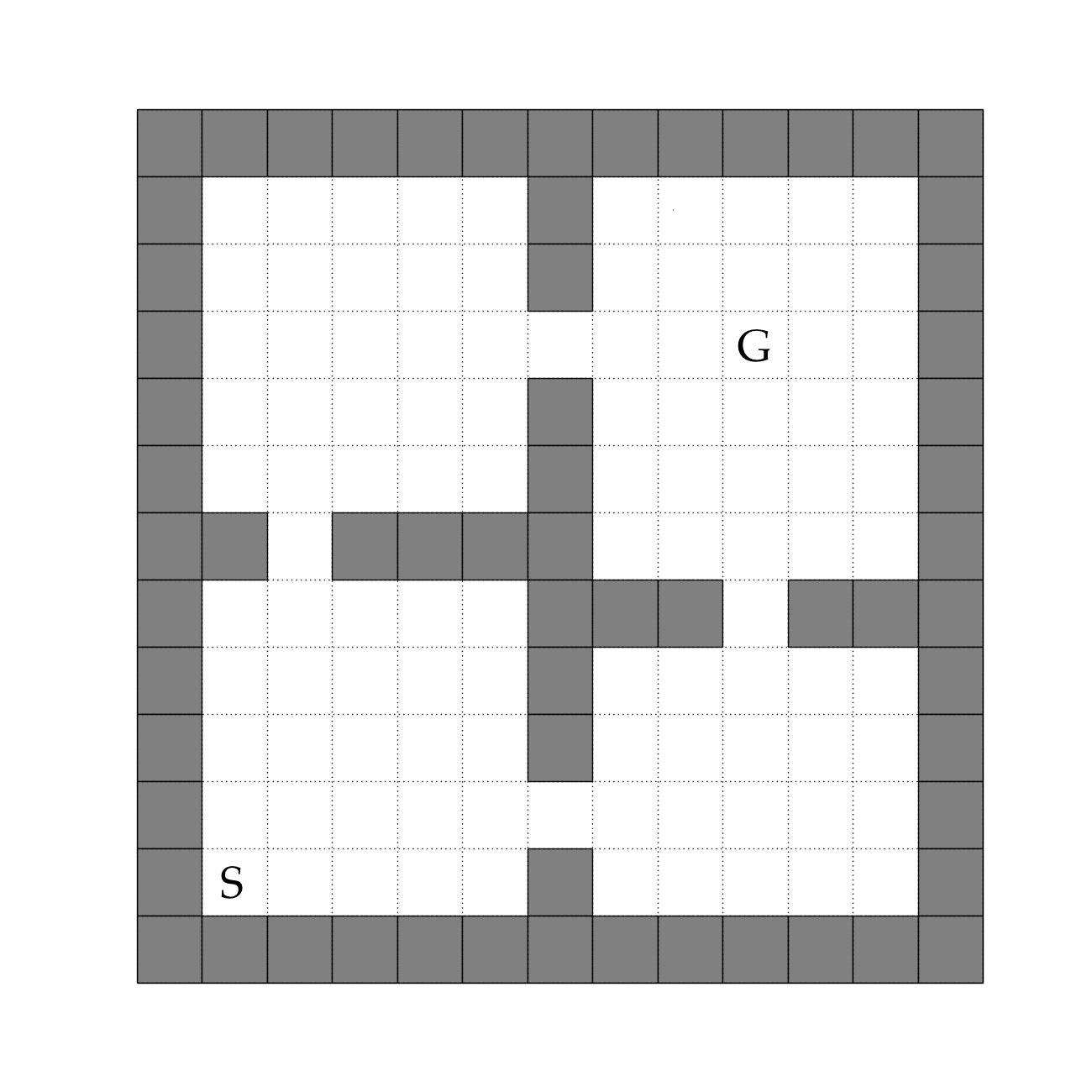}
    \end{subfigure}
    \begin{subfigure}[b]{0.24\columnwidth}
        \raisebox{4.00mm}{\includegraphics[width=\columnwidth]{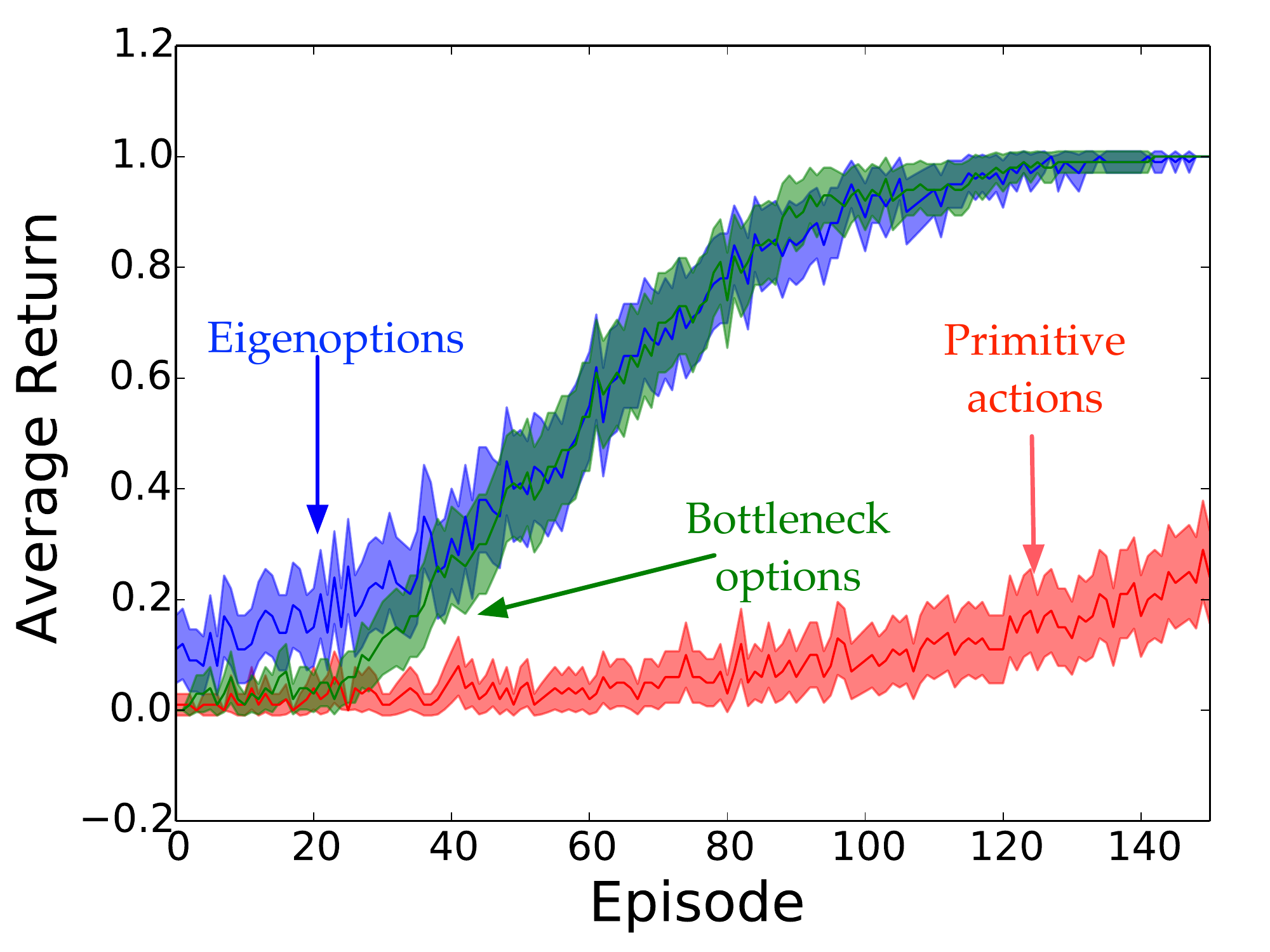}}
    \end{subfigure}
    \begin{subfigure}[b]{0.24\columnwidth}
        \includegraphics[width=\columnwidth]{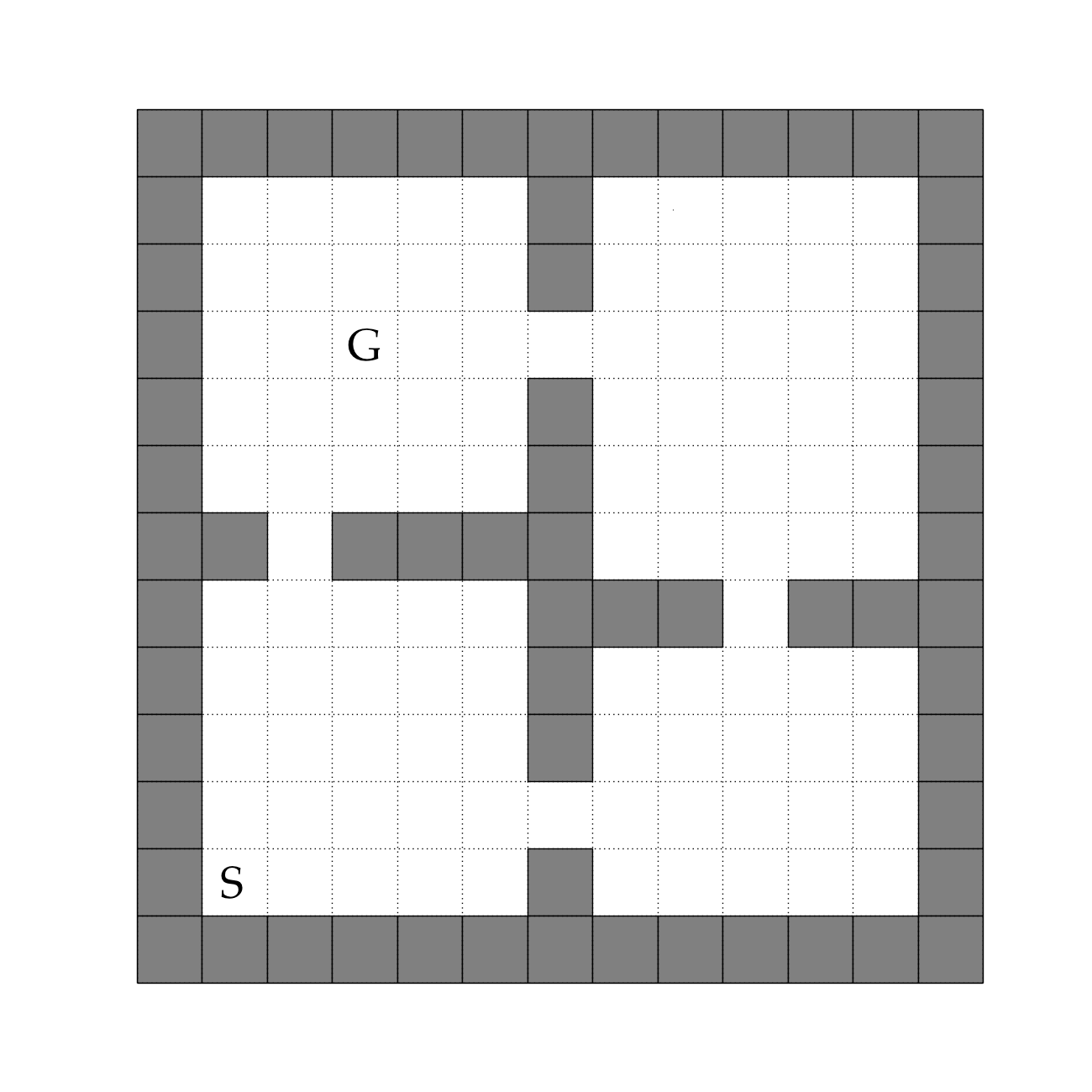}
    \end{subfigure}
    \begin{subfigure}[b]{0.24\columnwidth}
        \raisebox{4.00mm}{\includegraphics[width=\columnwidth]{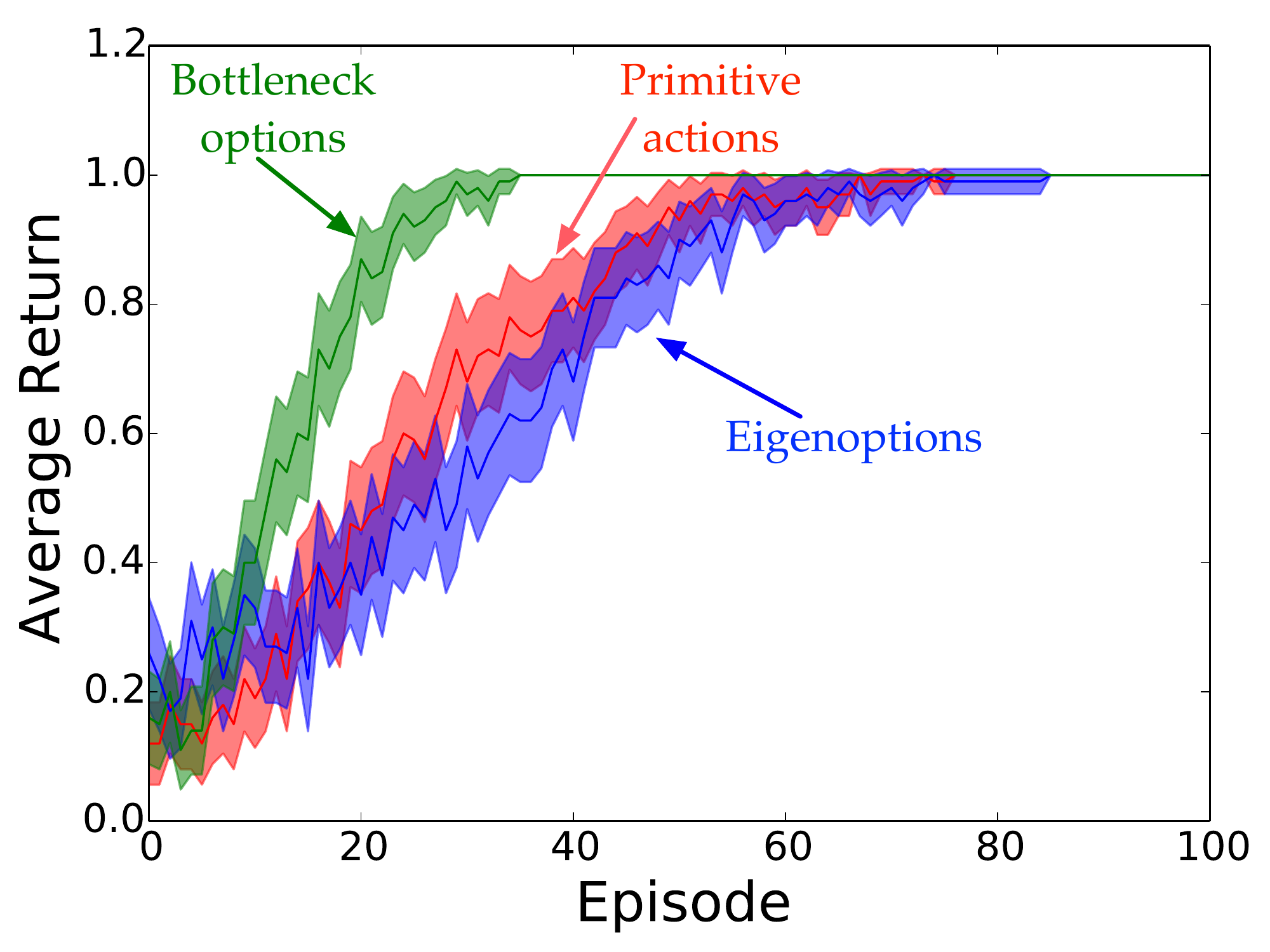}}
    \end{subfigure}
    \begin{subfigure}[b]{0.24\columnwidth}
        \includegraphics[width=\columnwidth]{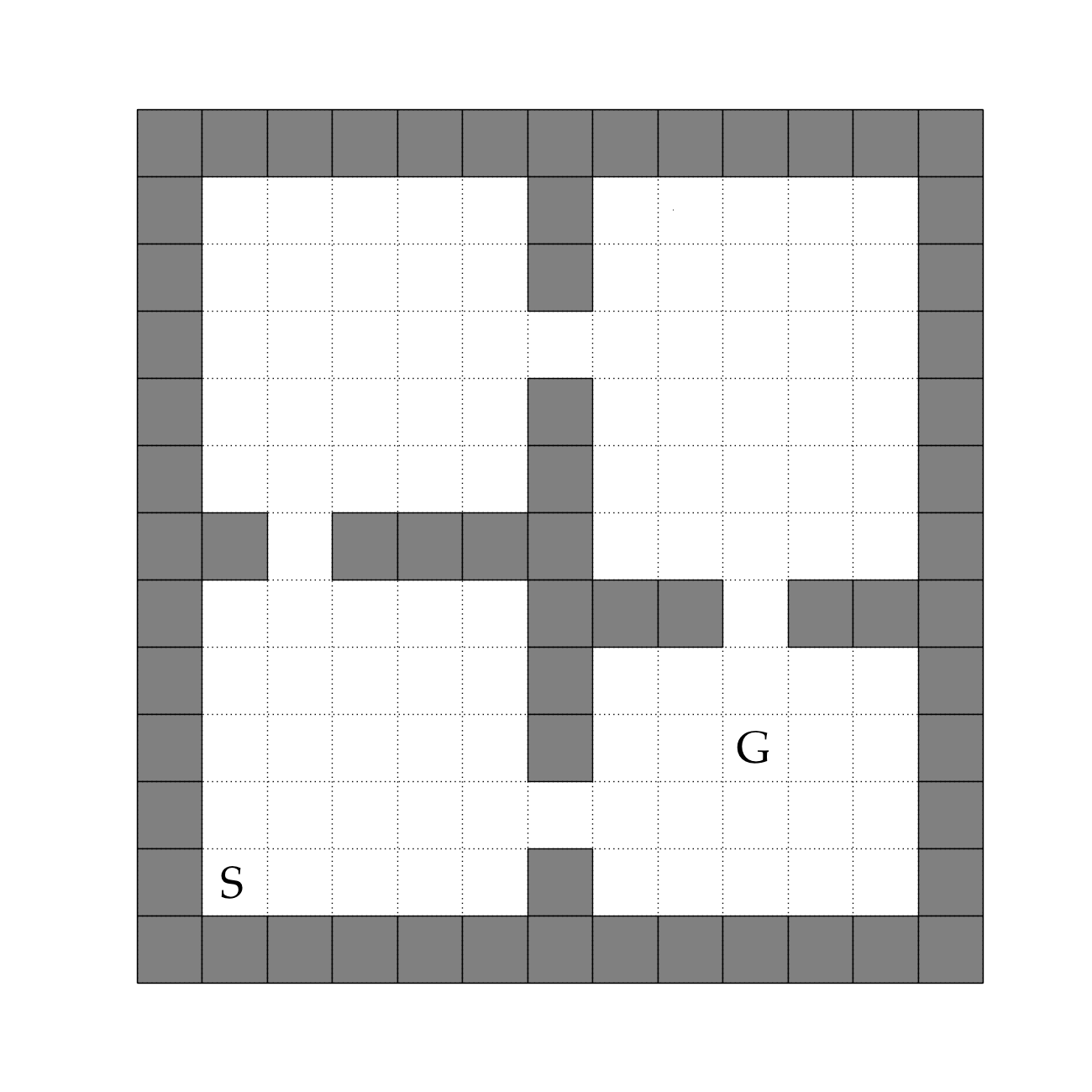}
    \end{subfigure}
    \begin{subfigure}[b]{0.24\columnwidth}
        \raisebox{4.00mm}{\includegraphics[width=\columnwidth]{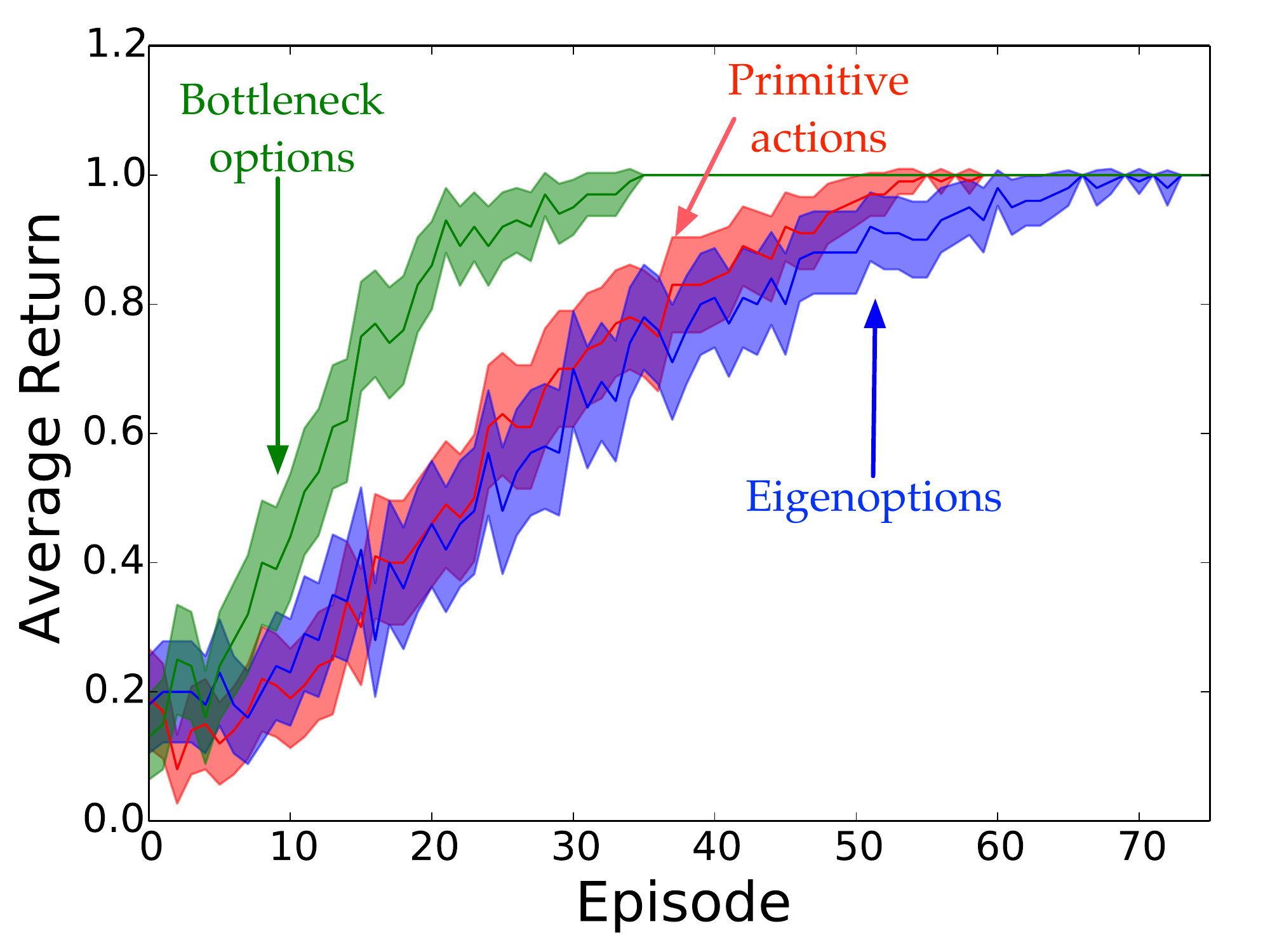}}
    \end{subfigure}
    \begin{subfigure}[b]{0.24\columnwidth}
        \includegraphics[width=\columnwidth]{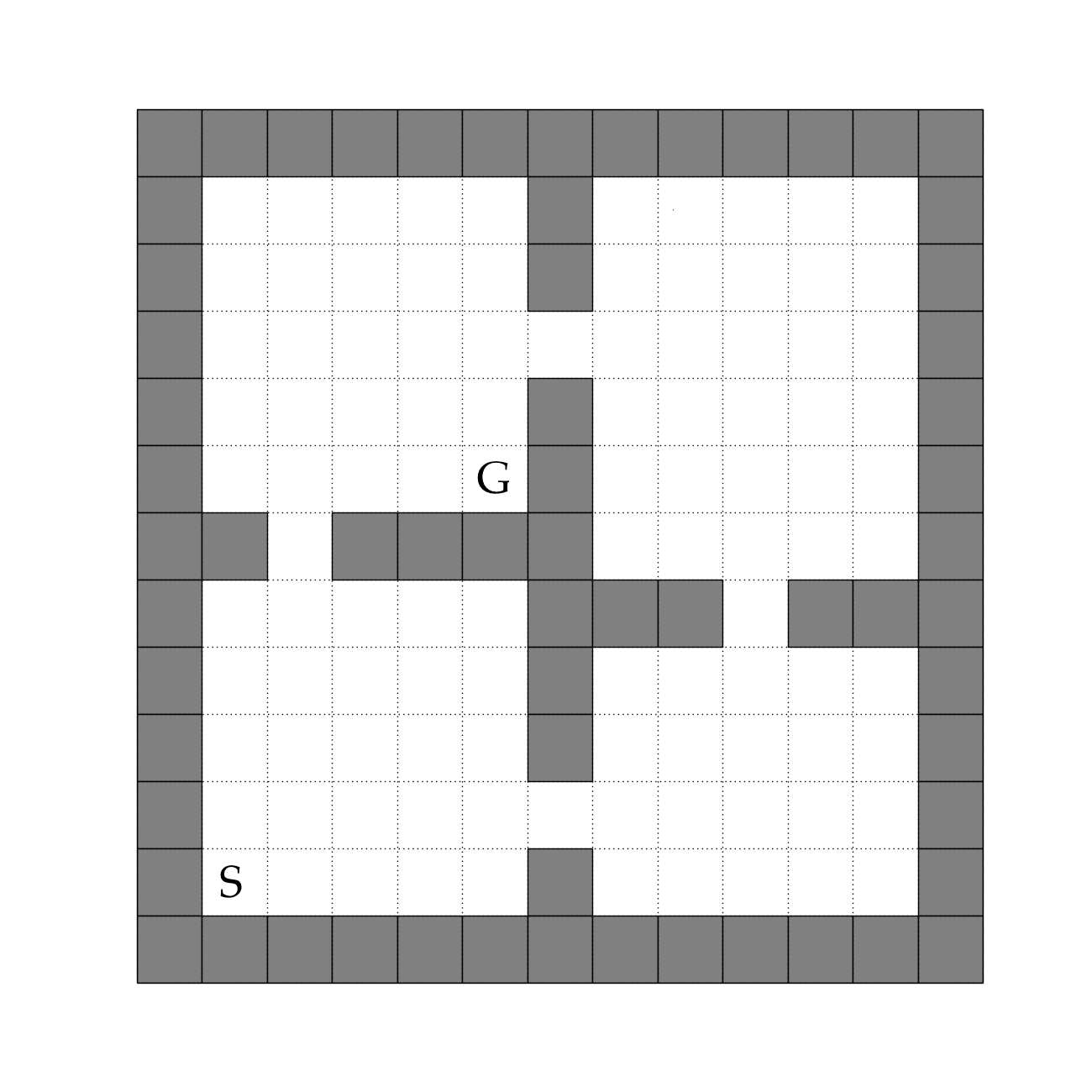}
    \end{subfigure}
    \begin{subfigure}[b]{0.24\columnwidth}
        \raisebox{4.00mm}{\includegraphics[width=\columnwidth]{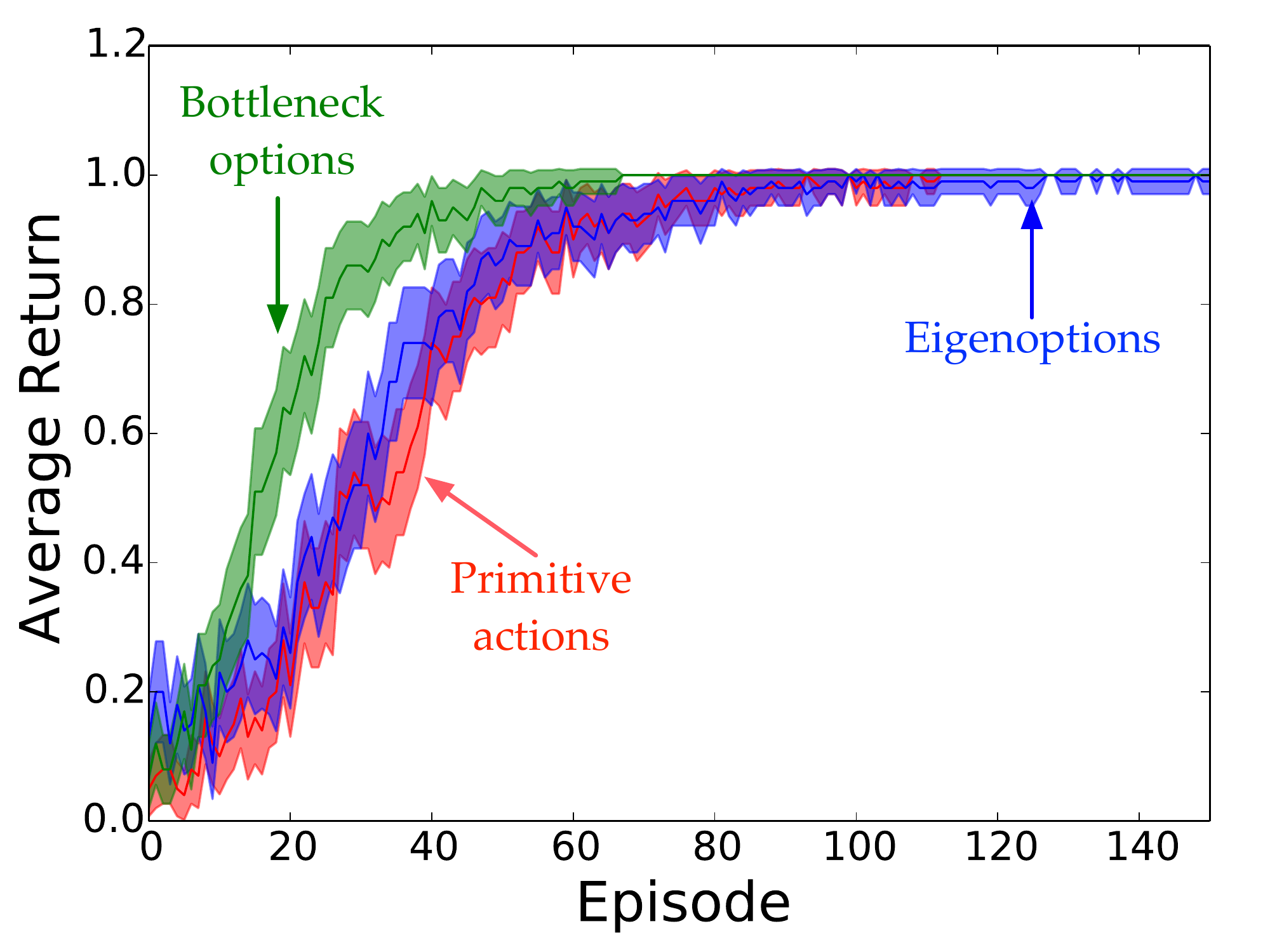}}
    \end{subfigure}
    \caption{Agents performance in different tasks when using eigenoptions, bottleneck options, and primitive actions.}\label{fig:multi}
\end{figure}

\end{document}